\documentclass{article}
\usepackage{graphicx}
\usepackage{adjustbox}
\usepackage{subcaption}
\usepackage{wrapfig}
\usepackage{afterpage}
\usepackage{xspace}
\usepackage{multirow}
\usepackage{amsmath}
\usepackage{mathtools}
\usepackage{amsthm}
\usepackage{thmtools}
\newtheorem{theorem}{Theorem}[section]
\newtheorem{lemma}[theorem]{Lemma}
\newcommand{\ourmethod}{BFF\xspace}

\usepackage{algorithm}
\usepackage{algorithmic}
\PassOptionsToPackage{numbers, compress}{natbib}
\usepackage[preprint]{neurips_2026}

\usepackage[utf8]{inputenc} % allow utf-8 input
\usepackage[T1]{fontenc}    % use 8-bit T1 fonts
\usepackage{hyperref}       % hyperlinks
\usepackage{url}            % simple URL typesetting
\usepackage{booktabs}       % professional-quality tables
\usepackage{amsfonts}       % blackboard math symbols
\usepackage{nicefrac}       % compact symbols for 1/2, etc.
\usepackage{microtype}      % microtypography
\usepackage{xcolor}         % colors
\usepackage{tikz}
\usetikzlibrary{positioning, arrows.meta, calc, fit}
\newcommand{\E}{\mathbb{E}}
\newcommand{\R}{\mathbb{R}}
\newcommand{\cK}{\mathcal{K}}
\newcommand{\cKphi}{\mathcal{K}_{\phi}}
\newcommand{\cH}{\mathcal{H}}
\newcommand{\cD}{\mathcal{D}}
\newcommand{\cO}{\mathcal{O}}
\newcommand{\sg}{\operatorname{sg}}
\newcommand{\given}{\,|\,}
\newcommand{\norm}[1]{\left\lVert #1 \right\rVert}
\usepackage{amssymb}

\newtheorem{proposition}{Proposition}
\theoremstyle{remark}
\newtheorem{remark}{Remark}[section]

\title{Bayesian Filtering in Physical Systems via Test-time Trained Flow Matching}

\author{%
  Ruiqi Feng\thanks{Equal contribution.} \\
  Westlake University \& Zhejiang University \\
  \texttt{fengruiqi@westlake.edu.cn} \\
  \And
  Chongyi Wang\footnotemark[1] \\
  Westlake University \\
  \texttt{wangchongyi@westlake.edu.cn} \\
  \And
  Tao Zhang \\
  Westlake University \& Zhejiang University \\
  \And
  Tailin Wu\thanks{Corresponding author.} \\
  Westlake University \\
  \texttt{wutailin@westlake.edu.cn}
}

\begin{document}

\maketitle

\begin{abstract}
Bayesian filtering provides a principled framework for online state estimation under uncertainty, yet its application to systems with high-dimensional states and complicated posterior distributions remains challenging. 
Recent generative models, such as flow matching, have shown potential in Bayesian filtering.
However, they still rely on particle-based representations of the posterior, which lose the rich information of the full distribution, or tackle a trajectory-level inverse problem that conflicts with the recursive structure of Bayesian filtering.
To address this, we propose a new perspective of \textit{directly encoding the evolving distribution into flow matching model weights}, namely, the Belief Flow Filter (BFF). 
It is a generative filtering framework that updates model weights via gradient descent at test time to track the posterior evolution. 
Thereby, BFF bypasses the scalability issue of particle representations or the flexibility limitation of Gaussian assumptions in conventional filters.
We theoretically justify that the BFF design is structurally aligned with Bayesian filtering, and its training objective targets the recursive filtering operator.
BFF is empirically verified across 5 different physical systems, including ones with chaotic dynamics and highly sparse, non-linear observations.
The results show that BFF attains the best score in 8 of 9 metric-benchmark cells across the three standard 1D and 2D PDE benchmarks, and similarly leads on the extreme single-moving-sensor setting and on a real-world-grounded tokamak plasma estimation task, demonstrating its potential to accurately approximate the Bayesian filtering operator in high-dimensional probability space.

% We empirically verify BFF across 1D and 2D benchmarking PDEs with non-linear and even chaotic dynamics under sparse observation (as few as a single moving sensor), and a real world-grounded tokamak plasma state estimation task with highly non-linear observations. 
% BFF consistently outperforms strong baselines under extreme conditions, including ultra-sparse single-point observations and a Tokamak plasma benchmark characterized by highly non-linear, asynchronous measurements. 
% These results validate BFF as an accurate and robust solver for the distributional Bayesian filtering operator.
\end{abstract}

\section{Introduction}
\label{sec:introduction}

In complex physical systems ranging from weather forecasting to magnetically confined nuclear fusion, a central problem is to sequentially infer the current posterior distribution of the system state (\textit{i.e.}, the belief state) from noisy observations, which is named sequential Bayesian filtering.
% It is particularly of important in safety-critical applications where the uncertainty or probability of certain states is of interest. For example, safe operation of tokamak plasma needs to avoid disruption, which requires accurately estimating the probability of certain high-dimensional plasma states.
Traditional Bayesian filters~\citep{kalman_new_1960, evensen_ensemble_2009, gordon_novel_1993} rely on Gaussian assumptions or discrete particle ensembles, which struggle in the presence of complicated dynamics and posteriors. 
Crucially, the Gaussian assumption fails to capture the multi-modality or long tail in the posterior, while particle-based methods tend to suffer from the curse of dimensionality and can require a prohibitively large number of particles to represent the full posterior in high-dimensional state spaces.

Recent flow-based generative models~\citep{lipman_flow_2023,ho_denoising_2020} provide a scalable free-form probabilistic modeling framework.
Yet, their application to Bayesian filtering problems is still inherently limited:
\citet{li2024learning} and \citet{rozet_score-based_2023} view the filtering process as a posterior sampling problem along the entire trajectory simultaneously. This breaks the recursive structure of sequential Bayesian filtering, reducing computational efficiency and distorting the posterior as dimensionality increases.
Alternatively, \citet{chen_flowdas_2025} and \citet{bao2024ensemble} comply with the recursive structure by explicitly tracking the posterior distribution. However, they still use particles to represent the posterior after each update. 
This approximates the distribution with a limited-capacity compression, leading to information loss as in classical particle-based filters.

Can we maintain the recursive nature of Bayesian filtering while alleviating the bottleneck of representing the posterior? Here, we introduce the Belief Flow Filter (BFF), whose key insight is \emph{using flow matching model weights to encode the posterior distribution}.
To lift the filtering procedure from the probability space into the weight space, we derive a learnable surrogate loss, whose gradient updates the flow matching model weight at test time (\textit{i.e.}, test-time training~\citep{sun_learning_2025,sun_test-time_2020}) to track the posterior distribution evolution at each physical time step. 
% By minimizing our proposed surrogate loss via gradient descent, this approach ensures the temporal continuity of the distribution. 
% This mechanism avoids the computational cost of retraining the network from scratch, providing a fast, recursive update suitable for real-time state estimation in complex physical systems.
Furthermore, we propose a training objective which ensures that the learnable surrogate loss is optimized to produce the true Bayesian filtering operator, enabling scalable learning of \ourmethod in high-dimensional probability space.

We tested BFF across 1D and 2D benchmarking PDEs with non-linear and even chaotic dynamics under sparse observation (as few as one single moving sensor), and a real-world-grounded tokamak plasma state estimation task with highly non-linear observations. 
Compared against Gaussian filters (EKF, EnKF) and state-of-the-art generative model-based filters (SDA, FlowDAS), BFF is the only method that holds up on reconstruction error, physical consistency of the predicted state, and calibration of the inferred posterior at once, which validates that the Bayesian filtering operator is faithfully learned. 
\begin{wrapfigure}{r}{0.5\linewidth} % on top of second page
    \centering
    % \vspace{-2em}
    \includegraphics[width=\linewidth]{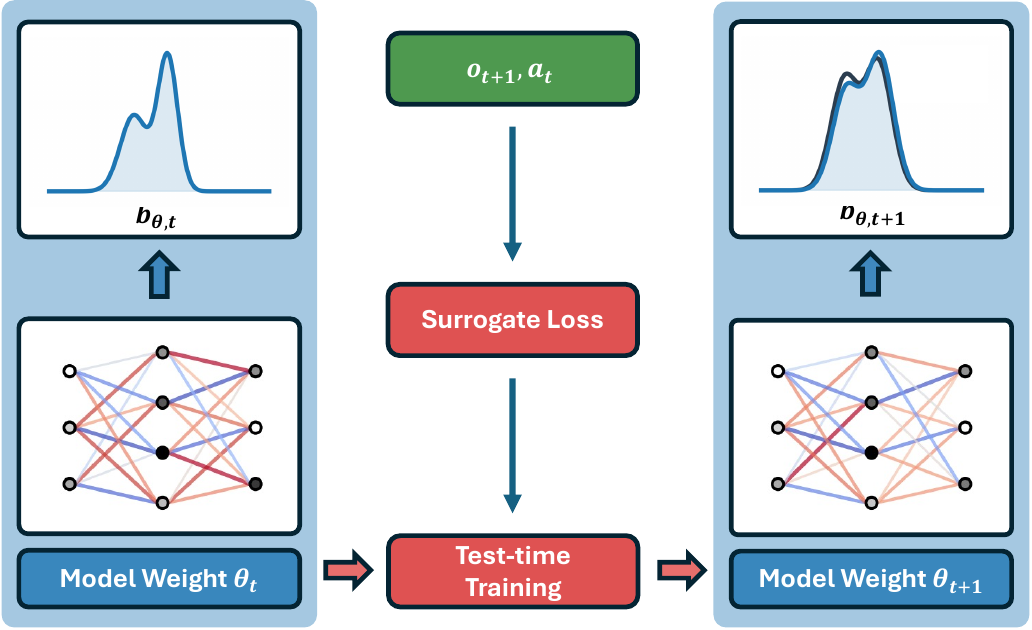}
    \caption{\textbf{An overview of our proposed \ourmethod.} 
    Blue frames denote the flow matching model parameterized by $\theta$ at different physical times.
    %  $t$ or $t+1$, where
    The model weights $\theta_t$ parameterize the posterior distribution $b_{\theta_t}$, and evolve ($\theta_t \rightarrow \theta_{t+1}$) as new information $o_{t+1},a_t$ comes in, creating the next posterior $b_{\theta_{t+1}}$. 
    The update is executed by a gradient step on a learnable surrogate loss $\ell_\phi$ that takes $\theta_{t}$, $o_{t+1}$, and $a_t$ as inputs.
    }
    \label{fig:fig1}
    \vspace{-1.5em}
\end{wrapfigure}
The gap is most pronounced under ultra-sparse observations and the highly non-linear tokamak diagnostics, suggesting that BFF's weight-space update is closer to the true Bayesian filtering operator than these baselines.
% \ruiqi{TODO: update accordingly after rewriting the experiment section}

Our main contributions are three-fold:

(1) A new perspective of using \emph{flow-based generative model weights as belief states} in sequential Bayesian filtering to overcome the posterior representation limitations.
% This framework overcomes the long-standing representational limits of posterior distributions using modern flow matching models.

(2) Theoretical foundations for BFF of a specifically designed surrogate loss that provably represents the exact recursive filtering operator, paired with an outer-loop training objective that bounds the distribution-level error.

% (3) We evaluate BFF on 1D/2D chaotic PDEs and a highly non-linear, partially observed Tokamak plasma benchmark.

% We propose a specific surrogate loss to guide the gradient update of the posterior distribution, proving that it can represent the exact recursive filtering operator.
% Furthermore, the outer-loop training bounds the distribution-level error. We also open-source a computationally feasible implementation.

(3) Empirical evidence that BFF accurately tracks the posterior on an array of challenging physical systems, including chaotic dynamics, sparse observations, and non-linear observation models. The code is open-sourced \href{https://anonymous.4open.science/r/BeliefFlow-81D4}{here}.
% We demonstrate that BFF accurately tracks high-dimensional states and generalizes to long-horizon inference, consistently preserving physical consistency and uncertainty calibration.

\section{Background}
\subsection{Problem Setup}
\label{sec:problem-setting}
Consider a discrete-time dynamical system. At time step $t$, the system is in a true physical state $s_t \in \mathcal{X} \subseteq \mathbb{R}^n$, where $\mathcal{X}$ denotes the state space. The state $s_t$ could be the fluid velocity and pressure field or the temperature and density profile of the tokamak plasma at time $t$, both of which are high-dimensional. Subject to a control action $a_t \in \mathcal{A} \subseteq \mathbb{R}^m$, the state evolves to $s_{t+1}$ governed by the environmental transition dynamics $p(s_{t+1} \mid s_t, a_t)$. However, in a partially observable setting, the exact physical state remains hidden. Instead, we only have access to noisy and potentially sparse observations $o_t \in \mathcal{O} \subseteq \mathbb{R}^k$ (typically $k\ll n$), generated by an observation model $p(o_t \mid s_t)$. Here, $\mathcal{A}$ and $\mathcal{O}$ denote the action and observation spaces, respectively. We adopt the standard Markov assumption: the transition depends only on the current state and action, and the observation depends only on the current state. This structure, satisfied by most PDE-governed physical systems under full-state representation, underlies the recursive form of the filtering updates below.

Due to this partial observability, we must infer the true state from the history of past observations and actions $h_t \coloneqq (o_{1:t}, a_{1:t-1})$. We encapsulate this uncertainty in the posterior probability distribution (also referred to as the belief state in control theory), defined as
$
    b_t(s) \coloneqq p(s_t = s \mid h_{t})
    % \label{eq:belief-def}
$.
The filtering problem can thus be viewed as recursively updating this belief state $b_t$ into $b_{t+1}$ upon executing action $a_t$\footnotemark and receiving a new observation $o_{t+1}$. This is achieved via a two-step iteration process:
\footnotetext{Filters must account for action when control is exerted on the physical system. For example, plasma in tokamaks is controlled, and our experiments also include this case. When action is not present, the argument $a$ is simply omitted.}
\begin{align}
    &\bar b_{t+1}(s) = \int p(s \mid s', a_t) b_t(s') \, \mathrm{d}s' && \text{(Prediction Step)} \label{prediction} \\
    &b_{t+1}(s) = \frac{\bar b_{t+1}(s) p(o_{t+1} \mid s)}{\int \bar b_{t+1}(s) p(o_{t+1} \mid s) \, \mathrm{d}s} && \text{(Update Step)} \label{eq:bayesian_update}
\end{align}
where $\bar{b}_{t+1}$ is the propagated \emph{prior}, the distribution of where the system would land at $t+1$ before the new observation arrives, whereas $b_{t+1}$ is the Bayesian \emph{posterior} obtained by updating the propagated prior with the latest observation.
Equations~\eqref{prediction} and~\eqref{eq:bayesian_update} together define a recursive filtering operator $K: \mathcal{P}(\mathcal{X}) \times \mathcal{O} \times \mathcal{A} \mapsto \mathcal{P}(\mathcal{X})$ that maps the current belief state, together with the latest observation and action, to the next belief state. In the following, we abbreviate $K(b_t, o_{t+1}, a_t)$ as $K b_t$ whenever the new information $(o_{t+1}, a_t)$ is clear from context:
\begin{equation}
    b_{t+1}(s) = K b_t(s) \coloneqq \frac{1}{Z_{t+1}} p(o_{t+1} \mid s) \int p(s \mid s', a_t) b_t(s') \, \mathrm{d}s',
    \label{eq:filter-operator}
\end{equation}
where $Z_{t+1} = \int p(o_{t+1} \mid s) \int p(s \mid s', a_t) b_t(s') \, \mathrm{d}s' \, \mathrm{d}s$ is the normalization constant ensuring $b_{t+1}$ integrates to one. 
Our aim is to learn this Bayesian update by faithfully represent the belief state and then learn the recursive operator $K$ via test-time training, as described and theoretically justified next.

% Equation~\eqref{eq:filter-operator} follows from applying Bayes' rule to $p(s_{t+1} \mid h_{t+1})$ and marginalizing over the previous state, under the Markov assumptions stated above. % it is NOT this case! bt is conditioned on history!
% Proposition~\ref{prop:sec23-latent-amortization} in Section~\ref{sec:ttt} establishes that our surrogate loss is expressive enough to represent $K$ exactly. This recursive operator $K$ is the target that our learned $\mathcal{K}_\phi$ aims to approximate.

\subsection{Neural Representation of Belief States}
\label{sec:neural-belief}
The recursive Bayesian filtering involves representation of the complex belief state (Eq.~\eqref{eq:bayesian_update}).
Previously, the belief is usually approximated by particle ensembles or simple parametric distributions (\emph{e.g.}, Gaussian), both of which limit the expressivity.
Alternatively, flow-based generative models~\citep{ho_denoising_2020, lipman_flow_2023} parameterize rich, high-dimensional distributions through a learned velocity field $u_\theta(s^{(\tau)}, \tau)$ that defines a probability flow ODE $\mathrm{d} s^{(\tau)}/\mathrm{d}\tau = u_\theta(s^{(\tau)}, \tau)$ for $\tau \in [0, 1]$. Integrating this ODE from $\tau = 0$ to $\tau = 1$ induces a flow map $\phi^{u_\theta}: \mathcal{X} \to \mathcal{X}$ that transports a simple Gaussian distribution\footnotemark{} $p^{(0)}$ into an arbitrarily complex target distribution $p^{(1)}$
%  via push-forward operation $p^{(1)}=(\phi^{u_\theta})_\# p^{(0)}$, 
up to mild regularity conditions~\citep{lipman_flow_2023}.
Throughout, superscripts in parentheses (\emph{e.g.}, $s^{(\tau)}, p^{(\tau)}$) index \emph{flow time} internal to the generation process, while subscripts (\emph{e.g.}, $\theta_t, b_t$) index \emph{physical time} of the filtering process.
\footnotetext{Technically, flow matching allows construction of flows between any two distributions, but most existing work focuses on the Gaussian base distribution for simplicity and efficiency.}

Fitting $\theta$ to a target distribution $q$ uses the conditional flow matching (CFM) loss~\citep{lipman_flow_2023,lipman_flow_2024}
\begin{equation}
    \mathcal{L}_{\text{CFM}}(\theta;q)
    \;=\;
    \E_{\tau\sim U(0,1),\,s^{(0)}\sim p^{(0)},\,s^{(1)}\sim q}\!\left[\norm{u_\theta(s^{(\tau)},\tau) - (s^{(1)}-s^{(0)})}^{2}\right],
    \label{eq:cfm-loss}
\end{equation}
with $s^{(\tau)} = \tau s^{(1)}+(1-\tau)s^{(0)}$. We write $\mu_q$ for the resulting joint distribution of $(s^{(\tau)},\tau)$, the flow matching probability path of $q$. The minimizer of $\mathcal{L}_{\text{CFM}}$ drives $p^{(1)}$ to $q$ in the population limit.

% {\color{red} TODO: the notation is a mess}
We use $\theta$ to represent the belief state. At each physical time step $t$, we maintain a weight vector $\theta_t\in \Theta$ that parameterizes the velocity field $u_{\theta_t}$, which in turn defines a parametric mapping
\begin{equation}
    \Phi:\Theta\to\mathcal{P}(\mathcal{X}),\quad \theta_t\mapsto b_{\theta_t},
    \label{eq:belief-param}
\end{equation}
where $b_{\theta_t}$ is our model's approximation of the true belief $b_t$.
In this parameterization, $\theta_t$ plays the role of a high-capacity ``belief representation''. Unlike a finite particle ensemble or a Gaussian mean-covariance pair, $\theta_t$ encodes a full continuous distribution through the flow map it induces.

% Under this parameterization, the recursive Bayesian filtering operator $K$ defined in Eq.~\eqref{eq:filter-operator} acts on the distribution space $\mathcal{P}(\mathcal{X})$, and its action on our parameterized belief $b_{\theta, t}$ corresponds to a map at the weight level: given $\theta_t$, the new observation $o_{t+1}$, and action $a_t$, we seek a weight update $\theta_t \mapsto \theta_{t+1}$ such that the induced belief $b_{\theta, t+1}$ approximates $K b_{\theta, t}$. Designing and learning this weight-level update rule is the subject of Section~\ref{sec:ttt}.
\subsection{Bayesian Filtering Update in Neural Network Weight Space}
\label{sec:ttt}

% The previous subsection represents a belief state by the weights of a
% flow-matching model.

At filtering step $t$, the flow matching weights $\theta_t$ induce a belief $b_{\theta_t}$.
As recursive Bayesian filtering updates beliefs in distribution space, our filtering algorithm needs to update neural network weights
\begin{equation}
  b_{t+1}
  =
  K(b_t,o_{t+1},a_t)
  \Leftrightarrow
  \theta_{t+1}
  =
  \cK(\theta_t,o_{t+1},a_t)
  ,
  \label{eq:sec23-filtering-operator}
\end{equation}
where $\cK$ denotes the weight-space update corresponding to $K$.
Yet, designing $\cK$ is non-trivial. We can compute $b_\theta$ from $\theta$ (one forward pass plus ODE integration), but inverting this map, \textit{i.e.}, finding the $\theta$ whose induced belief equals a given target $K b_{\theta_t}$, has no closed form.

In this subsection, we design a closed-form evaluable loss whose minimizer at every timestep correctly induces the next belief. This establishes the theoretical foundation that our algorithm builds on.

\paragraph{Bayesian filtering step as an ideal loss minimizer.}
We reformulate the belief-space target $K b_{\theta_t}$ as a velocity-matching objective in weight space, define $\cK$ as its minimizer, and approximate it by test-time gradient steps.
Concretely, instantiating the CFM loss Eq.~\eqref{eq:cfm-loss} with target $K b_{\theta_t}$ and writing it in the marginal form gives the \emph{ideal filtering loss}
\begin{equation}
    \ell^{\star}(\theta)
    \;=\;
    \E_{(s,\tau)\sim\mu_{K b_{\theta_t}}}\!\left[\norm{u_\theta(s,\tau)-u_{K b_{\theta_t}}(s,\tau)}^{2}\right],
    \label{eq:sec23-ideal-loss}
\end{equation}
where $u_{K b_{\theta_t}}(s,\tau)$ is the target marginal velocity field of $K b_{\theta_t}$, and $\mu_{K b_{\theta_t}}$ is the joint distribution along flow probability path as in Eq.~\eqref{eq:cfm-loss}. We anchor the target at $K b_{\theta_t}$ since the model's current belief is $b_{\theta_t}$ and the best one-step weight-space update fits $\theta_{t+1}$ to the Bayes update of this current belief. The residual gap between $b_{\theta_t}$ and the true $b_t$ is closed by the outer training of $\phi$ (Theorem~\ref{theorem:ttt_outer_loss}).
We then define the weight-space update $\cK$ as its minimizer,
\begin{equation}
  \cK(\theta_t,o_{t+1},a_t) \;:=\; \arg\min_\theta\,\ell^{\star}(\theta),
  \label{eq:sec23-cK-def}
\end{equation}
which defines the lift of $K$ in probability space to $\cK$ in weight space.

% Since $\ell^{\star}$ is not directly evaluable, the next two paragraphs aim to construct a tractable surrogate of it. First, the \emph{empirical filtering loss} $\hat{L}_N$ realizes a Monte Carlo estimator of $\ell^{\star}$, and then a learned latent surrogate loss $\ell_\phi$ approximates $\hat{L}_N$, which in turn implements $\cK$ in practice.

\paragraph{Empirical filtering loss as an evaluable estimator of the ideal loss.}

However, we cannot evaluate $\ell^{\star}$ directly because the target velocity $u_{K b_{\theta_t}}$ is unknown.
Fortunately, the Bayesian filtering and flow matching structures (elaborated in Appendix \ref{app:two-channel-anchor}) allow us to approximate $\ell^\star$ with the \emph{empirical filtering loss}
$\hat\ell$. It consists of a model velocity branch $F_N$ depending only on $\theta$ and a target branch $G_N$ that depends only on $(\theta_t,o_{t+1},a_t)$, both evaluable since the unknown $u_{K b_{\theta_t}}$ does not appear:
\begin{align}
    &\hat{\ell}(\theta)
    \;:=\;
    \norm{F_N(\theta)-G_N(\theta_t,o_{t+1},a_t)}^{2},
    \label{eq:sec23-anchor-loss} \\
    \nonumber
    &F_N(\theta)
    \;=\;
    \frac{1}{\sqrt N}\big(u_\theta(s_i^{(\tau_i)},\tau_i)\big)_{i=1}^{N},\\
    \;\;
    &G_N(\theta_t,o_{t+1},a_t)
    \;=\;
    \frac{1}{\sqrt N}\Big(\underbrace{s_i^p - s_i^{(0)}}_{\text{propagation}}
    +
    \frac{1-\tau_i}{\tau_i}\,
    \underbrace{
        \nabla_s\log\bar p^{(\tau_i)}(o_{t+1}\given s)
    }_{\text{likelihood}}\big|_{s=s_i^{(\tau_i)}}
    \Big)_{i=1}^{N},
    \label{eq:sec23-FG}
\end{align}
where $()_{i=1}^N$ denotes the concatenation of $N$ vectors.
% Here, the vectors $F_N,G_N$ are stacked from $N$ per-sample model velocity and target.
The constituents of $F_N,G_N\in\R^{N\times n}$ are defined as follows. For each anchor $i\in\{1,\dots,N\}$, $(s_i^{(0)},\tau_i)$ are the standard flow-matching anchor noise and time, and $s_i^p\sim\bar b_{\theta_{t+1}}$ is sampled from the propagated model prior of Eq.~\eqref{prediction}. 
Together, they form the path point $s_i^{(\tau_i)}=\tau_i s_i^p+(1-\tau_i)s_i^{(0)}$,
%  $F_N$ stacks the model velocity $u_\theta(s_i^{(\tau_i)},\tau_i)$, while $G_N$ stacks a two-channel target at the same point: 
thus defining the \emph{propagation} term and the \emph{likelihood} term in Eq.~\eqref{eq:sec23-FG}. 
In the likelihood term, path-smoothed likelihood $\bar p^{(\tau)}(o_{t+1}\given s)\propto\E_{s'\sim\bar b_{\theta_{t+1}}}[\,p(o_{t+1}\given s')\,\mathcal{N}(s;\tau s',(1-\tau)^2 I)\,]$ (normalized by integrating out $o_{t+1}$) is evaluated by self-normalized Monte-Carlo with $N_j$ propagation samples. The details are in Appendix~\ref{app:two-channel-anchor}.

Intuitively, the \textit{empirical filtering loss} makes \textit{ideal filtering loss} evaluable via conditional flow matching and Monte-Carlo estimation. $F_N$ stacks model velocity $u_\theta(s_i^{(\tau_i)})$, whereas $G_N$ stacks the \emph{propagation} term and the \emph{likelihood} term that together form the target velocity. 
The {propagation term} is the conditional flow matching velocity acquired by first sampling a cloud of particles from $b_\theta$ and then simulate them to where the prior would land, \textit{i.e.} $\bar b_{\theta_{t+1}}$. 
The {likelihood term} is the Bayesian correction that tilts $\bar b_{\theta_{t+1}}$ into the target posterior $K b_{\theta_t}$.
% Using $\bar p^{(\tau)}$ rather than the data-level $p$ keeps this score split exact at every $\tau\in(0,1]$, removing any Gaussianity assumption on either the prior or the observation model.
% Using $\bar p^{(\tau)}$ rather than the data-level $p$ keeps the score split exact at every $\tau\in(0,1]$, removing any Gaussianity assumption on either the prior or the observation model (Appendix~\ref{app:two-channel-anchor}).
% the distribution of where the system would land at $t+1$ before the new observation arrives. 
% The stacked model velocity $u_\theta(s_i^{(\tau_i)})$, $F_N$, involves the flow matching interpolant $s_i^{(\tau_i)}=\tau_i s_i^p+(1-\tau_i)s_i^{(0)}$, which is formed by pairing each $s_i^p$ with $s_i^{(0)}\sim\mathcal{N}(0,I)$ and $\tau_i\sim U(0,1)$. 
Furthermore, we formally state that the \textit{empirical filtering loss} estimates the \textit{ideal filtering loss} with controllable bias and shares its weight-space minimizer:
\begin{restatable}[Empirical filtering loss, informal]{proposition}{propEmpFiltLoss}
\label{prop:sec23-emp-filt}
Under standard regularity assumptions, $\hat\ell$ is an estimator of $\ell^\star$ up to a $\theta$-independent constant $C$, with bias and variance
\begin{equation*}
    \big|\E[\hat\ell(\theta)]-\ell^\star(\theta)-C\big|
    \;\le\;
    \frac{L_1}{N_j},
    \qquad
    \mathrm{Var}\!\big(\hat{\ell}(\theta)\big)
    \;\le\;
    \frac{1}{N}\Big(\sigma^{2}+\frac{c_5}{N_j}\Big),
\end{equation*}
for constants $L_1,c_5,\sigma^{2}$ depending only on the prior, observation, and dynamics models.
% \begin{equation}
%     \Pr\!\left(\big|\hat{\ell}(\theta)-\E[\hat{\ell}(\theta)]\big|\ge\varepsilon\right)
%     \;\le\;
%     \frac{\sigma^{2}}{N\varepsilon^{2}}.
%     \label{eq:sec23-cheby}
% \end{equation}
\end{restatable}
% Hence, taking $N_j\propto N$ gives a single $1/N$ rate, and 
Hence, $\hat\ell$ shares its weight-space minimizer with $\ell^\star$ up to an $\cO(1/N_j)$-controlled gradient bias. The full statement  and proof are deferred to Appendix~\ref{app:two-channel-anchor}.
% The proof in Appendix~\ref{app:two-channel-anchor} exploits a decomposition of the Bayesian filtering step Eq.~\eqref{eq:filter-operator} and the velocity-score identity in standard flow matching.

% The proof, in Appendix~\ref{app:two-channel}, lifts the Bayesian update to the velocity residual Eq.~\eqref{eq:sec23-anchor-loss} via a score decomposition. 

\section{Test-time Training of Flow Matching Filters}

In this section, we present our method \ourmethod. It represents the belief with flow matching model weights $\theta$, and executes Bayesian filtering through test-time training (TTT) using a surrogate loss $\ell_\phi$ with learnable parameters $\phi$. 
Two ``training'' loops are involved: the inner loop that updates $\theta$ via $\phi$-dependent surrogate loss (Section~\ref{sec:testing} and Algorithm~\ref{alg:ttt_test}) and the outer loop that updates $\phi$ via the outer loss (Section~\ref{sec:training} and Algorithm~\ref{alg:ttt_train}).
The implementation details are in Appendix~\ref{app:implementation}.

\subsection{Surrogate Loss for Test-time Training}
\label{sec:testing}
% \paragraph{ (TTT) update.}
The empirical filtering loss $\norm{F_N(\theta)-G_N(\theta_t,o_{t+1},a_t)}^2$ provides a representable estimation of $\ell^{\star}$, but evaluating $G_N$ at every filtering step requires forward dynamics and observation models, which can be expensive and even infeasible if these models are not known in advance.
This motivates us to amortize this by introducing a \textit{learnable} surrogate loss
\begin{equation}
  \ell_\phi(\theta;\theta_t,o_{t+1},a_t)
  \;=\;
  \norm{
    f_\phi(\theta)
    -
    g_\phi(\sg[\theta_t],o_{t+1},a_t)
  }_{\cH}^{2},
  \label{eq:sec23-bff-surrogate}
\end{equation}
% The heads do not need to reproduce the full $\R^{N\times n}$ anchor vectors Eq.~\eqref{eq:sec23-FG}, but instead perform velocity matching in a learned latent representation defined by $f_\phi$ and $g_\phi$. 
where $f_\phi$ and $g_\phi$ are two \textit{learnable} projection heads\footnote{To highlight that they constitute the learnable part in the surrogate loss $\ell_\phi$ and are optimized together in $\mathcal{L}(\phi)$ (Section~\ref{sec:training}), we denote both $f_\phi$ and $g_\phi$ with the same subscript $\phi$ though they are separate neural networks.}
which approximate $F_N$ and $G_N$ respectively.
Thereby, $\ell_\phi$ approximate $\hat\ell$ with error $|\ell_\phi-\hat\ell|$ controlled by each head's approximation error (Proposition~\ref{prop:sec23-latent-amortization} in the appendix).
In practice, we use a single test-time gradient step on $\ell_\phi$ to update $\theta_{t+1}$:
\begin{equation}
  \theta_{t+1}
  \;=\;\cKphi(\theta_t,o_{t+1},a_t)
  \;=\;\theta_t-\eta\,\nabla_{\theta}\,
  \ell_\phi(\theta;\theta_t,o_{t+1},a_t)\big|_{\theta=\theta_t},
  \label{eq:sec23-ttt-update}
\end{equation}
where $\cKphi$ denotes the $\phi$-parameterized realization of the weight-space update $\cK$. $\cKphi$ effectively performs test-time training (TTT) on the flow matching parameters $\theta$, using the learned surrogate loss $\ell_\phi$. As Figure~\ref{fig:arch} show, our implementation uses the DiT backbone~\citep{peebles_scalable_2023}, augmented with a per-layer TTT module inserted after each self-attention block; the TTT module's weights are updated as $\theta$, while the rest of the DiT backbone is shared across all physical steps. The surrogate loss $\ell_\phi$ constitutes two learned projection heads parameterized by $\phi$.

\begin{remark}[Surrogate loss as a learnable amortization]
The surrogate loss $\ell_\phi$ amortizes $\hat\ell$ of Section~\ref{sec:ttt} by replacing $F_N, G_N$ with learned heads $f_\phi, g_\phi$. The structural link to the ideal filtering loss $\ell^\star$ (residual-norm form, asymmetric split between $\theta$ and $(\theta_t,o_{t+1},a_t)$) is preserved, so the TTT update Eq.~\eqref{eq:sec23-ttt-update} stays a faithful realization of $\cK$. $\hat\ell$'s dependence on known dynamics, observation model, and finite-sample Monte Carlo anchors is absorbed into $\phi$.
Moreover, these approximations do not bottleneck $\cKphi$. \textit{(i) Expressivity.} Eq.~\eqref{eq:sec23-ttt-update} can realize the Bayesian filtering operator $K$ (Proposition~\ref{prop:appendix-expressivity}), so the function class is not capped by the inherited structure. \textit{(ii) Training target.} $\phi$ is trained against $K$, not $\hat\ell$ (Theorem~\ref{theorem:ttt_outer_loss}), pulling $\ell_\phi$ toward the true filter. Hence BFF can learn general Bayesian filters.
\end{remark}

\subsection{Learning the Surrogate Loss}
\label{sec:training}

\begin{figure}[t]
    \centering
    \includegraphics[width=\linewidth]{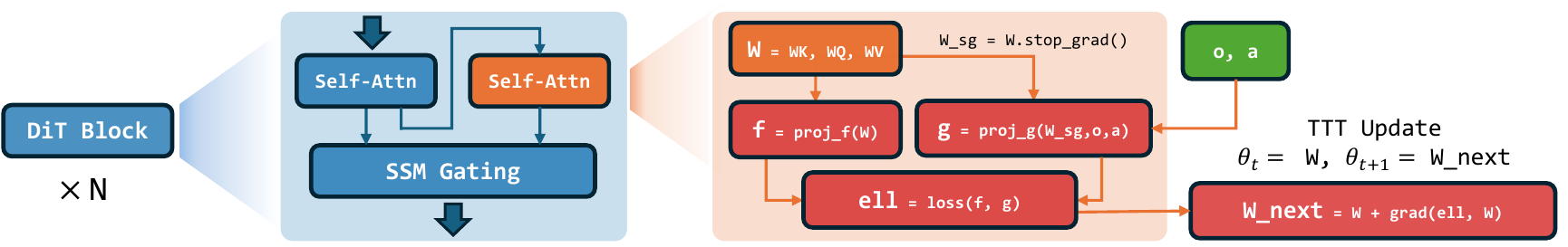}
    \caption{\textbf{Architecture overview of \ourmethod.} The weights $\theta$ (\texttt{W} in the figure) in the per-layer TTT modules inserted into a DiT backbone are updated at test time by the surrogate loss $\ell_\phi$ whose \texttt{proj\_f} and \texttt{proj\_g} are learnable heads parameterized by $\phi$. Detailed description is in Appendix~\ref{app:implementation}.}
    \label{fig:arch}
\end{figure}

So far, we have established that our $\cKphi$, \textit{i.e.} TTT with the surrogate loss $\ell_\phi$, is a structurally faithful realization of the Bayesian filtering operator $K$. Since $\ell_\phi$ also consists of learnable parameters $\phi$, we now show how the meta-parameters $\phi$ are trained from datasets with trajectories $(h_t,\, o_{t+1},\, a_t)$ so that $\cKphi$ actually approximates $K$ in practice.

\paragraph{Filtering operator error in probability space.}
The two operators live in different spaces. $K$ acts on beliefs in $\mathcal{P}(\mathcal{X})$, while $\cKphi$ acts on weights in $\Theta$. To compare them on equal footing, we use the weight-to-belief lift
$\Phi:\Theta\to\mathcal{P}(\mathcal{X}),\ \theta\mapsto b_\theta$ from Section~\ref{sec:neural-belief} (Eq.~\eqref{eq:belief-param}) and feed both operators the same current belief $\Phi(\theta_t)$, then measure how their resulting next
beliefs differ:
\begin{align}
    \nonumber
    &\mathcal{E}(\phi)
    \;\coloneqq\;
    \E_{t\sim U\{1{:}T-1\}}\,\E_{(h_t,\, o_{t+1},\, a_t)\sim P_{\text{data}}}\!
    \left[
        W_2^2\!\Big(
            K\big(\Phi(\theta_t),\, o_{t+1},\, a_t\big),\;
            \Phi\!\big(\cKphi(\theta_t,\, o_{t+1},\, a_t)\big)
        \Big)
    \right],\\
    &\theta_t\;=\;\cKphi^t(h_t)\;=\;
    \underbrace{\cKphi(...\cKphi(\cKphi}_{\cKphi\text{ applied }t\text{ times}}(\theta_0,o_1,a_0),o_2,a_1), ...,o_{t},a_{t-1})
    \label{eq:operator-error}
\end{align}
where $h_t$ is the history of observations and actions from the dataset, $W_2$ measures the Wasserstein-$2$ distance between distributions, and $\theta_t$ is produced from $\theta_0$ by recursively applying $\cKphi$ to the history $h_t\sim P_{\text{data}}$, identical to the filtering loop at test time.
This error quantifies the distributional mismatch between the true Bayesian filtering update and the learned TTT update, averaged across the dataset distribution $P_{\text{data}}$. The pseudocode is in Algorithm~\ref{alg:ttt_train}.

\paragraph{Sample loss bounds to filtering operator error.}
Eq.~\eqref{eq:operator-error} is at the level of distributions, which in general cannot be directly evaluated or optimized. The insight is that by the Markov factorization that defines the trajectory distribution, the ground-truth state $s_{t+1}$ in any training trajectory data is implicitly conditional on $h_{t+1}$ and is thus an exact sample from $b_{t+1}$. Therefore, we can define the outer-loop training loss as a flow matching-style loss that upper bounds the distributional error.
\begin{equation}
    \mathcal{L}(\phi)
    \;=\;
    \E_{t\sim U\{1{:}T\}}\,
    \E_{\tau\sim U(0,1)}\,
    \E_{h_t,\, s_t,\, s_t^{(0)}}\!
    \left[
        \big\|u_{\theta_t}(s_t^{(\tau)},\tau) - (s_t - s_t^{(0)})\big\|^{2}
    \right],
    \label{eq:training-loss}
\end{equation}
$\theta_t=\cKphi^t(h_t)$ is obtained by iteratively applying $\cKphi$ as in Eq.~\eqref{eq:operator-error}, $s_t$ is the corresponding ground-truth state at step $t$, $s_t^{(0)}\sim\mathcal{N}(0,I)$ is a noise sample, and $s_t^{(\tau)} = \tau s_t + (1-\tau)s_t^{(0)}$ is the flow matching interpolant. 
We show next that $\mathcal{L}$ bounds the operator error in Eq.~\eqref{eq:operator-error} (proof in Appendix~\ref{app:loss_proof}).

\begin{theorem}[Training loss bounds the operator error]
\label{theorem:ttt_outer_loss}
Assume the velocity field $u_\theta$ and the filtering operator $K$ are Lipschitz with constants $L$ and $L_K$ respectively. Then for any horizon $T \ge 2$ there exist a constant $A>0$ and a $\phi$-independent constant $C \ge 0$ such that the operator error in Eq.~\eqref{eq:operator-error} satisfies
\begin{equation}
    \mathcal{E}(\phi) \;\leq\; A \cdot \bigl(\mathcal{L}(\phi) - C\bigr),
    \qquad
    A \;=\; 4 C_L\,(L_K^2 + 1),
    \label{eq:operator-error-bound}
\end{equation}
where $C_L$ depends only on $L$ and is the Grönwall pre-factor from Step~(1) of the appendix proof.
\end{theorem}
The bound is sharp at zero. The constant $C$ is exactly the unoptimizable variance part in the conditional flow-matching loss~\citep{lipman_flow_2023}, and does not depend on $\phi$. The remaining $\mathcal{L}(\phi) - C$ is the marginal velocity-matching term that can be driven to zero, in which case the model belief equals the true posterior at every step and $\mathcal{E}(\phi) = 0$ as well. Minimizing the conditional flow matching-style loss $\mathcal{L}(\phi)$ therefore drives a sharp upper bound on the filtering operator error $\mathcal{E}(\phi)$.

%  combines a Grönwall-type argument for the flow-matching ODE (controlling the per-step trajectory error) with a triangle inequality, the elementary inequality $(a+b)^2 \le 2(a^2+b^2)$, and Lipschitz-$K$ on the filtering recursion (lifting the trajectory error to the operator error). 

\section{Results}
\label{sec:results}

Our experiments aim to answer the following questions: 
(1) As a prerequisite of Bayesian filtering, can BFF produce nontrivial, \textit{i.e.}, multimodal posterior distributions? (2) Can BFF predict the \textit{posterior distribution} more accurately than state-of-the-art methods in high-dimensional systems with nonlinear dynamics? (3) Beyond standard super-resolution benchmarks, does \ourmethod remain reliable in observation regimes encountered in real applications, where measurements may be extremely sparse or routed through highly nonlinear observation operators?

Therefore, we empirically validate \ourmethod on a range of physical systems including the Lorenz-63 system, two 1D PDEs (Burgers' equation and the 1D Kuramoto-Sivashinsky, KS), the 2D Navier-Stokes (NS) flow past triple cylinders, and a tokamak plasma profile estimation problem. 
These benchmarks combine chaotic turbulence or stiff dynamics with sparse, moving, and line-integral observation models, yielding complex multimodal posteriors.
We compare \ourmethod with state-of-the-art learning-based filtering methods and classical filters~\citep{kalman_new_1960,evensen_ensemble_2009}, deterministic supervised learners~\citep{li_fourier_2021}, and diffusion/flow-based Bayesian filters~\citep{rozet_score-based_2023,li2024learning,chen_flowdas_2025}. 
Detailed dataset, baseline, task descriptions, and supplementary results are given in Appendices~\ref{app:dataset}, \ref{app:baselines}, and \ref{app:full_results}.

\label{sec:main-results}

\subsection{Multimodal Posterior Representation}
\label{sec:experiments}

\begin{figure}[tb]
    \centering
    \includegraphics[width=0.85\linewidth]{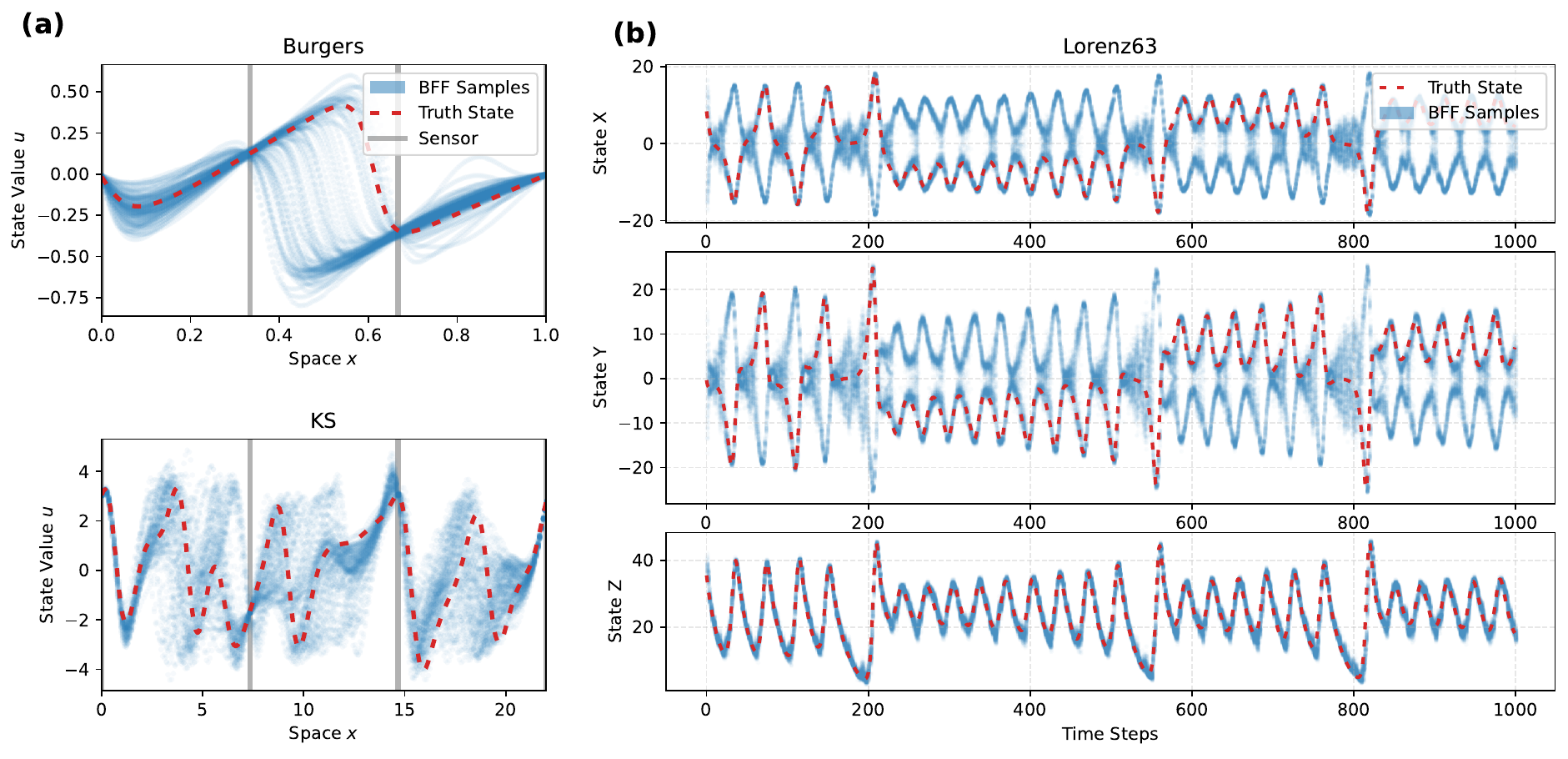}
    \caption{\textbf{Bimodal posterior visualizations.}
    \textbf{(a)} Spatial probability density on the 4-sensor Burgers' and KS equations showing spread-out and sometimes multi-modal posterior distributions.
    % (1000 samples); the bimodal mass at $x\!\approx\!0.5$ reflects shock-location uncertainty, and bright bands are sensor locations.
    \textbf{(b)} State-channel trajectories exhibiting the expected symmetric bimodal posterior on the Lorenz-63 system.}
    \label{fig:bimodal_visualizations}
\end{figure}

Bayesian filtering is challenging when the posterior is multimodal, since a single Gaussian collapses the modes into a spurious mean and distorts the underlying distribution. We demonstrate that \ourmethod{} captures such non-trivial multimodal posteriors in two settings, the Burgers' and KS equations, with only $4$ spatial sensors, and the Lorenz-63 system with only the $Z$ component observed (details in Appendix~\ref{app:dataset}).

We make two observations: 
(1) In the Lorenz-63 system, since the system is symmetrically observable ($Z$ observed, $X,Y$ hidden), an ideal filter would yield a bimodal posterior. As shown in Figure~\ref{fig:bimodal_visualizations}(b), \ourmethod successfully captures the bimodality, with one mode accurately tracking the true state.
(2) In the Burgers' and KS equations, the posterior may also be multimodal. For example, due to sparse sensor placement, there is uncertainty in shock wave propagation that creates two modes around the possible shock locations. As shown in Figure~\ref{fig:bimodal_visualizations}(a), \ourmethod captures the spread-out and sometimes multimodal posterior distributions.
Beyond bimodality, on the tokamak benchmark the residual posterior of $T_e$ exhibits sharp peaks, heavy tails, and radius-dependent correlation structure (near independence at the core, near-collinear dependence at the edge), all of which a single Gaussian cannot represent (Appendix~\ref{app:tokamak_posterior}, Figure~\ref{fig:tokamak_pair_corner_Te}).
% \ruiqi{tokamak}

\subsection{Distributional Accuracy of the Posterior}

\begin{figure}[tb]
    \centering
    \includegraphics[width=\linewidth]{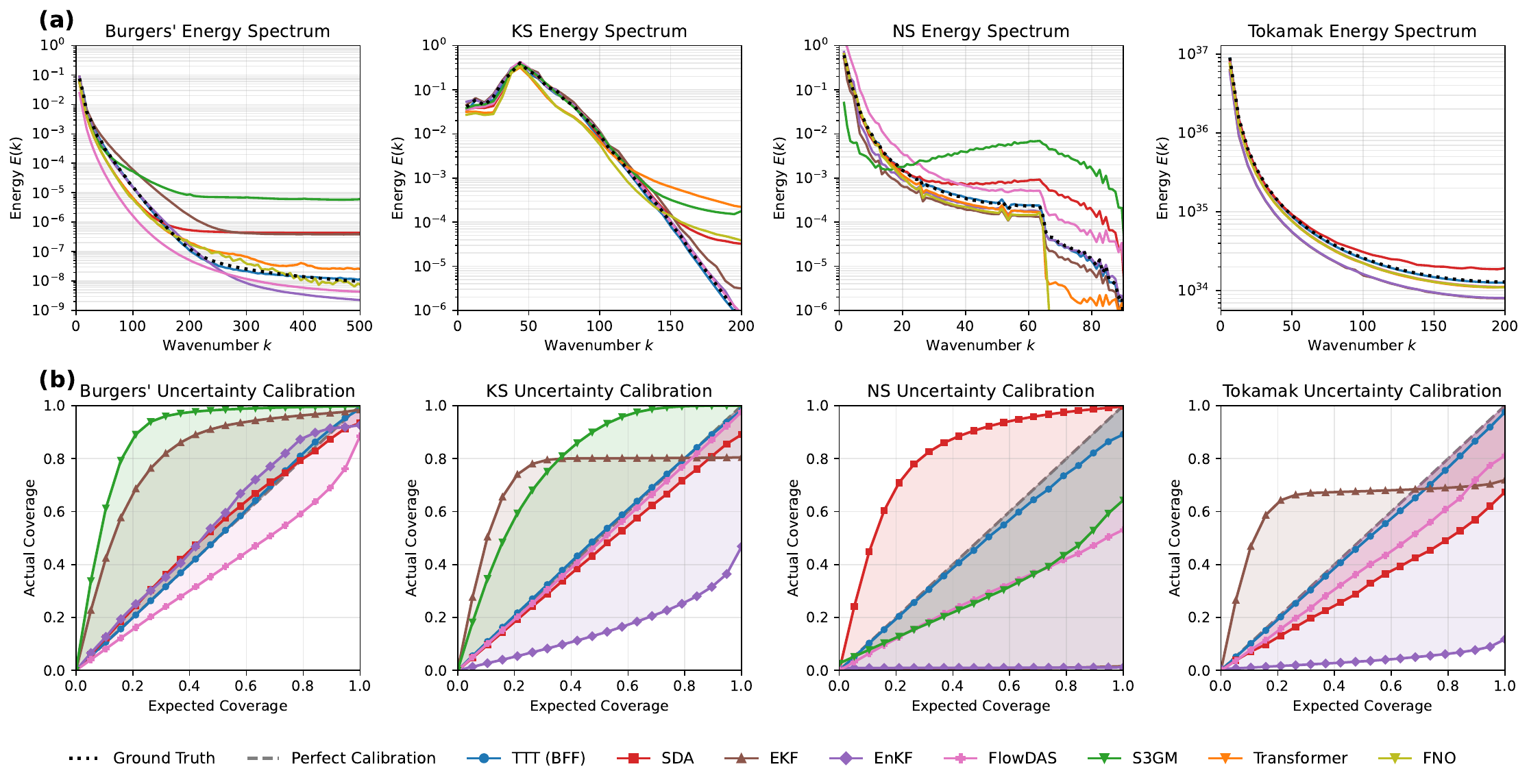}
    \caption{\textbf{Diagnostic evaluations of physical consistency and probabilistic modeling across PDE benchmarks.}
    \textbf{(a)} \ourmethod maintains spectral fidelity across scales, while baselines deviate significantly at high frequencies.
    \textbf{(b)} \ourmethod consistently provides reliable posterior distributional estimates with lower Miscalibration Area (MA), avoiding the over- or under-confidence observed in baselines.}
    \label{fig:diagnostic_metrics}
    \vspace{-0.1in}
\end{figure}

We evaluate on 1D Burgers', KS, and 2D NS (uniformly placed sensors) equations under sparse-sensing regimes (downsample ratio of $128$, $85$, and $16\times 16$) with $\sigma=0.1$ Gaussian noise. See Appendix~\ref{app:dataset} for data generation and sensor placement details. 

A Bayesian filter outputs a posterior \emph{distribution}, so evaluation should be distribution-level. We report three metrics, the state estimation error $\mathcal{E}_{L_2}$, the per-sample energy spectral error $\mathcal{E}_{\text{spec}}$ capturing marginal distribution deviation, and the Miscalibration Area (MA) capturing posterior coverage of the true posterior (computation details in Appendix~\ref{app:metrics}). A good filter must perform well on all three simultaneously.
As shown in Table~\ref{tab:all_datasets_results} and Figure~\ref{fig:diagnostic_metrics}, \ourmethod attains the best score in 8 of 9 benchmark cells across the three PDE systems, with the sole exception being spectral error on 1D KS, where \ourmethod stays within $1.6\times$ of the best Gaussian baseline (EnKF) while delivering $>50\times$ better calibration. The high $\mathcal{E}_{L_2}\!\approx\!0.91$ on 1D KS across all methods reflects benchmark difficulty rather than failure to learn, and the qualitative KS dynamics are in fact recovered by \ourmethod (Appendix~\ref{app:full_results_ks}, Figure~\ref{fig:combined_traj_ks}).

For classical filters, Ensemble Kalman Filter (EnKF) achieves the best or second-best spectral error on KS and 2D NS (uniform), but its Gaussian assumption hinders both uncertainty quantification and mean prediction. The same applies to Extended Kalman Filter (EKF). Deterministic learned models (Transformer, FNO) achieve good mean prediction but blur out high-frequency components (Figure~\ref{fig:diagnostic_metrics}~(a)) and cannot quantify uncertainty by construction.
Among generative filters, SDA~\citep{rozet_score-based_2023} captures uncertainty better than classical filters but suffers from trajectory-level posterior sampling, leaving its mean estimation and spectral consistency unstable. S$^3$GM~\citep{li2024learning} shows the same pattern. FlowDAS~\citep{chen_flowdas_2025} represents the posterior with particles, whose limited expressivity drives its unstable performance and motivates our weight-space approach.

To verify the gains come specifically from the weight-space belief representation, we replace the TTT block with a GRU-conditioned flow matching model of the same backbone and training data. The GRU variant degrades sharply on 2D NS (Appendix~\ref{app:ablation}, Figure~\ref{fig:ns_2d_ablation}), confirming that fixed-dimensional hidden states cannot encode the full posterior that weight-space memory carries.

\begin{table}[b]
    \vspace{-0.2in}
    \centering
    \footnotesize
    \caption{\textbf{Comprehensive quantitative results on the 1D KS and Burgers' and 2D NS datasets.} We report the relative $L_2$ error ($\mathcal{E}_{L_2}$), spectrum deviation ($\mathcal{E}_{\mathrm{spec}}$), and Miscalibration Area (MA). \textbf{Bold} denotes the best method and \underline{underline} denotes the second best. Transformer and FNO are deterministic models that cannot represent the posterior distribution, thus do not have MA score.
    }
    \label{tab:all_datasets_results}
    \setlength{\tabcolsep}{2.5pt} % 极限压缩列间距以容纳 13 列
    \resizebox{0.88\linewidth}{!}{% 强制缩放以适应页面宽度
    \begin{tabular}{@{}lcccccccccccc@{}}
    \toprule
    \multirow{2}{*}{\textbf{Method}} & \multicolumn{3}{c}{\textbf{1D KS}} & \multicolumn{3}{c}{\textbf{1D Burgers'}} & \multicolumn{3}{c}{\textbf{2D NS (Uniform)}} & \multicolumn{3}{c}{\textbf{2D NS (Moving)}} \\
    \cmidrule(lr){2-4} \cmidrule(lr){5-7} \cmidrule(lr){8-10} \cmidrule(lr){11-13}
     & $\mathcal{E}_{L_2}\!\!\downarrow$ & $\mathcal{E}_{\mathrm{spec}}\!\!\downarrow$ & \textbf{MA $\!\!\downarrow$} & $\mathcal{E}_{L_2}\!\!\downarrow$ & $\mathcal{E}_{\mathrm{spec}}\!\!\downarrow$ & \textbf{MA $\!\!\downarrow$} & $\mathcal{E}_{L_2}\!\!\downarrow$ & $\mathcal{E}_{\mathrm{spec}}\!\!\downarrow$ & \textbf{MA $\!\!\downarrow$} & $\mathcal{E}_{L_2}\!\!\downarrow$ & $\mathcal{E}_{\mathrm{spec}}\!\!\downarrow$ & \textbf{MA $\!\!\downarrow$} \\
    \midrule
    EKF     & 1.4186 & 0.1455 & 0.5117 & 0.3030 & 1.0408 & 0.5747 & 0.6621 & 0.2904 & 0.8301 & 0.5910 & 0.5844 & 0.8738 \\
    EnKF        & 1.0552 & \textbf{0.0376} & 0.6766 & 0.5514 & 0.3122 & 0.1017 & 0.9240 & \underline{0.1313} & 0.8682 & 0.7859 & {0.0642} & 0.7328 \\
    FNO         & 0.9873 & 0.2221 & -- & \underline{0.3274} & 0.2021 & -- & 0.3086 & 0.7974 & -- & 0.7812 & 0.1743 & -- \\
    Transformer & 0.9621 & 0.2668 & -- & 0.4338 & \underline{0.1432} & -- & 0.3202 & 0.3212 & -- & 0.7516 & 0.2766 & -- \\
    S$^3$GM     & 0.9828 & {0.1923} & 0.5449 & 0.3888 & {0.2002} & 0.7401 & {0.5432} & {1.2255} & 0.6786 & 0.7031 & \underline{0.0572} & 0.7943 \\
    SDA         & \underline{0.9622} & 0.1277 & 0.0904 & 0.3370 & 0.9140 & \underline{0.0477} & \underline{0.1951} & 0.4919 & \underline{0.5925} & \underline{0.5901} & 0.3133 & \underline{0.5112} \\
    FlowDAS     & 1.0057 & \underline{0.0502} & \underline{0.0293} & 1.1219 & 0.5201 & 0.7222 & 0.8171 & 0.4364 & 0.3938 & 1.0733 & 0.2617 & 0.9456 \\
    \midrule
    \ourmethod (ours) & \textbf{0.9061} & {0.0575} & \textbf{0.0135} & \textbf{0.2896} & \textbf{0.0441} & \textbf{0.0127} & \textbf{0.1651} & \textbf{0.0279} & \textbf{0.0643} & \textbf{0.3771} & \textbf{0.0286} & \textbf{0.1799} \\
    \bottomrule
    \end{tabular}}
\end{table}

\subsection{Realistic Applications}

\begin{figure}[t]
    \centering
    \begin{minipage}[t]{0.46\linewidth}
        \centering
        \vspace{0pt}
        \includegraphics[width=\linewidth,keepaspectratio]{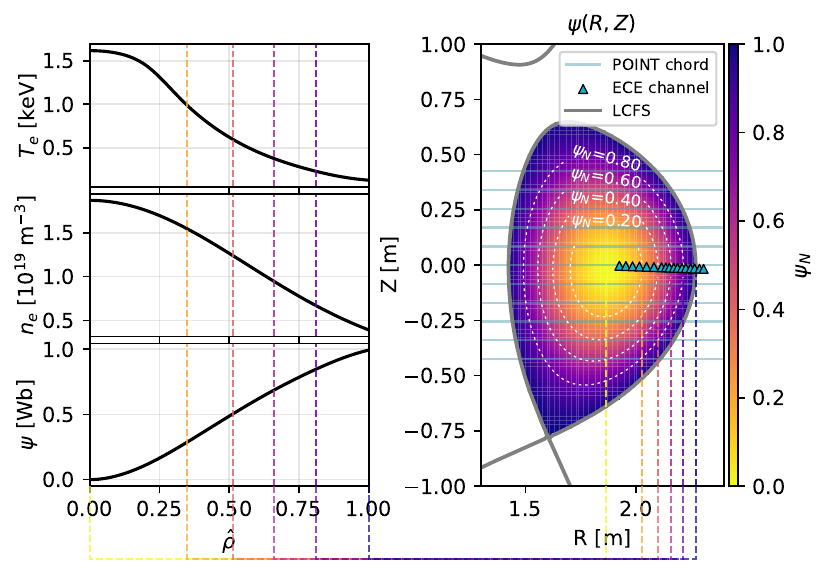}
        \captionof{figure}{\textbf{Poloidal cross-section of the tokamak in the profile estimation task.} 
        The observation includes line-integrated measurements (POINT) which is inherently nonlinear.
        }
        \label{fig:tokamak}
    \end{minipage}\hfill
    \begin{minipage}[t]{0.5\linewidth}
        \centering
        \vspace{0.5cm}
        \footnotesize
        \setlength{\tabcolsep}{4pt}
        \scalebox{1}{%
        \begin{tabular}{@{}lccccc@{}}
        \toprule
        \textbf{Method} & \textbf{sec./step $\!\!\downarrow$} & \textbf{$\mathcal{E}_{L_2}\!\!\downarrow$} & \textbf{$\mathcal{E}_{\mathrm{grad}}\!\!\downarrow$} & \textbf{$\mathcal{E}_{\mathrm{spec}}\!\!\downarrow$} & \textbf{MA $\!\!\downarrow$} \\
        \midrule
        EKF & 0.0499 & 0.1016 & 0.1657 & 0.2075 & 0.4275 \\
        EnKF & 0.0488 & 0.0870 & 0.1207 & 0.2039 & 0.9182 \\
        FNO & 0.0041 & \underline{0.0174} & 0.0779 & 0.0688 & -- \\
        Transformer & 0.0031 & \textbf{0.0170} & 0.1140 & 0.0728 & -- \\
        SDA & 3.7466 & 0.2940 & 0.9483 & 0.0788 & 0.3682 \\
        FlowDAS & 0.8553 & 0.0299 & \underline{0.0581} & \underline{0.0629} & \underline{0.2268} \\
        \midrule
        \ourmethod (ours) & 0.0369 & 0.0193 & \textbf{0.0486} & \textbf{0.0101} & \textbf{0.0225} \\
        \bottomrule
        \end{tabular}}
        \vspace{0.6cm}
        \captionof{table}{\textbf{Channel-averaged quantitative results for tokamak plasma profile estimation.} 
        % In addition to errors, we report per-step inference latency (wall-clock seconds/step).
        % the errors relative $L_2$ error ($\mathcal{E}_{L_2}$), gradient relative $L_2$ error ($\mathcal{E}_{\mathrm{grad}}$), spectrum deviation ($\mathcal{E}_{\mathrm{spec}}$), and Miscalibration Area (MA) averaged across all physical channels. 
        \textbf{Bold} denotes the best and \underline{underline} the second best.}
        \label{table:tkmk2}
    \end{minipage}
\vspace{-1em}
\end{figure}

To stress-test \ourmethod under the sparse and nonlinear observations encountered in real applications, we include two reality-oriented setups, (1) a single moving sensor scanning a 2D flow field (underwater robot exploration), and (2) tokamak plasma internal state estimation under indirect line-integrated diagnostics.

\paragraph{Single moving sensor.}

In the 2D benchmarks, we introduce a setting where a single moving sensor traverses the domain along a fixed zig-zag path to scan the entire flow field (Figure~\ref{fig:single_sensor}).
Here, the filters need to reconstruct the full $128 \times 128$ state distribution from a single measurement point with a $16384:1$ compression ratio. 
In this extremely sparsely observed problem, the gap between \ourmethod and the next baseline widens further. \ourmethod outperforms the second-best by $36\%$, $55\%$, and $65\%$ in $\mathcal{E}_{L_2}$, $\mathcal{E}_{\text{spec}}$, and MA respectively (Table~\ref{tab:all_datasets_results}). Visualizations of the predicted posterior are in Appendix~\ref{app:full_reults_ns}.

\begin{wrapfigure}{r}{0.27\textwidth}
    \centering
    \vspace{-1em} 
    \includegraphics[width=\linewidth]{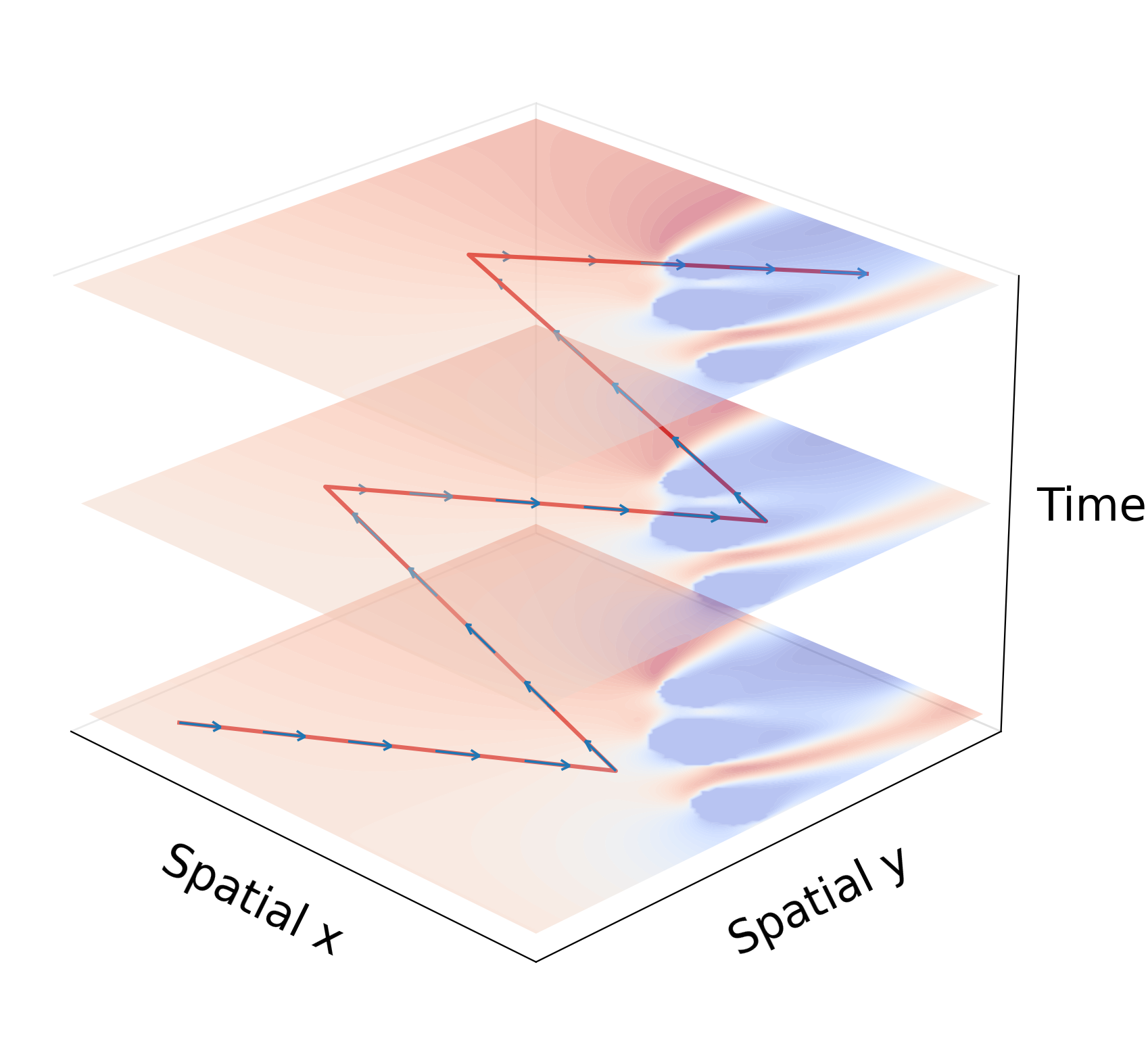}
    \caption{\textbf{Single-moving-sensor setup.} One sensor sweeps the flow field along a zig-zag trajectory.}
    \label{fig:single_sensor}
    \vspace{-1em}
\end{wrapfigure}

\paragraph{Tokamak plasma state}

Real-time access to the internal plasma profiles is a load-bearing capability for plasma control, and is increasingly critical for next-generation devices such as ITER (details in Appendix~\ref{app:tokamak_dataset}).
The state is the stack of three radial profiles $(n_e, T_e, \psi)$ on $\hat\rho\!\in\![0,1]$, evolving under stiff nonlinear transport PDEs conditioned on heating-power actions $a_t$. Diagnostics are local ECE temperature plus \emph{line-integrated} polarimetry/interferometry, making the observation operator inherently nonlinear.
Despite these challenges, \ourmethod attains the best score on three of the four metrics in Table~\ref{table:tkmk2}, namely $\mathcal{E}_{\mathrm{grad}}$, $\mathcal{E}_{\mathrm{spec}}$, and MA, where $\mathcal{E}_{\mathrm{grad}}$ is the relative-$L_2$ error in the profile gradient, critical for plasma downstream tasks. It stays within $14\%$ of the strongest baseline (Transformer), trading slightly increased mean prediction error for uncertainty estimates.

\subsection{Computation Efficiency}

\ourmethod{} reaches inference latencies compatible with the real-time control loops we target. It can be accelerated to $23$\,ms per filtering update on 2D systems and $37$\,ms per step on the tokamak benchmark (Table~\ref{table:tkmk2}), fitting within the $\sim\!100$\,ms budget of real-time tokamak plasma control. Under controlled ODE integration steps, \ourmethod{} is the fastest generative-model-based filter we benchmark, while sacrificing inference speed against deterministic baselines (FNO, Transformer) in exchange for full distributional estimates. A full discussion and ablations are provided in Appendix~\ref{app:efficiency}.

% \subsection{Probabilistic Modeling of the Posterior Distribution}
% \input{tex/results/2_uq}
% \subsection{Ablation Study of Test-time Training}
% \input{tex/results/3_ablation}
% \subsection{Generalization Test of the Learned Operator}
% \input{tex/results/4_generalize}

\section{Conclusion}
We introduced \ourmethod, a scalable Bayesian filtering framework that encodes the evolving posterior directly into the weights of a flow matching model via test-time training. By representing the belief state in weight space, \ourmethod inherently preserves the recursive structure of Bayesian filtering updates without relying on discrete particles or parametric Gaussian assumptions. Our analysis proves this update can exactly realize the Bayesian operator, with the outer-loop loss strictly bounding the $W_2$ operator error. Evaluations across chaotic PDEs and a highly non-linear tokamak plasma benchmark show that \ourmethod attains the best score in $8$ of $9$ PDE benchmark cells, retains its lead under a single moving sensor with $16384{:}1$ compression, and runs within the $\sim\!100$\,ms budget of real-time tokamak plasma control.
Limitations of \ourmethod and the corresponding future directions are discussed in Appendix~\ref{app:discussion}.

% \paragraph{Limitations.} The main limitations of \ourmethod, and the corresponding avenues for future work, are (i) scaling to \emph{3D flows} via latent-space filtering or low-rank TTT updates, (ii) accelerated inference with mean flow, rectified flow, or optimal-transport flow matching, (iii) mitigating the $30\%$ training overhead introduced by second-order gradients, and (iv) integrating collision-free \emph{active exploration} in the moving-sensor setting where observation paths heavily influence estimation accuracy. 

\bibliographystyle{plainnat}
\bibliography{references}

\clearpage
\appendix
\section{Notation}
\label{app:notation}

For ease of reading, Table~\ref{tab:notation} consolidates the symbols used throughout the paper. We follow the convention that \emph{subscripts} (\textit{e.g.}, $\theta_t,b_t$) index the discrete \emph{physical time} of the filtering process, while \emph{parenthesized superscripts} (\textit{e.g.}, $s^{(\tau)},p^{(\tau)}$) index the continuous \emph{flow time} internal to the generative model.

\begin{table}[htbp]
\centering
\small
\caption{Summary of notation used in the main paper and the appendix.}
\label{tab:notation}
\renewcommand{\arraystretch}{1.15}
\begin{tabular}{@{}p{0.27\linewidth} p{0.69\linewidth}@{}}
\toprule
\textbf{Symbol} & \textbf{Meaning} \\
\midrule
\multicolumn{2}{l}{\emph{Physical system and filtering problem}} \\
\midrule
$s_t \in \mathcal{X}\subseteq\R^{n}$ & Physical state at time $t$; $\mathcal{X}$ is the state space. \\
$o_t \in \mathcal{O}\subseteq\R^{k}$ & Observation at time $t$ (typically $k\ll n$); $\mathcal{O}$ is the observation space. \\
$a_t \in \mathcal{A}\subseteq\R^{m}$ & Control action at time $t$; $\mathcal{A}$ is the action space. \\
$h_t \coloneqq (o_{1:t}, a_{1:t-1})$ & History of past observations and actions up to time $t$. \\
$p(s_{t+1}\mid s_t,a_t)$ & Transition dynamics. \\
$p(o_t\mid s_t)$ & Observation model. \\
$b_t(s) \coloneqq p(s_t=s\mid h_t)$ & Belief state (posterior) at time $t$. \\
$\bar b_{t+1}$ & Propagated prior, the one-step pushforward of $b_t$ before the new observation. \\
$K:\mathcal{P}(\mathcal{X})\!\times\!\mathcal{O}\!\times\!\mathcal{A}\!\to\!\mathcal{P}(\mathcal{X})$ & Recursive Bayesian filtering operator: $b_{t+1} = K(b_t, o_{t+1}, a_t)$. \\
$T$ & Number of physical time steps in a trajectory. \\
\midrule
\multicolumn{2}{l}{\emph{Flow matching belief representation}} \\
\midrule
$\tau\in[0,1]$ & Flow time internal to the generative model. \\
$s^{(\tau)},\, p^{(\tau)}$ & Flow-time interpolant and its marginal distribution. \\
$u_\theta(s^{(\tau)},\tau)$ & Velocity field of the flow matching model with parameters $\theta$. \\
$\phi^{u_\theta}:\mathcal{X}\!\to\!\mathcal{X}$ & Flow map induced by integrating $u_\theta$ from $\tau{=}0$ to $\tau{=}1$. \\
$\theta_t \in \Theta$ & Flow matching weights at physical time $t$ (the BFF belief representation). \\
$b_{\theta_t}$ & Belief distribution induced by $\theta_t$. \\
$\bar b_{\theta_{t+1}}$ & Propagated model prior, the one-step pushforward of $b_{\theta_t}$ before the new observation. \\
$\Phi:\Theta\!\to\!\mathcal{P}(\mathcal{X}),\;\theta\mapsto b_\theta$ & Weight-to-belief lift. \\
\midrule
\multicolumn{2}{l}{\emph{Filtering operator in weight space}} \\
\midrule
$\cK(\theta_t,o_{t+1},a_t)$ & Weight-space filtering operator that lifts $K$ to $\Theta$: $\theta_{t+1}=\cK(\theta_t,o_{t+1},a_t)$. \\
$\cKphi$ & The $\phi$-parameterized realization of $\cK$ used by BFF. \\
$\eta$ & Inner-loop (test-time) learning rate; learned jointly with $\phi$. \\
$\sg[\,\cdot\,]$ & Stop-gradient operator. \\
\midrule
\multicolumn{2}{l}{\emph{Loss functions}} \\
\midrule
$\ell^\star(\theta)$ & Ideal filtering loss: target population flow-matching loss against $u_{K b_{\theta_t}}$ (Eq.~\eqref{eq:sec23-ideal-loss}). \\
$\hat\ell(\theta)$ & Empirical filtering loss: evaluable Monte-Carlo estimator of $\ell^\star$ (Eq.~\eqref{eq:sec23-anchor-loss}). \\
$\ell_\phi(\theta;\theta_t,o_{t+1},a_t)$ & BFF surrogate loss in latent space (Eq.~\eqref{eq:sec23-bff-surrogate}). \\
$F_N(\theta),\,G_N(\theta_t,o_{t+1},a_t)$ & Stacked model and target velocity branches of the empirical filtering loss. \\
$f_\phi,\,g_\phi$ & Learned latent heads taking values in the latent space $\cH$. \\
$\cH$ & Shared latent space of $f_\phi,g_\phi$. \\
$\Pi_N:\R^{N\times n}\!\to\!\cH$ & Latent map used in the latent amortization analysis. \\
$\mathcal{L}(\phi)$ & Outer-loop training loss (Eq.~\eqref{eq:training-loss}). \\
$\mathcal{E}(\phi)$ & Filtering operator error in $W_2^2$ (Eq.~\eqref{eq:operator-error}). \\
$N,\,N_j$ & Number of anchor samples and propagation Monte-Carlo samples. \\
$W_2$ & Wasserstein-2 distance. \\
\bottomrule
\end{tabular}
\end{table}

\section{Related Works}

\label{app:related_work}

\paragraph{Non-learning filters.}
Sequential state estimation for PDEs is primarily hindered by the difficulty of accurately representing the posterior distribution in high-dimensional, non-linear state spaces. Standard Kalman Filter variants, such as the EKF and EnKF \citep{kalman_new_1960,evensen_ensemble_2009}, rely on a parametric Gaussian assumption, which fails to capture the multi-modal or highly skewed distributions common in non-linear physical systems. Conversely, Particle Filtering \citep{gordon_novel_1993} offers a non-parametric representation of the posterior but is plagued by the "curse of dimensionality"; as the state dimension increases, the particle weights exponentially concentrate on a single sample, making standard PF computationally intractable for the discretized fields typically found in PDE simulations. To bypass these sampling issues, Variational Methods (e.g., 4D-Var or Variational Bayes) transform filtering into an optimization framework by maximizing the Evidence Lower Bound (ELBO). While these scale more effectively via adjoint-based gradients, they often converge to local optima and typically provide over-simplified uncertainty estimates, failing to recover the full, complex topology of the true posterior.

\paragraph{Diffusion- and flow matching posterior sampling-based filters.}
Following the paradigm of processing the entire trajectory for state estimation, like in variational inference, diffusion generative models like \citep{rozet_learning_2024,li2024learning} have been applied to PDE Bayesian filtering.
They execute posterior sampling at the trajectory level and utilize energy guidance of diffusion/flow matching models~\citep{ho_classifier-free_2022,feng_guidance_2025} to sample states conditioned on the observations. However, this approach has two fundamental limitations. The first one lies in the trajectory-level modeling, where the model is required to conduct posterior sampling of the entire trajectory in online filtering, which is computationally inefficient (as shown in Table~\ref{tab:efficiency}). Secondly, guidance does not, by itself, address the difficulties in high-dimensional posterior sampling. As proved in \citet{chung_diffusion_2024}, the exact guidance requires estimation of a high-dimensional integration of 
\begin{equation}
    \mathbb{E}_{x_1\sim p(x_1|x_t)}[\log p(y|x_1)],
\end{equation}
where $x_1$ denotes the clean sample, $x_t$ is the noisy sample, and $y$ is the label,
at every single integration time of diffusion sampling, and for flow matching, the same problem exists~\citep{feng_guidance_2025}. Therefore, contemporary guidance methods typically exhibit large approximation error \citep{feng_guidance_2025}. This results in unsatisfactory performance in high-dimensional complex distributions like the belief state in PDE state estimation, just as we observed in our experiments.

% using diffusion models to execute posterior sampling involves (1) the usage of guidance which introduces large approximation error; and (2) the generation of the entire trajectory which is more computationally intensive, typically requiring a larger model and more sampling steps.
% Therefore, we stick to the recursive belief posterior update formulation (Eq. \eqref{eq:bayesian_filtering}) in our work.

\paragraph{Transport- and optimal-transport-based filters.}
A complementary line of work frames the update step as learning a transport map from the prior to the posterior at each time step, including Brenier optimal-transport filters and ensemble coupling methods. A practical strength of these approaches is that the per-step map is typically close to identity, because the prior and posterior are often similar over a single filtering step, which makes the map easy to approximate. \ourmethod differs in a complementary way: instead of constructing a new transport map at every step, we keep the posterior inside the weights $\theta_t$ of a single flow matching model and evolve those weights through a meta-learned gradient step. This has two practical implications. First, once $\theta_t$ is updated, an arbitrary number of posterior samples can be drawn through the ODE integrator without re-solving an optimization problem. Second, the update is structurally Markovian at the weight level, which provides an inductive bias aligned with the recursive filtering operator and helps generalize beyond the training horizon. We therefore view transport-based and weight-space approaches as compatible perspectives on recursive filtering rather than competitors, and leave a direct empirical comparison to future work.

\paragraph{Flow-based generative models and test-time training.}
There have been rapid advances in the generative models, especially diffusion models and flow matching \cite{ho_denoising_2020,lipman_flow_2023}. They share the ideology of defining a forward and reverse probability path, where the forward path is simple and is used to conduct scalable training, while the reverse path is used to generate samples via numerical simulations of ODEs and SDEs. Flow matching has demonstrated accelerated sampling speed and better sample quality, which motivates us to adopt it in our work.

% We aim to unlock the potential in modern generative models in Bayesian filtering using online-adapted neural network weights of flow matching.
The online gradient descent algorithm in \ourmethod is closely related to the advances in the research field of Test-Time Training (TTT) \citep{sun_test-time_2020,sun_learning_2025} and Fast Weights \citep{schlag_linear_2021}, where the model parameters themselves evolve to represent the belief over the system state. This paradigm is rooted in the concept of dual-speed learning, where ``slow weights'' capture long-term dynamics and ``fast weights'' adapt to transient information, and it extends the traditional concept of associative memory~\citep{schlag_linear_2021}. Recently, TTT has also been applied to flow-based generative models \citep{dalal_one-minute_2025}.
This motivates us to treat inference as a continuous optimization process, using self-supervised ``inner loops'' to update the flow matching model's weights during the test phase.

\section{Limitations, Discussion and Future Work}
\label{app:discussion}

\paragraph{From existence to finite-capacity approximation.}
Proposition~\ref{prop:appendix-expressivity} is an existence result: it guarantees that the functional form of our surrogate loss is rich enough to represent $K$, but does not quantify the approximation error that arises when $f_\phi, g_\phi$ have finite capacity (for instance, the MLP projections we use in practice). Characterizing how this error depends on posterior complexity and embedding capacity is an interesting direction for future theoretical work. This assumption on the TTT block is a condition on its formulation, not on the specific backbone architecture.

% \paragraph{Scaling to 3D flows.}
% Our experiments target 1D and 2D systems whose full ambient state is tractable to push through a flow-matching backbone at every step. Scaling to fully 3D systems (e.g., 3D Navier--Stokes turbulence) is likely to require operating on a learned latent representation of the state rather than the ambient field, and may benefit from low-rank or LoRA-style TTT updates that reduce the per-step weight-update cost. We have not explored either direction in this paper, and they form a natural next step.

\paragraph{Accelerated inference via more efficient flow samplers.}
The midpoint solver with $5$ ODE steps already brings per-step latency on the tokamak benchmark to $\sim\!37$\,ms (Appendix~\ref{app:efficiency}), but ODE integration remains the dominant cost in our inference pipeline. Recent flow-matching variants such as mean flow, rectified flow, and optimal-transport flow matching are designed to produce nearly straight trajectories and can in principle generate high-quality samples in a handful of steps (or even one). Plugging these training objectives into \ourmethod is largely orthogonal to its weight-space update mechanism and would directly translate into faster inference. We have not pursued this in the current paper but consider it a high-value direction.

\paragraph{Training overhead from second-order gradients.}
The inner-loop second-order derivative used by the TTT update inflates the training-step cost by roughly $30\%$ over vanilla flow matching at a comparable batch size (Appendix~\ref{app:efficiency}). This overhead is moderate per step but does compound over long training runs. Reducing it via truncated backpropagation through the inner loop, implicit-differentiation tricks, or learned approximations of the weight update is a natural direction we have not pursued.

\paragraph{Weight-as-feature for downstream tasks.}
To alleviate particle-number limits on downstream tasks such as estimating statistical properties of the posterior, it is a promising direction to feed the TTT-updated weights themselves into downstream models that predict quantities of interest end-to-end. Our work demonstrates the feasibility of manipulating flow-based generative-model weights to tackle complex distributional operations in data assimilation, which opens such possibilities.

\paragraph{Designing task-specific surrogate losses.}
Proposition~\ref{prop:appendix-expressivity} indicates that our surrogate loss with sufficiently expressive $f_\phi,g_\phi$ can represent the filtering operator exactly, but our MLP or linear projections are a deliberate minimal choice. Incorporating physical structure or constraints directly into $\ell_\phi$ or the inner-loop optimization (e.g., a conservation-aware projection after each TTT step) is a natural direction that we have not pursued here.

\paragraph{Active, safety-aware sensor exploration.}
Our 2D moving-sensor benchmark uses a fixed, predefined zig-zag trajectory, which simplifies evaluation but sidesteps the practical question of how to choose informative measurement paths online. Since observation paths heavily influence estimation accuracy in our experiments, coupling \ourmethod with active-exploration algorithms that simultaneously maximize information gain and respect physical safety (e.g., collision-free trajectories around obstacles) is a key open direction.

\section{Implementation Details}
\label{app:implementation}

This section unpacks the practical instantiation of \ourmethod, walking through the backbone, the choice of online-updated weights $\theta$, the architecture of the surrogate-loss heads $f_\phi,g_\phi$, and the training pipeline. We organize the description around three concrete design questions: (i) which subset of the flow matching model constitutes $\theta$, (ii) how $f_\phi$ ingests the high-dimensional $\theta$, and (iii) how $g_\phi$ fuses the observation $o_{t+1}$ and action $a_t$.

\paragraph{Backbone: a DiT with per-layer TTT modules.}
Our flow matching velocity field $u_\theta(s^{(\tau)},\tau)$ is parameterized by a Diffusion Transformer (DiT)~\citep{peebles_scalable_2023}. Inside each transformer layer we insert a TTT module after self-attention, so that the layer applies the sequence \texttt{Attention} $\to$ \texttt{TTT} $\to$ \texttt{SSMGating}. By default, the TTT module performs a second multi-head self-attention on top of the first attention's output, parameterized by a single per-layer weight $W^{(\ell)}\in\R^{3H\times H}$ that stacks the Q/K/V projections of this second attention; the SSM-style gate then residually adds the TTT output back to the attention output. $W^{(\ell)}$ is the only online-updated parameter in the layer: it is evolved across physical time by a single gradient step on the inner surrogate loss (Eq.~\eqref{eq:sec23-bff-surrogate}). All other DiT components (token embedding, the layer's first-attention QKV/output projections, MLP sub-layers, AdaLN modulation, output head) form a static backbone shared across all physical steps. Figure~\ref{fig:ttt-arch} illustrates the integration and the per-step weight update.

\begin{figure}[t]
\centering
\begin{tikzpicture}[
    >={Latex[length=2mm,width=1.6mm]},
    font=\small,
    every node/.style={align=center, inner sep=2.5pt},
    static/.style={draw, rounded corners=2pt, fill=gray!10, minimum height=8mm, minimum width=3.0cm},
    online/.style={draw, rounded corners=2pt, fill=orange!22, minimum height=8mm, minimum width=3.0cm, very thick},
    digest/.style={draw, rounded corners=2pt, fill=white, minimum height=8mm, minimum width=5.2cm},
    bonline/.style={draw, rounded corners=2pt, fill=orange!22, minimum height=8mm, minimum width=5.2cm, very thick},
    head/.style={draw, rounded corners=2pt, fill=blue!12, minimum height=8mm, minimum width=2.4cm},
    losssty/.style={draw, rounded corners=2pt, fill=red!12, minimum height=8mm, minimum width=5.2cm},
    bigonline/.style={draw, rounded corners=2pt, fill=orange!22, minimum height=8mm, minimum width=5.2cm, very thick},
    flow/.style={->, thick},
    sg/.style={->, thick, dashed},
]
%--- Panel (a): one DiT layer (hardcoded y so titles align with panel (b)) ---
\node[font=\bfseries] (titleA) at (0, 0)    {(a) DiT layer with TTT};
\node[static]         (xin)    at (0, -1.1) {hidden state $x$};
\node[static]         (sa)     at (0, -2.2) {self-attention};
\node[online]         (ttt)    at (0, -3.3) {TTT block\\\scriptsize weight $W^{(\ell)}\!\in\!\R^{3H\times H}$};
\node[static]         (gate)   at (0, -4.4) {SSM gate};
\node[static]         (out)    at (0, -5.5) {to MLP};

\draw[flow] (xin)  -- (sa);
\draw[flow] (sa)   -- (ttt);
\draw[flow] (ttt)  -- (gate);
\draw[flow] (gate) -- (out);
\draw[flow] (sa.east) -- ++(7mm,0) |- (gate.east);

%--- Panel (b): inner-loop W update at step t (titles share y=0 with panel (a); rows match panel (a)) ---
\node[font=\bfseries] (titleB) at (7.2, 0)    {(b) Inner-loop $W^{(\ell)}$ update at step $t$};
\node[bonline]        (Wprev)  at (7.2, -1.1) {$W^{(\ell)}_{t-1}$};
\node[digest]         (dig)    at (7.2, -2.2) {digest $W^{(\ell)}_{t-1}\,x^{(\ell)}$};
\node[head]           (f)      at (5.6, -3.3) {$f_\phi(W)$};
\node[head]           (g)      at (9.0, -3.3) {$g_\phi(\sg[W_{t-1}],o_t,a_{t-1})$};
\node[losssty]        (L)      at (7.2, -4.4) {$\ell_\phi=\|f_\phi - g_\phi\|^2$};
\node[bigonline]      (Wnew)   at (7.2, -5.5) {$W^{(\ell)}_t = W^{(\ell)}_{t-1} - \eta\,\nabla_W\,\ell_\phi$};

\draw[flow] (Wprev)     -- (dig);
\draw[flow] (dig.south) -- (f.north);
\draw[sg]   (dig.south) -- (g.north);
\draw[flow] (f.south)   -- (f.south |- L.north);
\draw[flow] (g.south)   -- (g.south |- L.north);
\draw[flow] (L)         -- (Wnew);
\end{tikzpicture}
\caption{Architecture and inner-loop update of the TTT block.
\textbf{(a)} Inside every transformer layer of the DiT, a TTT module sits between self-attention and an SSM-style residual gate. It performs a \emph{second} multi-head self-attention on the first attention's output, parameterized by a single per-layer weight $W^{(\ell)}\!\in\!\R^{3H\times H}$ stacking $Q,K,V$.
\textbf{(b)} At each physical step $t$, $W^{(\ell)}$ is updated by one gradient step on the inner surrogate loss $\ell_\phi$. Both heads $f_\phi=\mathtt{theta\_proj}(W\!\cdot\!x^{(\ell)})$ and $g_\phi=\mathtt{target\_proj}([\sg(W_{t-1}\!\cdot\!x^{(\ell)}),o_t,a_{t-1}])$ share the digest $W^{(\ell)}_{t-1} x^{(\ell)}$, but $g_\phi$ stop-gradients it (dashed arrow), realizing the asymmetric dependence on $\theta$ in Eq.~\eqref{eq:sec23-bff-surrogate}. Orange highlights the online-updated $W^{(\ell)}$; gray boxes are the static backbone shared across all physical steps.}
\label{fig:ttt-arch}
\end{figure}
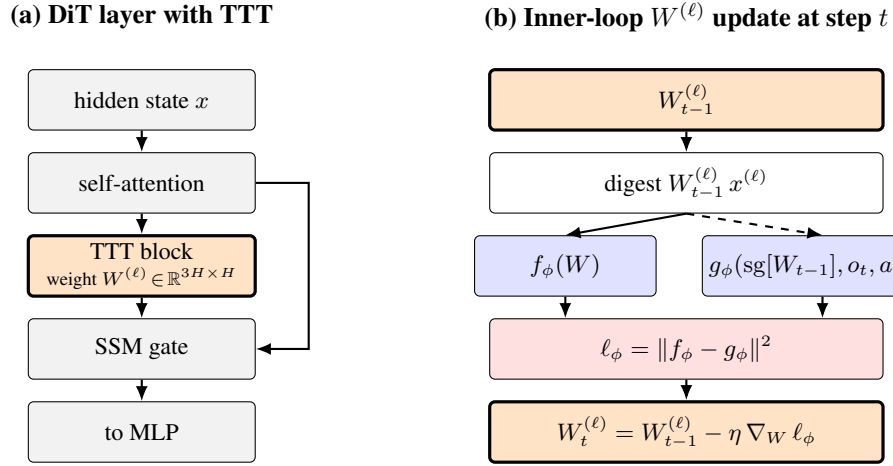

\paragraph{Online-updated weights $\theta$.}
The online state $\theta$ comprises only the per-layer TTT weight matrices $W^{(\ell)}\in\R^{3H\times H}$, one per transformer layer (with $H$ the model hidden size); each $W^{(\ell)}$ stacks the query, key, and value projections that drive the second self-attention computed on top of the layer's standard attention output. The full belief representation is the concatenation $\theta=\big(W^{(1)},\dots,W^{(L)}\big)$, where $L$ is the number of transformer layers. Each $W^{(\ell)}$ is updated independently of the others by the inner step, so the per-step gradient computation parallelizes across layers and depth. The static DiT backbone (token embedding, regular attention QKV/output, MLP sub-layers, AdaLN, output head) is \emph{not} part of $\theta$ and is shared across all physical steps; it encodes the prior knowledge of the state distribution and dynamics learned during outer training. With our defaults ($H{=}256$, $L{=}6$), $\theta$ contains $\sim\!1.2\!\times\!10^6$ scalars, which is of the same order as the static backbone but small enough that a single backward pass through the inner loss is the dominant per-step cost.

\paragraph{Per-layer auxiliary parameters.}
Three additional per-layer parameter groups parameterize the inner objective and are trained \emph{only} by the outer loss:
\begin{itemize}
  \item An initial weight $W_0^{(\ell)}\in\R^{*\times H}$ from which the inner loop starts at the beginning of every trajectory.
  \item A scalar inner step size $\eta^{(\ell)}\in\R$, initialized to $0.01$ and learned jointly with $\phi$.
  \item A learnable probe vector $x^{(\ell)}\in\R^{H}$ that summarizes $W^{(\ell)}$ via the contraction $W^{(\ell)} x^{(\ell)}$ (see below).
\end{itemize}
We collect all three groups together with the heads $f_\phi,g_\phi$ defined next under the meta-parameter symbol $\phi$ used in the main text.

\paragraph{Surrogate-loss heads $f_\phi$ and $g_\phi$.}
A direct projection from the high-dimensional matrix $W^{(\ell)}\in\R^{*\times H}$ into a feature space would be parameter-heavy and break differentiability through the inner step. We instead summarize $W^{(\ell)}$ by its action on the learnable probe $x^{(\ell)}$:
\begin{equation*}
  \mathrm{pred}^{(\ell)} \;\coloneqq\; W^{(\ell)} x^{(\ell)} \;\in\; \R^{*},
  \qquad * \in \{H,\ 3H\},
\end{equation*}
which is a smooth, low-dimensional digest of the full matrix that retains the linear dependence on $W^{(\ell)}$ needed for the inner gradient.

The model channel $f_\phi$ then projects $\mathrm{pred}^{(\ell)}$ into observation space:
\begin{equation*}
  f_\phi^{(\ell)}(\theta) \;=\; \mathtt{theta\_proj}^{(\ell)}\!\big(W^{(\ell)} x^{(\ell)}\big) \;\in\; \R^{\mathrm{obs\_dim}},
\end{equation*}
where $\mathtt{theta\_proj}^{(\ell)}$ is a per-layer MLP (or, optionally, a linear map).

The target channel $g_\phi$ fuses the three inputs $(\sg[\theta],o_{t+1},a_t)$ by raw feature-dimension concatenation followed by a single projection:
\begin{equation*}
  g_\phi^{(\ell)}(\sg[\theta],o_{t+1},a_t)
  \;=\;
  \mathtt{target\_proj}^{(\ell)}\!\Big(
    \big[\,\mathrm{sg}\!\big(W^{(\ell)} x^{(\ell)}\big),\; o_{t+1},\; a_t\,\big]
  \Big)
  \;\in\; \R^{\mathrm{obs\_dim}}.
\end{equation*}
Three points are worth noting. First, the heads are valued in observation space, so the latent space $\cH$ in the main text is concretely $\R^{\mathrm{obs\_dim}}$. Second, the weight digest $W^{(\ell)} x^{(\ell)}$ entering $g$ is detached: the stop-gradient operator $\sg[\cdot]$ blocks the inner-loop gradient from flowing into $g$'s side of the residual, so the inner update direction is shaped only by $f$'s $W^{(\ell)}$ path, while $g$'s own weights $\mathtt{target\_proj}^{(\ell)}$ remain in the outer-loop graph and are still trained by the outer flow-matching loss. Third, the $g$-side encoding of $o_{t+1}$ and $a_t$ is intentionally minimal (raw concatenation followed by a single per-layer linear or MLP projection), because the asymmetry between $f$ (which depends on the live $\theta$) and $g$ (which depends only on the new information) is what makes a single inner gradient step on $\ell_\phi$ realize a Bayesian-style update; richer $g$-side encoders are admissible but unnecessary in practice.

\paragraph{Inner update at inference.}
Per layer, at every physical step, we perform the differentiable single gradient step (Algorithm~\ref{alg:ttt_test})
\begin{equation*}
  W^{(\ell)}_{t+1}
  \;=\;
  W^{(\ell)}_{t} \;-\; \eta^{(\ell)}\,\nabla_{W^{(\ell)}_{t}}\,
  \big\|\,f_\phi^{(\ell)}(\theta_t) - g_\phi^{(\ell)}(\sg[\theta_t], o_{t+1}, a_t)\,\big\|^{2},
\end{equation*}
implemented by \texttt{torch.autograd.grad} with \texttt{create\_graph=True} so that the inner update is itself differentiable for the outer-loop training described next. We restrict to a single inner step both for efficiency and because the expressivity result in Proposition~\ref{prop:appendix-expressivity} already covers single-step updates.

\paragraph{Training pipeline.}
Outer training proceeds in two stages.
\textbf{(1) Pretraining the static backbone.} We pretrain the DiT backbone together with $W_0$ on the marginal state distribution $p(s)$ with a standard flow-matching loss before outer training, which gives the inner loop a better starting point at small cost since no spatio-temporal dynamics are involved at this stage.
\textbf{(2) Outer-loop training.} We then jointly train all parameter groups (DiT backbone, $W_0^{(\ell)}$, $\eta^{(\ell)}$, $x^{(\ell)}$, $f_\phi$, and $g_\phi$) with the unrolled outer-loop loss Eq.~\eqref{eq:training-loss}. The bottleneck of this stage is the DiT forward pass through the unrolled trajectory rather than the lightweight TTT modules; we therefore subsample physical time steps when evaluating the flow-matching loss and initialize the gating with $\alpha\!=\!0.1$ to stabilize early outer-loop dynamics. We optimize with AdamW (learning rate $10^{-4}$, weight decay $0$, gradient clip $1.0$) using a cosine schedule with $1{,}000$-step linear warmup; $\eta^{(\ell)}$ uses a $10\!\times\!$ smaller learning rate to keep the inner step size stable in early training.

\paragraph{Hyperparameters and efficiency.}
Per-experiment hyperparameter choices (DiT depth/width, training horizon $T$, batch size, total outer steps) are reported in our released code, and sensor configurations and dataset specifics are reported in Appendix~\ref{app:dataset}. As reported in Table~\ref{tab:efficiency}, the per-step inference overhead from the TTT update is negligible (one extra backward pass restricted to the per-layer $W^{(\ell)}$ matrices), and outer-training wall-clock time is comparable to standard flow-matching training of a DiT of the same backbone size.

\section{BFF Algorithm Pseudocode}
The pseudo-code for the training and inference loops of \ourmethod{} are summarized in Algorithm~\ref{alg:ttt_train} and Algorithm~\ref{alg:ttt_test}, respectively.

\begin{algorithm}[tb]
   \caption{\ourmethod{} Training.}
   \label{alg:ttt_train}
\begin{algorithmic}
    \STATE {\bfseries Input:} Dataset $s_{[1:T]}, o_{[1:T]}, a_{[1:T-1]} \sim P_{\text{data}}$.
    \STATE {\bfseries Output:} Trained $\phi, \theta_0, \eta$.
    \STATE Pretrain $\theta_0$ on $s_{t}\sim P_{\text{data}},~ t\sim U\{1:T\}$ using standard flow matching loss
    \FOR{$i=0$ {\bfseries to} train steps  \hfill $\triangleright$ Outer Loop \renewcommand{\algorithmicdo}{}}
    \STATE $\mathcal{L} \gets 0$ 
    \FOR{$t = 1$ {\bfseries to} $T$ {\bfseries do} \hfill $\triangleright$ Inner Loop \renewcommand{\algorithmicdo}{}}
        \STATE $\theta_{t}\gets \theta_{t-1} - \eta \nabla_{\theta_{t-1}} \|f_\phi(\theta_{t-1}) - g_\phi(\sg[\theta_{t-1}], o_t, a_{t-1})\|^2$
        \IF{$t \in$ loss steps}
            \STATE Sample $\tau\sim U(0,1),~ s_t^{(0)}\sim\mathcal{N}(0,I)$
            \STATE $s_t^{(\tau)} \gets \tau s_t^{(1)} + (1-\tau) s_t^{(0)}$
            \STATE $\mathcal{L} \gets \mathcal{L} + \|u_{\theta_t}(s_t^{(\tau)},\tau) - (s_t^{(1)} - s_t^{(0)})\|^2$
        \ENDIF
    \ENDFOR
    \STATE Optimize $\phi,\theta_0,\eta$ with $\nabla_{\phi,\theta_0,\eta}\mathcal{L}$
    \ENDFOR
    \renewcommand{\algorithmicdo}{\textbf{do}}
\end{algorithmic}
\end{algorithm}

\begin{algorithm}[tb]
   \caption{\ourmethod{} Inference.}
   \label{alg:ttt_test}
\begin{algorithmic}
    \STATE {\bfseries Input:} Observation sequence $o_{[1:T]}, a_{[1:T-1]}$, trained model $\phi, \theta_0, \eta$.
    \STATE {\bfseries Output:} Sampled states $\hat{s}_{[1:T]}$.
    \FOR{$t = 1$ {\bfseries to} $T$ {\bfseries do} \hfill $\triangleright$ Physical Time Loop \renewcommand{\algorithmicdo}{}}
        \STATE $\theta_{t}\gets \theta_{t-1} - \eta \nabla_{\theta} \ell_\phi(\theta;\theta_{t-1}, o_t, a_{t-1})\big|_{\theta=\theta_{t-1}}$
        \STATE Sample $s_t^{(0)} \sim \mathcal{N}(0, I)$
        \FOR{$\tau=0$ {\bfseries to} $1$ {\bfseries step} $d\tau$ {\bfseries do} \hfill $\triangleright$ Flow ODE Integration \renewcommand{\algorithmicdo}{}}
            \STATE $s_t^{(\tau+d\tau)}\gets s_t^{(\tau)} + u_{\theta_t}(s_t^{(\tau)},\tau) \cdot d\tau$
        \ENDFOR
        \STATE $\hat{s}_t \gets s_t^{(1)}$
    \ENDFOR
    \renewcommand{\algorithmicdo}{\textbf{do}}
\end{algorithmic}
\end{algorithm}

\section{Surrogate (Inner) Loss Design}
\label{app:two-channel}

This appendix provides the proofs of the lemmas and proposition stated in Section~\ref{sec:ttt}. 

\subsection{Empirical Filtering Loss}
\label{app:two-channel-anchor}

This subsection proves the main-text Proposition~\ref{prop:sec23-emp-filt} of the \emph{Empirical filtering loss} paragraph. Three helper facts are used: a score--velocity affine bijection that is standard in flow matching, an additive split of the Bayesian update at the score level, and a conditional flow-matching identity that turns the propagation channel into an unbiased target.

\begin{lemma}[Score--velocity affine bijection]
\label{lem:appA-score-velocity}
For the linear-Gaussian probability path $s^{(\tau)}=\tau s+(1-\tau)s^{(0)}$ with $s\sim q$ and $s^{(0)}\sim\mathcal{N}(0,I)$, the marginal velocity field $u_q(s,\tau)=\E[s'-s^{(0)}\given s^{(\tau)}=s]$ and the marginal score $\nabla_s\log q^{(\tau)}(s)$ are related by
\begin{equation}
    u_q(s,\tau)
    \;=\;
    \frac{s}{\tau}
    +
    \frac{1-\tau}{\tau}\,\nabla_s\log q^{(\tau)}(s),
    \qquad
    \nabla_s\log q^{(\tau)}(s)
    \;=\;
    \frac{\tau\,u_q(s,\tau)-s}{1-\tau},
    \label{eq:appA-affine}
\end{equation}
for any $\tau\in(0,1)$ and any $q\in\mathcal{P}(\mathcal{X})$ with differentiable marginal $q^{(\tau)}$.
\end{lemma}

\begin{proof}
The path identity $s^{(\tau)}=\tau s'+(1-\tau)s^{(0)}$ rearranges to $s'-s^{(0)}=(s'-s^{(\tau)})/(1-\tau)$, hence taking conditional expectations gives
\begin{equation}
    u_q(s,\tau)
    \;=\;
    \frac{\E[s'\given s^{(\tau)}=s]-s}{1-\tau}.
    \label{eq:appA-velocity-as-mean}
\end{equation}
The conditional density of $s^{(\tau)}$ given $s'$ is $\mathcal{N}(\tau s',(1-\tau)^2 I)$. We use the following form of Tweedie's formula~\citep{efron2011tweedie}: if $X=\mu+\sigma Z$ with $\mu\sim\pi$, $Z\sim\mathcal{N}(0,I)$ independent, and marginal density $m(x)=\int\mathcal{N}(x;\mu,\sigma^2 I)\,\pi(\mu)\,d\mu$, then
\begin{equation}
    \E[\mu\given X=x]\;=\;x+\sigma^{2}\,\nabla_x\log m(x).
    \label{eq:appA-tweedie-general}
\end{equation}
We apply Eq.~\eqref{eq:appA-tweedie-general} with $X=s^{(\tau)}$, $\mu=\tau s'$ (so the prior on $\mu$ is the pushforward of $q$ under $s'\mapsto\tau s'$, and the marginal of $X$ is exactly $q^{(\tau)}$), and $\sigma=1-\tau$. This yields
\begin{equation}
    \E[\tau s'\given s^{(\tau)}=s]\;=\;s+(1-\tau)^{2}\,\nabla_s\log q^{(\tau)}(s),
    \label{eq:appA-tweedie-applied}
\end{equation}
and dividing by $\tau$ gives
\begin{equation}
    \E[s'\given s^{(\tau)}=s]
    \;=\;
    \frac{s}{\tau}+\frac{(1-\tau)^2}{\tau}\,\nabla_s\log q^{(\tau)}(s).
    \label{eq:appA-tweedie}
\end{equation}
Substituting Eq.~\eqref{eq:appA-tweedie} into Eq.~\eqref{eq:appA-velocity-as-mean} gives the first identity in Eq.~\eqref{eq:appA-affine}; the second is its inverse.
\end{proof}

The map Eq.~\eqref{eq:appA-affine} is the technical bridge between scores and velocities under the FM probability path: every additive decomposition at the score level induces an additive decomposition at the velocity level with gain factor $\alpha(\tau)=(1-\tau)/\tau$.

\begin{lemma}[Bayes split at the score level]
\label{lem:appA-bayes-score}
Let $\pi\in\mathcal{P}(\mathcal{X})$ be a prior with differentiable density and $p(o\given s)$ a likelihood with finite normalization $Z=\int p(o\given s)\,\pi(s)\,ds\in(0,\infty)$. Define the Bayes posterior $\pi^o(s):=\pi(s)\,p(o\given s)/Z$. Then
\begin{equation}
  \nabla_s\log \pi^o(s)\;-\;\nabla_s\log\pi(s)
  \;=\;
  \nabla_s\log p(o\given s).
  \label{eq:appA-score-correction}
\end{equation}
\end{lemma}

\begin{proof}
Bayes' formula gives
\begin{equation}
    \log \pi^o(s) - \log\pi(s)
    \;=\;
    \log p(o\given s) - \log Z.
    \label{eq:appA-log-split}
\end{equation}
The normalization $\log Z$ does not depend on $s$, so differentiating in $s$ removes it and yields Eq.~\eqref{eq:appA-score-correction}.
\end{proof}

In the inner-loss derivation we apply Lemma~\ref{lem:appA-bayes-score} to the prior--posterior pair $(\bar b_{\theta_{t+1}},\,K b_{\theta_t})$ defined relative to the model belief $b_{\theta_t}$, where $\bar b_{\theta_{t+1}}(s):=\int p(s\given s',a_t)\,b_{\theta_t}(s')\,ds'$ is the propagated model prior and $K b_{\theta_t}$ is the Bayes update of $b_{\theta_t}$. The same identity Eq.~\eqref{eq:filter-operator} applies with $b_t$ replaced by $b_{\theta_t}$.

\begin{lemma}[Conditional flow-matching identity]
\label{lem:appB-condFM}
Let $s'\sim b_{\theta_t}$, $s^p\sim p(\,\cdot\,\given s',a_t)$, $s^{(0)}\sim\mathcal{N}(0,I)$, $\tau\sim U(0,1)$ (mutually independent given $s'$), and $s^{(\tau)}=\tau s^p+(1-\tau)s^{(0)}$ as in the anchor sampling of Section~\ref{sec:ttt}. Then
\begin{equation}
    \E\!\left[s^p - s^{(0)}\,\Big|\, s^{(\tau)}=s,\;\tau\right]
    \;=\;
    u_{\bar b_{\theta_{t+1}}}(s,\tau).
    \label{eq:appB-condFM}
\end{equation}
\end{lemma}

\begin{proof}
We first identify the marginal law of the propagated anchor $s^p$, then match it to the linear-Gaussian probability path of Lemma~\ref{lem:appA-score-velocity}, and conclude by reading off the conditional mean.

Since $s'\sim b_{\theta_t}$ and $s^p\given s'\sim p(\,\cdot\,\given s',a_t)$, the marginal law of $s^p$ is, by the tower property,
\begin{equation*}
    \mathrm{Law}(s^p)
    \;=\;
    \int p(\,\cdot\given s',a_t)\,b_{\theta_t}(s')\,ds'
    \;=\;
    \bar b_{\theta_{t+1}},
\end{equation*}
the propagated model prior obtained by applying the prediction step Eq.~\eqref{prediction} to $b_{\theta_t}$ in place of $b_t$. Thus $s^p\sim\bar b_{\theta_{t+1}}$, and by construction $s^{(0)}\sim\mathcal{N}(0,I)$ is independent of $s^p$ (and of $\tau$). The triple $(s^p,s^{(0)},\tau)$ therefore realizes exactly the linear-Gaussian probability path of Lemma~\ref{lem:appA-score-velocity} with data law $q=\bar b_{\theta_{t+1}}$, anchor noise $s^{(0)}$, and path point $s^{(\tau)}=\tau s^p+(1-\tau)s^{(0)}$.

Specializing Eq.~\eqref{eq:appA-velocity-as-mean} to $q=\bar b_{\theta_{t+1}}$ gives
\begin{equation*}
    u_{\bar b_{\theta_{t+1}}}(s,\tau)
    \;=\;
    \frac{\E[s^p\given s^{(\tau)}=s,\,\tau]-s}{1-\tau}.
\end{equation*}
On the other hand, the path identity $s^{(\tau)}=\tau s^p+(1-\tau)s^{(0)}$ rearranges as $s^p-s^{(0)}=(s^p-s^{(\tau)})/(1-\tau)$, so taking conditional expectation in $(s^{(\tau)},\tau)$,
\begin{equation*}
    \E\!\left[s^p-s^{(0)}\,\Big|\, s^{(\tau)}=s,\,\tau\right]
    \;=\;
    \frac{\E[s^p\given s^{(\tau)}=s,\,\tau]-s}{1-\tau}.
\end{equation*}
The two right-hand sides coincide, yielding Eq.~\eqref{eq:appB-condFM}.
\end{proof}

\paragraph{Anatomy of the empirical filtering loss.}
We restate the constituents of Eqs.~\eqref{eq:sec23-anchor-loss}--\eqref{eq:sec23-FG} with the dependencies the proof will use. For each anchor $i\in\{1,\dots,N\}$, draw a particle $s_i'\sim b_{\theta_t}$ from the current model belief, simulate one transition step $s_i^p\sim p(\,\cdot\given s_i',a_t)$, and pair it with independent draws $s_i^{(0)}\sim\mathcal{N}(0,I)$ and $\tau_i\sim U(0,1)$, giving the path point $s_i^{(\tau_i)}=\tau_i s_i^p+(1-\tau_i)s_i^{(0)}$. By the tower-property argument of Lemma~\ref{lem:appB-condFM}, the marginal law of $s_i^p$ is the propagated model prior $\bar b_{\theta_{t+1}}$, so $(s_i^p,s_i^{(0)},\tau_i,s_i^{(\tau_i)})$ realizes the linear-Gaussian probability path of Lemma~\ref{lem:appA-score-velocity} with data law $\bar b_{\theta_{t+1}}$.

The empirical filtering loss compares two branches at these path points. The model branch $F_N$ stacks $u_\theta(s_i^{(\tau_i)},\tau_i)$, the standard flow-matching velocity to be fit. The target branch $G_N$ stacks a two-channel target. The first is the \emph{propagation term} $s_i^p-s_i^{(0)}$, which by Lemma~\ref{lem:appB-condFM} is conditionally unbiased for the prior velocity $u_{\bar b_{\theta_{t+1}}}(s_i^{(\tau_i)},\tau_i)$. The second is the \emph{likelihood term} $\alpha(\tau_i)\,\nabla_s\log\bar p^{(\tau_i)}(o_{t+1}\given s)\big|_{s=s_i^{(\tau_i)}}$, with prefactor $\alpha(\tau)=(1-\tau)/\tau$, that converts the Bayesian score correction into a velocity contribution via the score-to-velocity affine map of Lemma~\ref{lem:appA-score-velocity}. The path-smoothed likelihood $\bar p^{(\tau)}$ is the central object that the proof will introduce in Eq.~\eqref{eq:appA-pbar-def} below; in practice it is replaced by the self-normalized Monte-Carlo estimator with $N_j$ extra samples drawn from $\bar b_{\theta_{t+1}}$, independent of $s_i^p$, whose explicit form appears in Eq.~\eqref{eq:appA-pbar-mc}.

We now state the full version of Proposition~\ref{prop:sec23-emp-filt} and prove it.

\begin{proposition}[Empirical filtering loss, full version]
\label{prop:appA-emp-filt-full}
Assume the densities in Eq.~\eqref{eq:filter-operator} are differentiable, the path-smoothed likelihood $\bar p^{(\tau)}(o_{t+1}\given s)$ is differentiable and bounded away from zero in $s$, and the model velocity $u_\theta$ is uniformly bounded with bounded $\theta$-derivative on the compact $\theta$-domain of interest. Let $\hat\ell$ be the empirical filtering loss Eq.~\eqref{eq:sec23-anchor-loss} in which $\nabla_s\log\bar p^{(\tau)}$ is replaced by its $N_j$-sample self-normalized Monte-Carlo estimator. Then there exist constants $L_1,L_2,c_5\ge 0$ and a $\theta$-independent $C$ depending only on the prior, observation, and dynamics models such that, for every $\theta$ in this domain,
\begin{equation*}
    \big|\E[\hat{\ell}(\theta)]-\ell^{\star}(\theta)-C\big|\;\le\;\frac{L_1}{N_j},
    \qquad
    \big\|\nabla_\theta\E[\hat{\ell}(\theta)]-\nabla_\theta\ell^{\star}(\theta)\big\|\;\le\;\frac{L_2}{N_j},
\end{equation*}
and
\begin{equation*}
    \mathrm{Var}\!\big(\hat{\ell}(\theta)\big)
    \;\le\;
    \frac{1}{N}\Big(\sigma^{2}+\frac{c_5}{N_j}\Big),
\end{equation*}
where $\sigma^{2}$ is the per-anchor residual variance under closed-form $\bar p^{(\tau)}$. The $1/N_j$ bias and variance contributions vanish identically whenever $\bar p^{(\tau)}$ admits a closed form (e.g.\ Gaussian observation with Gaussian prior).
\end{proposition}

\begin{proof}
We first establish a propagation--likelihood decomposition of the target velocity at each flow time $\tau$, then check that the per-anchor stack of Eq.~\eqref{eq:sec23-FG} is conditionally unbiased for this target, and finally promote pointwise unbiasedness to a variance bound on the empirical loss.

\paragraph{Bayesian velocity decomposition.}
At the data level $\tau=1$, applying Lemma~\ref{lem:appA-bayes-score} to the prior--posterior pair $(\bar b_{\theta_{t+1}},\,K b_{\theta_t})$ gives the additive split $\nabla_s\log K b_{\theta_t}-\nabla_s\log\bar b_{\theta_{t+1}}=\nabla_s\log p(o_{t+1}\given s)$. To carry this split through the path, we must extend it to the marginal scores at $\tau\in(0,1)$. Let $\kappa_\tau(s\given s'):=\mathcal{N}(s;\tau s',(1-\tau)^2 I)$ be the conditional density of $s^{(\tau)}$ given the data anchor $s'$. Then
\begin{align}
    \bar b_{\theta_{t+1}}^{(\tau)}(s)
    &=\int\bar b_{\theta_{t+1}}(s')\,\kappa_\tau(s\given s')\,ds',
    \label{eq:appA-prior-marginal}\\
    (K b_{\theta_t})^{(\tau)}(s)
    &=\frac{1}{Z_{t+1}}\int p(o_{t+1}\given s')\,\bar b_{\theta_{t+1}}(s')\,\kappa_\tau(s\given s')\,ds'
    =\bar b_{\theta_{t+1}}^{(\tau)}(s)\cdot\frac{\E_{\bar b_{\theta_{t+1}}^{(\tau)}(s'\given s)}\!\big[p(o_{t+1}\given s')\big]}{Z_{t+1}},
    \label{eq:appA-post-marginal}
\end{align}
where $\bar b_{\theta_{t+1}}^{(\tau)}(s'\given s):=\bar b_{\theta_{t+1}}(s')\,\kappa_\tau(s\given s')/\bar b_{\theta_{t+1}}^{(\tau)}(s)$ is the conditional law of the data anchor $s'$ given the path-level sample $s$ under the prior, and $Z_{t+1}=\int p(o_{t+1}\given s')\,\bar b_{\theta_{t+1}}(s')\,ds'$ is the normalization of $K b_{\theta_t}$. Taking $\nabla_s\log$ of the ratio in Eq.~\eqref{eq:appA-post-marginal} and noting that $\log Z_{t+1}$ does not depend on $s$,
\begin{equation}
  \nabla_s\log (K b_{\theta_t})^{(\tau)}(s)\;-\;\nabla_s\log\bar b_{\theta_{t+1}}^{(\tau)}(s)
  \;=\;
  \nabla_s\log\E_{\bar b_{\theta_{t+1}}^{(\tau)}(s'\given s)}\!\big[p(o_{t+1}\given s')\big].
  \label{eq:appA-marg-score-diff-exact}
\end{equation}
Eq.~\eqref{eq:appA-marg-score-diff-exact} motivates the central object of this proof, the \emph{path-smoothed likelihood} at flow time $\tau$:
\begin{equation}
    \bar p^{(\tau)}(o_{t+1}\given s)
    \;:=\;
    \E_{\bar b^{(\tau)}_{\theta_{t+1}}(s'\given s)}\!\big[p(o_{t+1}\given s')\big]
    \;=\;
    \frac{\int p(o_{t+1}\given s')\,\bar b_{\theta_{t+1}}(s')\,\kappa_\tau(s\given s')\,ds'}{\int \bar b_{\theta_{t+1}}(s')\,\kappa_\tau(s\given s')\,ds'},
    \label{eq:appA-pbar-def}
\end{equation}
which averages the data-level likelihood against the kernel-weighted conditional law of the data anchor given the path point. With this definition, Eq.~\eqref{eq:appA-marg-score-diff-exact} reads
\begin{equation}
  \nabla_s\log (K b_{\theta_t})^{(\tau)}(s)\;-\;\nabla_s\log\bar b_{\theta_{t+1}}^{(\tau)}(s)
  \;=\;
  \nabla_s\log\bar p^{(\tau)}(o_{t+1}\given s),
  \label{eq:appA-marg-score-diff-pbar}
\end{equation}
and is exact for \emph{any} observation model and \emph{any} prior $\bar b_{\theta_{t+1}}$, with the only required regularity being differentiability of $\bar p^{(\tau)}$ in $s$. The data anchor is the point limit: at $\tau=1$ the kernel collapses to $\kappa_1(s\given s')=\delta(s-\tau s')$, so $\bar b^{(\tau)}_{\theta_{t+1}}(s'\given s)$ concentrates on $s'=s$, hence $\bar p^{(1)}(o_{t+1}\given s)=p(o_{t+1}\given s)$ and Eq.~\eqref{eq:appA-marg-score-diff-pbar} reduces to the data-level Bayes split (Lemma~\ref{lem:appA-bayes-score}). For $\tau<1$, $\bar p^{(\tau)}$ is the genuine likelihood object on the path-level marginal $\bar b^{(\tau)}_{\theta_{t+1}}$; the data-level $p(o_{t+1}\given s)$ recovers it only in special cases (e.g.\ Gaussian observation with Gaussian prior) and is otherwise an uncontrolled approximation, which is why we work with $\bar p^{(\tau)}$ throughout.

Applying Lemma~\ref{lem:appA-score-velocity} separately to $q=K b_{\theta_t}$ and $q=\bar b_{\theta_{t+1}}$ and subtracting (the deterministic $s/\tau$ term cancels) yields
\begin{equation*}
    u_{K b_{\theta_t}}(s,\tau)
    \;-\;
    u_{\bar b_{\theta_{t+1}}}(s,\tau)
    \;=\;
    \frac{1-\tau}{\tau}\,\big[\nabla_s\log (K b_{\theta_t})^{(\tau)}(s)-\nabla_s\log\bar b_{\theta_{t+1}}^{(\tau)}(s)\big].
\end{equation*}
Combining with Eq.~\eqref{eq:appA-marg-score-diff-pbar} gives the exact propagation--likelihood decomposition of the target velocity,
\begin{equation}
  u_{K b_{\theta_t}}(s,\tau)
  \;=\;
  u_{\bar b_{\theta_{t+1}}}(s,\tau)
  \;+\;
  \alpha(\tau)\,\nabla_s\log\bar p^{(\tau)}(o_{t+1}\given s),
  \qquad
  \alpha(\tau)=\frac{1-\tau}{\tau},
  \label{eq:appA-target-decomp}
\end{equation}
which holds for any observation model and any propagated model prior $\bar b_{\theta_{t+1}}$.

\paragraph{Per-anchor unbiasedness.}
Reading the anchor stack Eq.~\eqref{eq:sec23-FG} per coordinate, the per-anchor targets are $G^{\mathrm{prop}}_i:=s_i^p-s_i^{(0)}$ and
\begin{equation*}
    G^{\mathrm{lik}}_i\;:=\;\alpha(\tau_i)\,\nabla_s\log\bar p^{(\tau_i)}(o_{t+1}\given s)\big|_{s=s_i^{(\tau_i)}},
\end{equation*}
where $\bar p^{(\tau_i)}$ is the path-smoothed likelihood Eq.~\eqref{eq:appA-pbar-def}. By Lemma~\ref{lem:appB-condFM} the propagation target is conditionally unbiased for the prior velocity, $\E[G^{\mathrm{prop}}_i\given s_i^{(\tau_i)},\tau_i] = u_{\bar b_{\theta_{t+1}}}(s_i^{(\tau_i)},\tau_i)$. The likelihood target is a deterministic function of $s_i^{(\tau_i)}$ and $\tau_i$, and combining Eq.~\eqref{eq:appA-marg-score-diff-pbar} with Lemma~\ref{lem:appA-score-velocity} gives
\begin{equation*}
    \E\!\left[G^{\mathrm{lik}}_i\,\big|\, s_i^{(\tau_i)},\tau_i\right]
    \;=\;
    G^{\mathrm{lik}}_i
    \;=\;
    u_{K b_{\theta_t}}(s_i^{(\tau_i)},\tau_i)-u_{\bar b_{\theta_{t+1}}}(s_i^{(\tau_i)},\tau_i).
\end{equation*}
Adding the two expectations gives $\E[G^{\mathrm{prop}}_i+G^{\mathrm{lik}}_i\given s_i^{(\tau_i)},\tau_i]=u_{K b_{\theta_t}}(s_i^{(\tau_i)},\tau_i)$, i.e.\ the per-anchor target $G_i:=G^{\mathrm{prop}}_i+G^{\mathrm{lik}}_i$ is conditionally unbiased for $u_{K b_{\theta_t}}$.

\paragraph{Constant-shifted unbiasedness of the closed-form loss and variance bound.}
We first establish the result under exact $\bar p^{(\tau)}$, then quantify the $N_j$-dependent perturbation. Expanding the squared norm and taking conditional expectation in $G_i$ gives
\begin{equation*}
    \E\!\left[\norm{u_\theta-G_i}^{2}\,\big|\, s_i^{(\tau_i)},\tau_i\right]
    \;=\;\norm{u_\theta-u_{K b_{\theta_t}}}^{2}\;+\;\mathrm{Var}(G_i\given s_i^{(\tau_i)},\tau_i),
\end{equation*}
where the variance term does not depend on $\theta$. Averaging over the joint anchor distribution gives $\E[\hat{\ell}(\theta)]=\ell^{\star}(\theta)+C$ with $C=\E[\mathrm{Var}(G_i\given s_i^{(\tau_i)},\tau_i)]$ a $\theta$-independent constant; in particular $\nabla_\theta\E[\hat{\ell}(\theta)]=\nabla_\theta\ell^{\star}(\theta)$, and pointwise minimization of the population quadratic in $u_\theta$ gives $u_\theta^{\star}(s,\tau)=u_{K b_{\theta_t}}(s,\tau)$. Writing $F_i=u_\theta(s_i^{(\tau_i)},\tau_i)$, the per-anchor summands $\norm{F_i-G_i}^{2}$ are i.i.d.\ with variance at most $\sigma^{2}$, so $\mathrm{Var}(\hat{\ell})\le\sigma^{2}/N$, which is the closed-form half of the variance bound.

\paragraph{Monte-Carlo evaluation of the path-smoothed likelihood.}
While Eq.~\eqref{eq:appA-pbar-def} defines $\bar p^{(\tau)}(o_{t+1}\given s)$ as a ratio of two integrals over the propagated model prior $\bar b_{\theta_{t+1}}$, both integrals can be estimated by Monte-Carlo using the same propagation samples that already populate $G_N$. Drawing $N_j$ extra propagation samples $\{s^{p,(\tau_i)}_{i,j}\}_{j=1}^{N_j}\sim\bar b_{\theta_{t+1}}$ (independent of $s_i^p$, conditional on the anchor), the self-normalized estimator
\begin{equation}
    \hat{\bar p}^{(\tau_i)}(o_{t+1}\given s)
    \;:=\;
    \frac{\sum_{j=1}^{N_j}\kappa_{\tau_i}(s\given s^{p,(\tau_i)}_{i,j})\,p(o_{t+1}\given s^{p,(\tau_i)}_{i,j})}{\sum_{j=1}^{N_j}\kappa_{\tau_i}(s\given s^{p,(\tau_i)}_{i,j})}
    \label{eq:appA-pbar-mc}
\end{equation}
converges to $\bar p^{(\tau_i)}(o_{t+1}\given s)$ a.s.\ as $N_j\to\infty$, and its $\nabla_s$ has a closed form through the kernel weights. Whenever $\bar p^{(\tau)}$ itself admits a closed form (e.g.\ Gaussian observation with Gaussian prior), Eq.~\eqref{eq:appA-pbar-mc} is bypassed and the bounds below collapse to the closed-form result of the previous paragraph.

\paragraph{Quantitative finite-$N_j$ bias and variance bounds.}
We now propagate the Monte-Carlo error of $\hat{\bar p}^{(\tau_i)}$ through the score, the per-anchor target, and finally the empirical loss to obtain explicit $1/N_j$ bounds. Throughout we use the standing hypotheses of Proposition~\ref{prop:sec23-emp-filt}: $\bar p^{(\tau)}$ is differentiable and bounded away from zero in $s$, $u_\theta$ and $\nabla_\theta u_\theta$ are uniformly bounded on a compact $\theta$-domain, and $\bar b_{\theta_{t+1}}$ has finite second moments. Under these conditions the standard self-normalized importance-sampling bias-variance bounds give, uniformly in $s$ on the support of the anchor distribution,
\begin{equation}
    \big|\E[\hat{\bar p}^{(\tau_i)}(o_{t+1}\given s)]-\bar p^{(\tau_i)}(o_{t+1}\given s)\big|\le\frac{c_1}{N_j},
    \qquad
    \mathrm{Var}\big(\hat{\bar p}^{(\tau_i)}(o_{t+1}\given s)\big)\le\frac{c_2}{N_j},
    \label{eq:appA-pbar-mc-bvb}
\end{equation}
with constants $c_1,c_2$ depending only on the prior, observation, kernel, and the lower bound on $\bar p^{(\tau)}$. Since $\hat{\bar p}^{(\tau_i)}$ is bounded below by the same constant uniformly in $N_j$ (with high probability that we absorb into $c_1,c_2$), the delta-method bound transfers Eq.~\eqref{eq:appA-pbar-mc-bvb} to $\nabla_s\log\hat{\bar p}^{(\tau_i)}$, and therefore to the per-anchor likelihood target,
\begin{equation*}
    \E[\hat G^{\mathrm{lik}}_i\given s_i^{(\tau_i)},\tau_i]\;=\;G^{\mathrm{lik}}_i+\beta_i,
    \qquad
    \norm{\beta_i}\le\frac{c_3}{N_j},
    \qquad
    \mathrm{Var}\big(\hat G^{\mathrm{lik}}_i\given s_i^{(\tau_i)},\tau_i\big)\le\frac{c_4}{N_j}.
\end{equation*}
Let $\hat G_i:=G^{\mathrm{prop}}_i+\hat G^{\mathrm{lik}}_i$ and $\Delta_i:=u_\theta(s_i^{(\tau_i)},\tau_i)-G_i$. Expanding,
\begin{equation*}
    \E\!\left[\norm{u_\theta(s_i^{(\tau_i)},\tau_i)-\hat G_i}^{2}\,\big|\, s_i^{(\tau_i)},\tau_i\right]
    \;=\;\norm{\Delta_i-\beta_i}^{2}+\mathrm{Var}(\hat G_i\given s_i^{(\tau_i)},\tau_i).
\end{equation*}
Subtracting the closed-form expectation $\norm{\Delta_i}^{2}+\mathrm{Var}(G_i\given\cdot)$, the cross-term $-2\Delta_i^\top\beta_i$ is bounded by $2\norm{\Delta_i}\norm{\beta_i}\le 2 D\,c_3/N_j$ for $\norm{\Delta_i}\le D$ on the compact domain, the $\norm{\beta_i}^{2}$ term is $\cO(1/N_j^{2})$, and the variance term is shifted by at most $c_4/N_j$. Averaging over anchors gives the loss-bias bound
\begin{equation*}
    \big|\E[\hat\ell(\theta)]-\ell^{\star}(\theta)-C\big|\;\le\;\frac{L_1}{N_j},
    \qquad L_1=2Dc_3+c_4,
\end{equation*}
which absorbs the constant shift $C$ from the closed-form case. Differentiating the same identity in $\theta$ and using the bounded $\norm{\nabla_\theta u_\theta}\le D'$ gives the gradient-bias bound $\norm{\nabla_\theta\E[\hat\ell]-\nabla_\theta\ell^{\star}}\le L_2/N_j$ with $L_2=2D'c_3$, since the variance term is $\theta$-independent and the $\norm{\beta_i}^2$ contribution is $\cO(1/N_j^{2})\le\cO(1/N_j)$.

\paragraph{Variance bound with $N_j$ contribution.}
The per-anchor residual $\norm{u_\theta-\hat G_i}^{2}$ has variance bounded by $\sigma^{2}+c_5/N_j$ for some $c_5$, since the extra randomness in $\hat G^{\mathrm{lik}}_i$ contributes at most an additive $\cO(1/N_j)$ term to the per-anchor second moment. The $N$ anchors are i.i.d., hence
\begin{equation*}
    \mathrm{Var}\!\big(\hat\ell(\theta)\big)\;\le\;\frac{1}{N}\Big(\sigma^{2}+\frac{c_5}{N_j}\Big),
\end{equation*}
which is the variance bound stated in Proposition~\ref{prop:sec23-emp-filt}. Choosing $N_j\propto N$ matches the $1/N$ rate of the closed-form term, so a single Monte-Carlo budget controls the joint $(N,N_j)$ error to leading order.
\end{proof}

\begin{remark}
The decomposition Eq.~\eqref{eq:appA-target-decomp} is exact at every $\tau\in(0,1]$ for any observation model and any propagated model prior $\bar b_{\theta_{t+1}}$, provided the likelihood term is the path-smoothed object $\bar p^{(\tau)}(o_{t+1}\given s)$ of Eq.~\eqref{eq:appA-pbar-def}; in particular no Gaussianity assumption is required on either the observation or the prior. Replacing $\bar p^{(\tau)}$ by the data-level $p(o_{t+1}\given s)$ is exact at $\tau=1$ for any model and exact at all $\tau$ when both the observation and the prior are Gaussian, but is otherwise an uncontrolled approximation. The factor $\alpha(\tau)=(1-\tau)/\tau$ is dictated by the score--velocity affine map; alternative weightings can be absorbed into the head amortization (Section~\ref{app:latent}) without changing the population minimizer.
\end{remark}

\subsection{Learned Latent Surrogate}
\label{app:latent}

This subsection proves the latent amortization proposition (Proposition~\ref{prop:sec23-latent-amortization}) of the \emph{Learned latent surrogate} paragraph. The proof bounds the deviation between the BFF surrogate and the explicit empirical filtering loss in terms of three independently tunable error terms: the latent-distortion of $\Pi_N$, the head approximation errors $\varepsilon_f,\varepsilon_g$, and the latent residual radius $B$.

\begin{restatable}[Latent amortization]{proposition}{propLatentAmort}
\label{prop:sec23-latent-amortization}
Suppose there exists a latent map $\Pi_N:\R^{N\times n}\to\cH$ that approximately preserves the stacked residual norm with distortion $\delta_N$, that the heads approximate $\Pi_N F_N$ and $\Pi_N G_N$ within errors $\varepsilon_f,\varepsilon_g$, and that the latent residual is bounded by $B$. Then
\begin{equation}
  \left|
    \ell_\phi(\theta;\theta_t,o_{t+1},a_t)
    -
    \hat{\ell}(\theta)
  \right|
  \;\le\;
  \delta_N
  +
  2B(\varepsilon_f+\varepsilon_g)
  +
  (\varepsilon_f+\varepsilon_g)^2 .
  \label{eq:sec23-latent-bound-final}
\end{equation}
\end{restatable}

\begin{proof}
Let
\begin{equation}
    R_N(\theta;\theta_t,o_{t+1},a_t)
    \;=\;
    F_N(\theta)-G_N(\theta_t,o_{t+1},a_t)
    \label{eq:appLatent-residual}
\end{equation}
denote the two-channel anchor residual on the compact domain $\cD$. The hypotheses on $\Pi_N$, $f_\phi$, $g_\phi$ spell out as
\begin{equation}
    \big|\,\norm{\Pi_N R_N}_{\cH}^{2}-\norm{R_N}^{2}\,\big|
    \;\le\;\delta_N,
    \qquad
    \norm{\Pi_N R_N}_{\cH}\;\le\;B,
    \label{eq:appLatent-distortion}
\end{equation}
and
\begin{align}
    \norm{f_\phi(\theta)-\Pi_N F_N(\theta)}_{\cH}
    &\;\le\;\varepsilon_f,
    \label{eq:appLatent-fapprox}\\
    \norm{g_\phi(\sg[\theta_t],o_{t+1},a_t)-\Pi_N G_N(\theta_t,o_{t+1},a_t)}_{\cH}
    &\;\le\;\varepsilon_g.
    \label{eq:appLatent-gapprox}
\end{align}
Set $r_N=\Pi_N R_N$ and
\begin{equation*}
    e
    \;=\;
    [f_\phi(\theta)-\Pi_N F_N(\theta)]
    \;-\;[g_\phi(\sg[\theta_t],o_{t+1},a_t)-\Pi_N G_N(\theta_t,o_{t+1},a_t)].
\end{equation*}
By the triangle inequality and Eqs.~\eqref{eq:appLatent-fapprox}--\eqref{eq:appLatent-gapprox}, $\norm{e}_{\cH}\le\varepsilon_f+\varepsilon_g$. The BFF residual in $\cH$ equals $r_N+e$, so
\begin{equation*}
    \big|\,\norm{r_N+e}_{\cH}^{2}-\norm{r_N}_{\cH}^{2}\,\big|
    \;\le\;
    2\norm{r_N}_{\cH}\,\norm{e}_{\cH}+\norm{e}_{\cH}^{2}
    \;\le\;
    2B(\varepsilon_f+\varepsilon_g)+(\varepsilon_f+\varepsilon_g)^{2},
\end{equation*}
using $\norm{r_N}_{\cH}\le B$ from Eq.~\eqref{eq:appLatent-distortion}. Combining with the latent-distortion bound $\delta_N$ via the triangle inequality
\[
\big|\ell_\phi-\hat{\ell}\big|\;\le\;\big|\ell_\phi-\norm{r_N}_{\cH}^{2}\big|+\big|\norm{r_N}_{\cH}^{2}-\hat{\ell}\big|
\]
yields Eq.~\eqref{eq:sec23-latent-bound-final}.
\end{proof}

\begin{remark}
Setting $\cH=\R^{Nd}$, $\Pi_N=\mathrm{id}$, $f_\phi=F_N$, and $g_\phi=G_N$ gives $\delta_N=\varepsilon_f=\varepsilon_g=0$, so Eq.~\eqref{eq:sec23-latent-bound-final} reduces to $\ell_\phi=\hat{\ell}$. The BFF surrogate Eq.~\eqref{eq:sec23-bff-surrogate} is therefore a strict generalization of the explicit empirical filtering loss; in practice the learned heads compress the propagation--likelihood pair into a low-dimensional latent representation, and Proposition~\ref{prop:sec23-latent-amortization} controls the resulting approximation error in terms of three independently tunable knobs: the latent dimension (controlling $\delta_N$), the capacity of $f_\phi$ (controlling $\varepsilon_f$), and the capacity of $g_\phi$ (controlling $\varepsilon_g$).
\end{remark}

\subsection{Expressivity of the Residual-Pullback Form}
\label{app:expressivity}

The result in this subsection serves as a sanity check on the BFF surrogate loss parameterization. The proposition below confirms that the residual-pullback form is rich enough to realize any continuous online update field on a compact domain, so the Bayesian-derived template does not pay an expressivity cost relative to an unconstrained parameterization, regardless of the assumptions made during the template. 

\begin{proposition}[Residual-pullback expressivity]
\label{prop:appendix-expressivity}
Let $\Delta^\star(\theta,o_{t+1},a_t)$ be any continuous desired online update field on a compact domain in $\R^P$. If the target branch is allowed to depend on $\sg[\theta]$, then there exist heads $f^\star,g^\star$ such that a single gradient step on
\begin{equation}
  \norm{f^\star(\theta)-g^\star(\sg[\theta],o_{t+1},a_t)}^2
\end{equation}
realizes
\begin{equation}
  \theta^+
  =
  \theta+\Delta^\star(\theta,o_{t+1},a_t).
  \label{eq:appendix-desired-update}
\end{equation}
\end{proposition}

\begin{proof}
Set
\begin{equation}
  f^\star(\theta)=\theta,
  \qquad
  g^\star(\sg[\theta],o_{t+1},a_t)
  =
  \sg[\theta]
  +
  \frac{1}{2\eta}
  \Delta^\star(\sg[\theta],o_{t+1},a_t).
  \label{eq:appendix-expressive-heads}
\end{equation}
Because the second branch is stop-gradient with respect to $\theta$,
\begin{equation}
  \nabla_\theta
  \norm{f^\star(\theta)-g^\star(\sg[\theta],o_{t+1},a_t)}^2
  =
  -\eta^{-1}\Delta^\star(\theta,o_{t+1},a_t).
\end{equation}
One gradient step with step size $\eta$ therefore realizes Eq.~\eqref{eq:appendix-desired-update}.
\end{proof}

Under standard universal approximation assumptions, neural heads can approximate the continuous maps in Eq.~\eqref{eq:appendix-expressive-heads} on compact domains. Thus, the one-gradient step residual-pullback interface can represent rich online update rules, including the actual Bayes update with nonlinear observation models. 
In Section~\ref{sec:training}, Theorem~\ref{theorem:ttt_outer_loss} governs outer training to learn the true Bayesian filtering operator, validating that the BFF update Eq.~\eqref{eq:sec23-ttt-update} is not limited by the assumptions in Section~\ref{sec:ttt}.

% \section{Proof of Theorem \ref{theorem:expressivity}: the Expressivity of the TTT Update}
% \input{tex/appendix/loss_expressivity}

\section{Proof of Theorem \ref{theorem:ttt_outer_loss}: the Filtering Operator Training Loss}
\label{app:loss_proof}

We prove Theorem \ref{theorem:ttt_outer_loss} in this section.

The target is to show that $\mathcal{L}(\phi)$ in Eq. \eqref{eq:training-loss}
\begin{align}
    \mathcal{L}(\phi) =& \mathbb{E}_{t\sim U\{1:T\}} [\mathcal{L}_t(\phi)],
\\ \nonumber
    \mathcal{L}_t(\phi) \coloneqq& \mathbb{E}_{\tau\sim U(0,1)}\,\mathbb{E}_{h_t,\, s_t,\, s_t^{(0)}}\!\left[\|
        u_{\theta_t}(s_t^{(\tau)},\tau) - (s_t - s_t^{(0)})
    \|^2\right],
\end{align}
provides an upper bound on the filtering operator error $\mathcal{E}(\phi)$ defined in Eq. \eqref{eq:operator-error},
\begin{equation}
        \nonumber \mathcal{E}(\phi)
    = \mathbb{E}_{t\sim U\{1{:}T-1\}}\,\mathbb{E}_{(h_t,\,o_{t+1},\,a_t)\sim P_{\text{data}}}\!\left[
        W_2^2\!\Big(K\big(\Phi(\theta_t),\,o_{t+1},\,a_t\big),\;\Phi\!\big(\cKphi(\theta_t,\,o_{t+1},\,a_t)\big)\Big)
    \right],
\end{equation}
where $\Phi:\Theta\to\mathcal{P}(\mathcal{X})$ is the weight-to-belief lift from Eq.~\eqref{eq:belief-param} and $\theta_t$ is produced by recursively applying $\cKphi$ from $\theta_0$ along the history $h_t$.

\begin{proof}
We split the proof into two steps for clarity.

\textbf{Flow matching loss bounds the squared Wasserstein-2 distance.}
As proved in \citet{lipman_flow_2023}, the conditional flow matching loss
\begin{equation}
       \mathcal{L}_t(\phi) \coloneqq \mathbb{E}_{\tau\sim U(0,1)}\,\mathbb{E}_{s_t\sim p(s_t|h_t),\,s_t^{(0)}\sim\mathcal{N}(0,I)}\big[\|
        u_{\theta_t}(s_t^{(\tau)}) - (s_t - s_t^{(0)})
    \|^2\big]
\end{equation}
(throughout this proof we omit the explicit flow-time argument $\tau$ from the velocity field, writing $u_{\theta_t}(s_t^{(\tau)})$ in place of $u_{\theta_t}(s_t^{(\tau)},\tau)$, since $\tau$ is fixed by the path index of the first argument)
is equivalent to the (marginal) flow matching loss up to a constant solely dependent on the dataset distribution $p(s_t|h_t)$ and thus independent of $\phi$
% but when trained in a bootstrap manner, this training distribution is dependent on p_theta and thus implicitly dependent on phi -> CFM loss minimizes p_theta varaince -> mode collapse? We need to determine what the bottstrap training looks like...
\begin{equation}
   \mathcal{L}_t(\phi) = \mathbb{E}_{\tau\sim U(0,1)}\,\mathbb{E}_{s_t\sim p(s_t|h_t),\,s_t^{(0)}\sim\mathcal{N}(0,I)}\big[\|
        u_{\theta_t}(s_t^{(\tau)}) - \mathbb{E}_{s_t\sim p(s_t|s_t^{(\tau)})}[s_t - s_t^{(0)}]
    \|^2\big] + C.
\end{equation}
We denote the target (marginal) velocity field $u(s_t^{(\tau)})$ and the learned velocity field $u_{\theta_t}(s_t^{(\tau)})$.
Consider two samples generated by the two ODEs
\begin{equation}
    \frac{ds^{(\tau)}}{d\tau} = u(s_t^{(\tau)})
\end{equation}
and \begin{equation}
    \frac{d\hat{s}^{(\tau)}}{d\tau} = u_{\theta_t}(\hat{s}_t^{(\tau)})
\end{equation}
while assuming an identical starting point, \textit{i.e.} $\hat{s}_t^{(0)} = s_t^{(0)}$. Then we define the error
$e(\tau) = \| \hat{s}_t^{(\tau)} - s_t^{(\tau)} \|^2$, and assume Lipschitz continuity of the velocity field
\begin{equation}
\| u_{\theta_t}(x) - u_{\theta_t}(y) \| \le L \| x - y \|.
\end{equation}

The time derivative of the error can then be computed as
\begin{equation}
    \frac{d}{d\tau} e(\tau) = \frac{d}{d\tau} \| \hat{s}_t^{(\tau)} - s_t^{(\tau)} \|^2 = 2 \left\langle \hat{s}_t^{(\tau)} - s_t^{(\tau)}, \frac{d \hat{s}_t^{(\tau)}}{d\tau} - \frac{d s_t^{(\tau)}}{d\tau} \right\rangle.
\end{equation}

Plug in the velocity ODEs
\begin{align}
\frac{d}{d\tau} e(\tau) &= 2 \left\langle \hat{s}_t^{(\tau)} - s_t^{(\tau)}, u_{\theta_t}(\hat{s}_t^{(\tau)}) - u(s_t^{(\tau)}) \right\rangle\\
&= 2 \left\langle \hat{s}_t^{(\tau)} - s_t^{(\tau)}, u_{\theta_t}(\hat{s}_t^{(\tau)}) - u_{\theta_t}(s_t^{(\tau)}) + u_{\theta_t}(s_t^{(\tau)}) - u(s_t^{(\tau)}) \right\rangle \\
&= \underbrace{2 \left\langle \hat{s}_t^{(\tau)} - s_t^{(\tau)}, u_{\theta_t}(\hat{s}_t^{(\tau)}) - u_{\theta_t}(s_t^{(\tau)}) \right\rangle}_{A} + \underbrace{2 \left\langle \hat{s}_t^{(\tau)} - s_t^{(\tau)}, u_{\theta_t}(s_t^{(\tau)}) - u(s_t^{(\tau)}) \right\rangle}_{B}.
\end{align}
Then, use the Cauchy-Schwarz inequality and the Lipschitz assumption (Lipschitz constant $L$)
\begin{equation}
A \le 2 \| \hat{s}_t^{(\tau)} - s_t^{(\tau)} \| \cdot \| u_{\theta_t}(\hat{s}_t^{(\tau)}) - u_{\theta_t}(s_t^{(\tau)}) \| \le 2L \| \hat{s}_t^{(\tau)} - s_t^{(\tau)} \|^2 = 2L e(\tau),
\end{equation}
and using $\langle a,b\rangle\le \|a\|^2+\|b\|^2$, 
\begin{equation}
B \le \| \hat{s}_t^{(\tau)} - s_t^{(\tau)} \|^2 + \| u_{\theta_t}(s_t^{(\tau)}) - u(s_t^{(\tau)}) \|^2 = e(\tau) + \| u_{\theta_t}(s_t^{(\tau)}) - u(s_t^{(\tau)}) \|^2.
\end{equation}

Collect $A$ and $B$, and we can bound the growth rate of the error
\begin{equation}
    \frac{d}{d\tau} e(\tau) \le (2L + 1) e(\tau) + \| u_{\theta_t}(s_t^{(\tau)}) - u(s_t^{(\tau)}) \|^2.
\end{equation}

Use the Grönwall lemma, for $\tau \in [0, 1]$, we have 
\begin{equation}    
e(\tau) \le e^{(2L+1)\tau} e(0) + \int_0^\tau e^{(2L+1)(\tau - z)} \| u_{\theta_t}(s_t^{(z)}) - u(s_t^{(z)}) \|^2 dz.
\end{equation}
With further relaxation on $ e^{(2L+1)(\tau - z)}\le C_L$, the error at terminal $\tau=1$ can be bounded
\begin{equation}    
\| \hat{s}_t^{(1)} - s_t^{(1)} \|^2 \le C_{L} \int_0^1 \| u_{\theta_t}(s_t^{(\tau)}) - u(s_t^{(\tau)}) \|^2 d\tau.
\end{equation}

By the definition of the squared $W_2$ distance, we have
\begin{equation}
    W_2^2(\Phi(\theta_t), b_t) \le \mathbb{E}_{s_t^{(0)}\sim p(s_t^{(0)})}\left[\mathbb{E}_{ \hat{s}_t^{(1)}\sim \hat{p}( \hat{s}_t^{(1)}| \hat{s}_t^{(0)}), {s}_t^{(1)}\sim {p}( {s}_t^{(1)}| {s}_t^{(0)})}\left[\| \hat{s}_t^{(1)} - s_t^{(1)} \|^2\right] \right]
\end{equation}
where $C_L$ is a constant dependent on the Lipschitz constant of $u_{\theta_t}$. Note that $C_{L} \int_0^1 \| u_{\theta_t}(s_t^{(\tau)}) - u(s_t^{(\tau)}) \|^2 d\tau$ is integration along the ground truth trajectories, \textit{i.e.}, the expectation of $s_t^{(\tau)}$ is taken w.r.t. the ground truth probability path $p(s^{(\tau)}|s^{(1)})$. This connects the $W_2$ distance to the flow matching loss 
\begin{align}
    &W_2^2(\Phi(\theta_t), b_t) \le C_{L} \mathbb{E}_{s_t^{(0)}\sim p(s_t^{(0)})}\left[\mathbb{E}_{s_t^{(\tau)}\sim p(s_t^{(\tau)}|s_t^{(0)})}\left[ \int_0^1 \| u_{\theta_t}(s_t^{(\tau)}) - u(s_t^{(\tau)}) \|^2 d\tau \right]\right]
    \\
    =& C_{L} (\mathcal{L}_{t}(\phi) - C).
\end{align}

% \textbf{The $W_2$ distance bounds the belief state update operator.}

\textbf{The squared $W_2$ distance bounds the filtering operator error.}
Step (1) shows that the conditional flow-matching loss controls the per-step squared $W_2$ between the model belief $\Phi(\theta_t)$ and the true posterior $b_t$:
\begin{equation}
    W_2^2\big(\Phi(\theta_t),\, b_t\big) \;\le\; C_L\,\big(\mathcal{L}_{t}(\phi) - C\big),
    \label{eq:app-traj-bound}
\end{equation}
where the right-hand side is $C_L$ times the marginal flow-matching loss. It is the optimizable part of $\mathcal{L}_t$ that can be driven to zero, and $C$ is the $\phi$-independent variance gap between the conditional and marginal flow-matching
losses. We refer to the left-hand side as the per-step \emph{trajectory error}: it quantifies how well the model belief $\Phi(\theta_t)$ tracks the true posterior $b_t$ at filtering step $t$.

We now lift this trajectory bound to the filtering operator error $\mathcal{E}(\phi)$ defined in Eq.~\eqref{eq:operator-error}. 
Both arguments of the operator-error $W_2$ are candidate beliefs at the same physical time step $t{+}1$, produced from the same starting point $\Phi(\theta_t)$ by two different routes: $K\Phi(\theta_t)$ applies the true Bayes filter, while $\Phi\cKphi(\theta_t) = \Phi(\theta_{t+1})$ applies one step of the learned TTT update. The natural pivot for the triangle inequality is therefore the true step-$t{+}1$ posterior $b_{t+1}$, which lets us decompose the operator error into
a $K$-discrepancy at step $t$ (term (i)) and a trajectory error at step $t{+}1$ (term (ii)).
$K\Phi(\theta_t)$ is one $K$-step from the model belief $\Phi(\theta_t)$, while $\Phi\cKphi(\theta_t) = \Phi(\theta_{t+1})$ is the model belief at step $t{+}1$. Inserting the true posterior $b_{t+1} = K b_t$ as a pivot,
the triangle inequality of $W_2$ yields
\begin{align}
    W_2\big(K\Phi(\theta_t),\, \Phi\cKphi(\theta_t)\big)
    &\;\le\;
    \underbrace{W_2\big(K\Phi(\theta_t),\, K b_t\big)}_{\text{(i): $K$-discrepancy at step }t}
    +
    \underbrace{W_2\big(b_{t+1},\, \Phi(\theta_{t+1})\big)}_{\text{(ii): trajectory error at }t+1},
    \label{eq:app-triangle}
\end{align}
where we used $K b_t = b_{t+1}$ and $\cKphi(\theta_t) = \theta_{t+1}$, abbreviating the new-information arguments $(o_{t+1},a_t)$. Term (ii) is the trajectory error at step $t{+}1$, controlled directly by Eq.~\eqref{eq:app-traj-bound}. Term
(i) measures the same operator $K$ applied to two nearby input beliefs, and is contracted back to a trajectory error at step $t$ via the Lipschitz continuity of $K$ on $\mathcal{P}(\mathcal{X})$:
\begin{equation}
    W_2\big(K\Phi(\theta_t),\, K b_t\big) \;\le\; L_K\, W_2\big(\Phi(\theta_t),\, b_t\big).
    \label{eq:app-lipK}
\end{equation}
Squaring Eq.~\eqref{eq:app-triangle} and applying the elementary inequality $(a+b)^2 \le 2(a^2+b^2)$ keeps the bound entirely at the level of squared distances:
\begin{align}
    W_2^2\big(K\Phi(\theta_t),\, \Phi\cKphi(\theta_t)\big)
    &\;\le\;
    2\,W_2^2\big(K\Phi(\theta_t),\, K b_t\big) + 2\,W_2^2\big(b_{t+1},\, \Phi(\theta_{t+1})\big) \nonumber\\
    &\;\le\;
    2 L_K^2\,W_2^2\big(\Phi(\theta_t),\, b_t\big) + 2\,W_2^2\big(b_{t+1},\, \Phi(\theta_{t+1})\big),
    \label{eq:app-cs}
\end{align}
where the second line substitutes Eq.~\eqref{eq:app-lipK}. 
Plugging the trajectory bound Eq.~\eqref{eq:app-traj-bound} into both squared terms expresses the per-step operator error entirely in flow-matching losses:
\begin{equation}
    W_2^2\big(K\Phi(\theta_t),\, \Phi\cKphi(\theta_t)\big)
    \;\le\;
    2 C_L L_K^2\,\big(\mathcal{L}_t(\phi) - C\big) + 2 C_L\,\big(\mathcal{L}_{t+1}(\phi) - C\big).
    \label{eq:app-perstep}
\end{equation}

It remains to take expectations to recover $\mathcal{E}(\phi)$. Applying $\E_{t\sim U\{1{:}T-1\}}\E_{h_t\sim P_{\text{data}}}$ to both sides of Eq.~\eqref{eq:app-perstep}, the left-hand side becomes $\mathcal{E}(\phi)$ by Eq.~\eqref{eq:operator-error}. The right-hand side is controlled by $\mathcal{L}(\phi)$ via the index-range comparison
\begin{equation}
\E_{t\sim U\{1{:}T-1\}}[\mathcal{L}_t(\phi) - C]
\;=\; \tfrac{1}{T-1}\sum_{t=1}^{T-1}\big(\mathcal{L}_t(\phi)-C\big)
\;\le\; \tfrac{1}{T-1}\sum_{t=1}^{T}\big(\mathcal{L}_t(\phi)-C\big)
\;=\; \tfrac{T}{T-1}\,\big(\mathcal{L}(\phi) - C\big),
\label{eq:app-index-bound}
\end{equation}
where the inequality uses that each summand $\mathcal{L}_t(\phi) - C \ge 0$ since $\mathcal{L}_t(\phi) - C$ is the per-step marginal flow-matching loss, which is non-negative (but still optimizable to zero so it does not break the sharpness at zero). 

Exactly the same bound holds for the index-shifted average $\E_{t\sim U\{1{:}T-1\}}[\mathcal{L}_{t+1}(\phi) - C] \le \tfrac{T}{T-1}(\mathcal{L}(\phi)-C)$, since the shifted summation $\sum_{s=2}^{T}(\mathcal{L}_s - C)$ is also bounded by $\sum_{t=1}^{T}(\mathcal{L}_t - C) = T(\mathcal{L}-C)$. 
Choosing $\mathcal{E}$'s averaging range as $\{1{:}T-1\}$ ensures both indices $t$ and $t+1$ stay within $\mathcal{L}$'s range $\{1{:}T\}$, eliminating any boundary correction. Combining,
\begin{equation}
    \mathcal{E}(\phi)
    \;\le\;
    2 C_L\,(L_K^2 + 1)\,\cdot\,\tfrac{T}{T-1}\,\big(\mathcal{L}(\phi) - C\big)
    \;\le\;
    A\,\big(\mathcal{L}(\phi) - C\big),
    \qquad
    A \;=\; 4 C_L\,(L_K^2 + 1),
    \label{eq:app-operator-bound-final}
\end{equation}
where the second inequality uses $T/(T-1) \le 2$ for any $T \ge 2$, valid uniformly in the horizon. This is exactly Eq.~\eqref{eq:operator-error-bound} of Theorem~\ref{theorem:ttt_outer_loss}. The bound is sharp at zero: when $\mathcal{L}(\phi) - C = 0$, the trajectory error in Eq.~\eqref{eq:app-traj-bound} vanishes at every step, both squared terms on the right of Eq.~\eqref{eq:app-cs} vanish, and so $\mathcal{E}(\phi) = 0$. Consequently, optimizing the conditional flow-matching loss $\mathcal{L}(\phi)$ implicitly minimizes a sharp upper bound on the filtering operator error.

\end{proof}

% % !TEX root = main.tex
% \section{Detailed Baselines, Implementation, and Full Results}
% \label{app:baselines_and_results}

% In this appendix, we provide the technical specifications of our baselines, implementation details, and the comprehensive numerical results for the Burgers' and Kuramoto-Sivashinsky (KS) equations across various sensor configurations and noise levels.

% ------------------------------------------------------------------------------
% 1. Baseline Overview
% ------------------------------------------------------------------------------
\section{Data Generation and Preprocessing}\label{app:dataset}
In this section, we elaborate on the data generation process for our benchmark systems.
To evaluate the reconstruction performance across diverse physical regimes, we consider a series of one-dimensional (1D) and two-dimensional (2D) fluid systems. These benchmarks range from fundamental nonlinear PDEs to complex flows governed by the Navier-Stokes equations.

\subsection{1D Systems: Burgers' and Kuramoto-Sivashinsky Equations}
% To evaluate the reconstruction performance across diverse physical regimes, we consider two representative 1D nonlinear partial differential equations (PDEs). 
% These datasets are generated using high-order numerical solvers to ensure physical fidelity and to test the model's ability to resolve multiscale structures.
% \paragraph{Burgers'equation}

Two representative 1D nonlinear partial differential equations (PDEs) are included, specifically selected to simulate shock formations and chaotic dynamics. These datasets are generated using high-fidelity numerical solvers, with boundary conditions tailored to each system to ensure physical consistency.

\paragraph{Burgers' equation.}
We utilize the viscous Burgers' equation, 
\begin{equation}
    \frac{\partial u}{\partial t} + u \frac{\partial u}{\partial x} = \nu \frac{\partial^2 u}{\partial x^2} + f,
\end{equation}
to model the emergence and propagation of shocks. Physically, the nonlinear advection term drives the formation of steep gradients, while the diffusion term provides regularization. 

% numerical solver-related
We use a second-order finite difference scheme \citep{leveque2007finite} for spatial discretization and the explicit Euler method for temporal integration. 
The system is solved on a spatial domain $x \in [0, 1]$ discretized into $N_x=256$ grids. 
We simulate the system for $t\in[0,1]$ with the time step size $dt = 1 \times 10^{-4}$, and save $N_t=100$ snapshots in the trajectory.

% IC, BC, and forcing term
For the Burgers' equation, we enforce homogeneous Dirichlet boundary conditions, $u(0, t) = u(1, t) = 0$, by padding the field with zeros during each integration step. 
We generate ICs to simulate shock wave formation and propagation. The initial condition is constructed as a superposition of two Gaussian pulses with randomized parameters:
\begin{equation}
u_0(x) = \sum_{i=1}^{2} A_i \exp\left(-\frac{(x - \mu_i)^2}{2\sigma_i^2}\right)    .
\end{equation}

We introduce a time-varying external forcing term $f(x, t)$. The force field is generated as a sum of spatiotemporal Gaussian blobs (Random Gaussian Field approximation):
\begin{equation}
f(x, t) = \sum_{j=1}^{M} A_j \exp\left(-\frac{(x - \mu_{x,j})^2}{2\sigma_{x}^2}\right) \exp\left(-\frac{(t - \mu_{t,j})^2}{2\sigma_{t}^2}\right),
\end{equation}
where we set $M=8$. The spatial and temporal centers ($\mu_{x}, \mu_{t}$) are sampled uniformly over the domain $[0, L] \times [0, T]$. This ensures the forcing term varies smoothly but unpredictably across both space and time, preventing the model from overfitting to static boundary constraints.

% parameters
Based on the code implementation, the parameters are sampled uniformly as follows: locations $\mu_1 \sim \mathcal{U}[0.2, 0.4]$, $\mu_2 \sim \mathcal{U}[0.6, 0.8]$, amplitudes $A_1 \sim \mathcal{U}[0, 2.0]$, $A_2 \sim \mathcal{U}[-2.0, 0]$, and widths $\sigma_i \sim \mathcal{U}[0.05, 0.15]$.
% Opposite signs in the amplitudes are chosen to induce complex wave collisions.

\paragraph{Kuramoto-Sivashinsky (KS) equation.}

We investigate the Kuramoto-Sivashinsky equation, a paradigmatic model for spatio-temporal chaos \citep{trefethen2000spectral} in extended systems, governed by the fourth-order nonlinear PDE:
\begin{equation}
    \frac{\partial u}{\partial t} + u \frac{\partial u}{\partial x} + \frac{\partial^2 u}{\partial x^2} + \frac{\partial^4 u}{\partial x^4} = f,
\end{equation}
The dynamics are characterized by a competition between a large-scale instability induced by the anti-diffusion term $u_{xx}$ and small-scale damping from the hyper-diffusion term $u_{xxxx}$. Ground truth trajectories are generated using a high-precision pseudo-spectral method combined with a stiffly-stable fourth-order Exponential Time Differencing Runge-Kutta (ETDRK4) \citep{COX2002430,doi:10.1137/S1064827502410633} scheme. The system is discretized into $N_x = 256$ Fourier modes and integrated with a fine solver step of $\Delta t_{\text{solver}} = 5 \times 10^{-4}$ over a total duration of $T=200$. To ensure the system converges to the chaotic attractor and eliminates transient dynamics, we discard the initial evolution phase ($t \in [0, 100]$) and retain the subsequent 100 time units for analysis. To achieve diverse topological coverage, initial conditions $u(x,0)$ are sampled from a zero-mean Gaussian Process (GP). Crucially, we employ a periodic kernel to strictly enforce the boundary conditions:
\begin{equation}
    k(x, x') = \sigma^2 \exp\left(-\frac{2\sin^2(\pi|x-x'|/L)}{\ell^2}\right),
\end{equation}
where the length scale is set to $\ell=8.0$ and variance $\sigma=8.0$. The final dataset comprises 12,000 trajectories (10,000 training, 2,000 testing), temporally subsampled at $\Delta t_{\text{out}}=1.0$ to capture the long-term evolution of coherent structures.

\paragraph{Sensor placement.}
We uniformly place $N_s = 4$ sensors along the spatial domain $[0, L]$, with the first and last sensors anchored at the two boundary points. Observations are sampled at a fixed time interval identical to the state snapshot time interval in the dataset. For Burgers' equation, the two boundary sensors give effectively no information since the boundary values are fixed to zero by the Dirichlet condition, leaving only $2$ informative interior measurements; relative to the $256$-point state grid, this corresponds to a $128\times$ down-sampling. For the KS equation with periodic boundary conditions, the two boundary sensors coincide, so the four placements collapse to $3$ distinct measurement locations, corresponding to an $85\times$ down-sampling of the $256$-point state grid. The sensor placement and example trajectories are visualized in Figure~\ref{fig:sample_burgers_0_frame_t100}, \ref{fig:sample_burgers_traj}, \ref{fig:sample_ks_traj}, and \ref{fig:sample_ks_0_frame_t24}.
\subsection{2D System: Navier-Stokes Flow Past Multiple Cylinders}
\label{ns-2d}
We introduce a large-scale multi-body flow benchmark to evaluate the model's capacity for spatial generalization across disconnected flow structures. The system is governed by the 2D incompressible Navier-Stokes equations:
\begin{align}
    \frac{\partial \mathbf{u}}{\partial t} + (\mathbf{u} \cdot \nabla) \mathbf{u} &= -\nabla p + \frac{1}{Re} \nabla^2 \mathbf{u}, \\
    \nabla \cdot \mathbf{u} &= 0.
\end{align}
We configure a computational domain with a spatial resolution of $128 \times 128$ containing three stationary cylinders, subjected to a uniform inflow velocity $\mathbf{u}_{\text{in}} = (0, 1)$. As the fluid encounters these obstacles, boundary layer separation occurs, leading to the formation of low-pressure wake regions and unsteady Von Kármán vortex streets downstream of each cylinder \citep{von1911mechanismus,roshko1954development}. A core challenge lies in the localized scale variations, where individually varying cylinder diameters $D_i$ induce wakes with distinct spatial scales and shedding frequencies, governed by the Strouhal relationship:
\begin{equation}
    f_s \approx \frac{St \cdot U}{D_i}.
\end{equation}
Consequently, the model must simultaneously capture these heterogeneous local dynamics, ranging from high-frequency eddies behind small obstacles to slower, large-scale structures behind massive ones, while correctly maintaining a laminar background flow in the void regions. In our current setup, the cylinders are arranged approximately perpendicular to the inflow direction, ensuring that the three wake regions remain spatially independent with negligible mutual interference. We reserve the investigation of strong wake interactions and moving obstacles for future work.

To ensure robust generalization over physical and geometric parameters, we construct a large-scale dataset of 10,100 simulation environments using Latin Hypercube Sampling (LHS) \citep{Stein01051987}. The Reynolds number is sampled uniformly from $Re \sim \mathcal{U}[50, 500]$, covering the transition from steady recirculation bubbles to fully developed periodic shedding. To prevent geometric overfitting and force the model to learn the underlying fluid-structure interactions, the positions and diameters $D \in [0.8, 2.0]$ of three cylinders are randomized. Ground truth trajectories are generated via the \texttt{Lily Pad} \citep{weymouth2015lilypadrealtimeinteractive} solver using the Boundary Data Immersion Method (BDIM) \citep{WEYMOUTH20116233} with a spatial resolution of $128 \times 128$ and a temporal step of $\Delta t = 0.01$. The final dataset consists of sequence lengths of $T=50$ and is partitioned into 8000 samples for training, 2000 for validation, and a held-out set of 100 samples for testing.

\paragraph{Sensor placement.}

We consider two complementary observation regimes for the 2D Navier-Stokes system. The first is a structured sparse setting based on a coarse-grained uniform grid: we downsample the high-resolution spatial domain ($\Omega \in \mathbb{R}^{128 \times 128}$) by a factor of $scale=16$, placing sensors at the centroids of non-overlapping $16 \times 16$ local patches (with a half-interval spatial offset). This yields $N_{obs} = (128/16)^2 = 64$ stationary measurement points, a regular but low-resolution view of the flow dynamics. An illustrative snapshot of the sensor layout is shown in Figure~\ref{fig:S13_T20_3channels}.

The second is a \emph{single-moving-sensor} observation mode, designed to test extreme data compression. At each physical time step, a single mobile sensor samples only one scalar from the $128\!\times\!128$ velocity field, yielding a staggering $16384\!:\!1$ compression ratio. The sensor follows a deterministic trajectory: a horizontal zig-zag sweep (20-frame period) coupled with a linear vertical descent over the total duration $T$, bounded by a $10\%$ spatial margin. To ensure fair and consistent evaluation, this identical predefined trajectory is rigidly applied across all test samples. Consequently, we assume an idealized sensor without collision avoidance constraints; the joint optimization of active exploration efficiency and physical safety is left for future work. The scan path and quantitative results are presented in Figure~\ref{fig:single_sensor} and Appendix~\ref{app:full_results}.

\begin{figure}
    \centering
    \includegraphics[width=1\linewidth]{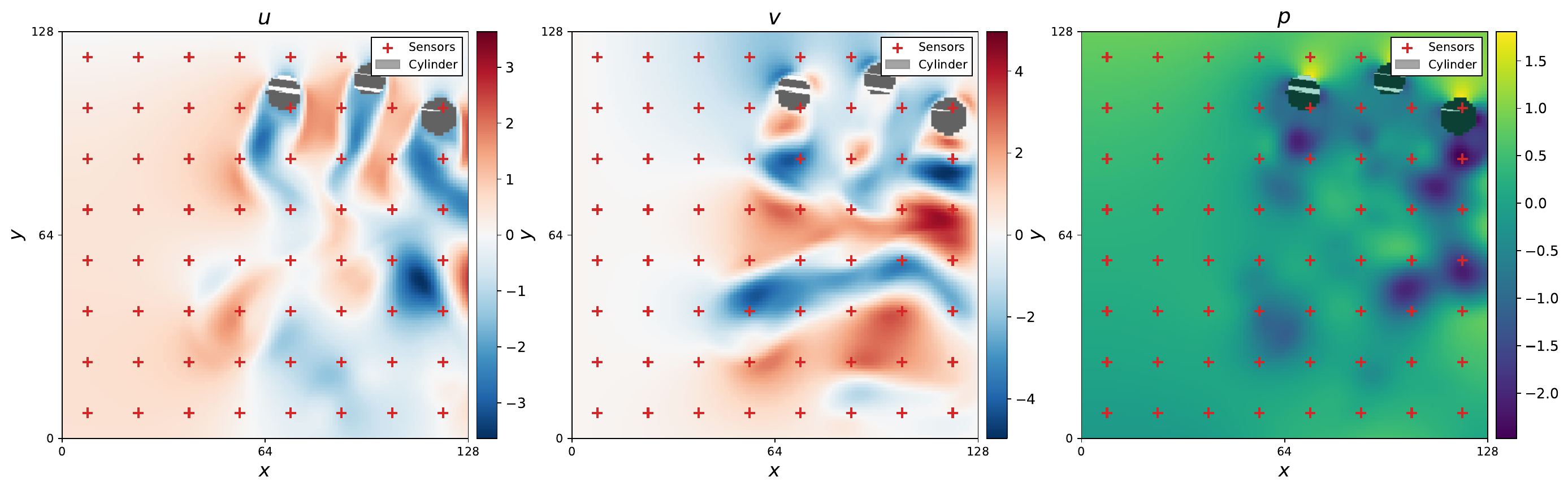}
    \caption{Illustration of the Navier-Stokes uniform $8\!\times\!8$ sensor placement and an example state.}
    \label{fig:S13_T20_3channels}
\end{figure}
\clearpage 
\begin{figure}
    \centering
    \begin{subfigure}{0.42\linewidth}
        \centering
        \adjincludegraphics[width=\linewidth, trim={0 0 {0.6667\width} 0}, clip]{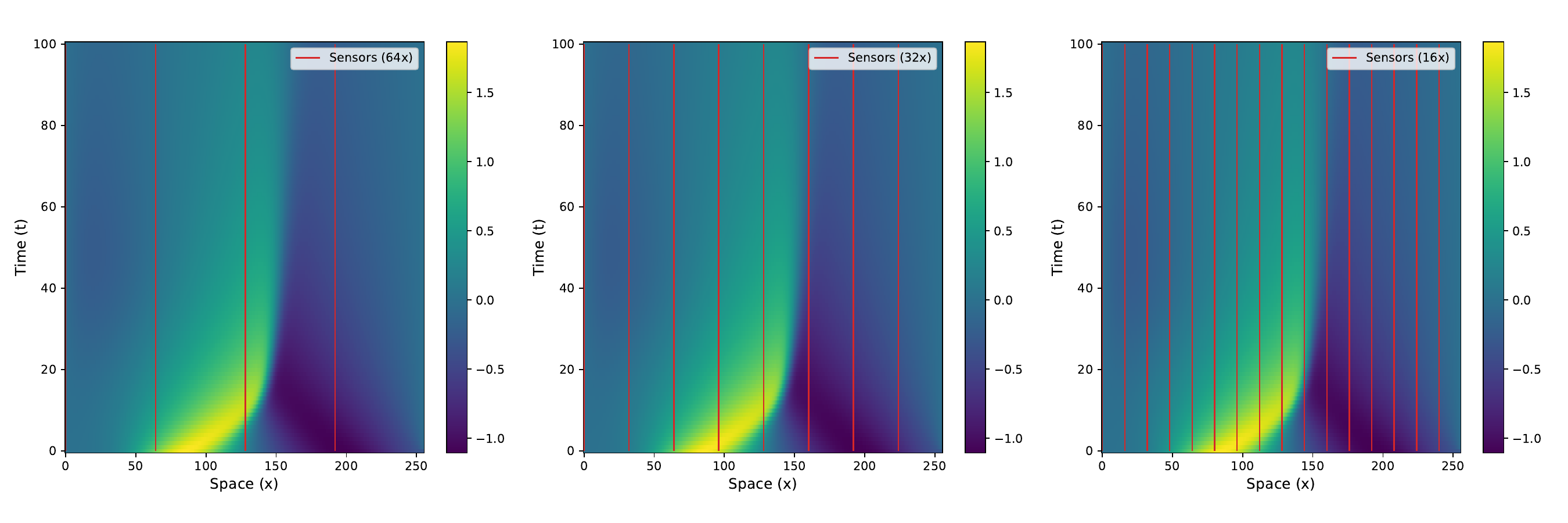}
        \caption{Entire trajectory.}
        \label{fig:sample_burgers_traj}
    \end{subfigure}\hfill
    \begin{subfigure}{0.48\linewidth}
        \centering
        \adjincludegraphics[width=\linewidth, trim={0 0 {0.6667\width} 0}, clip]{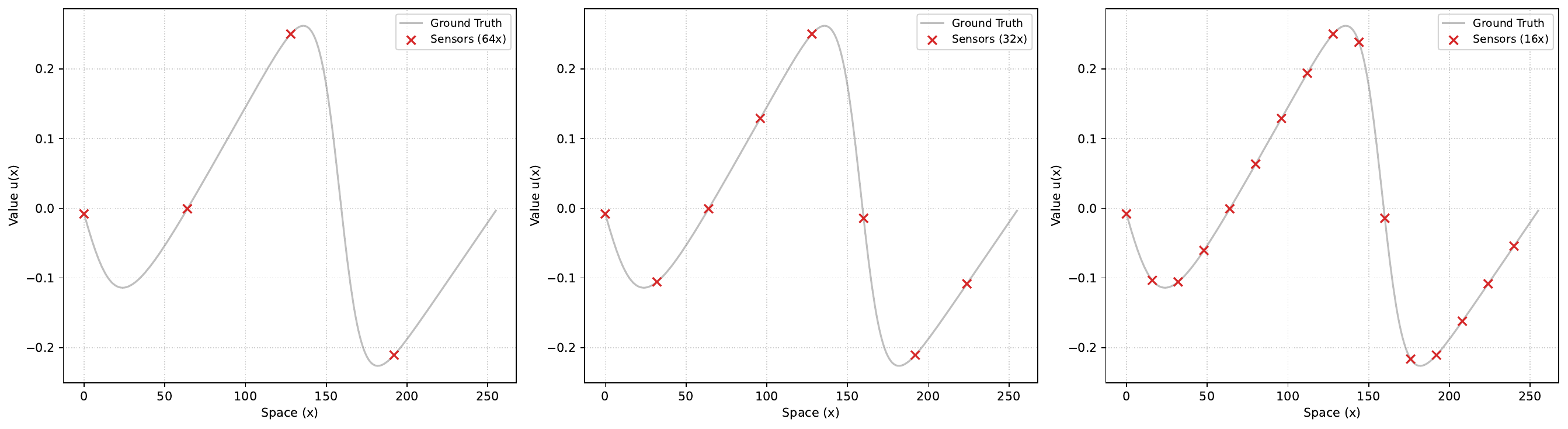}
        \caption{Single snapshot.}
        \label{fig:sample_burgers_0_frame_t100}
    \end{subfigure}
    \caption{Illustration of the Burgers' equation in the dataset. The four sensors are highlighted in red.}
\end{figure}

\begin{figure}
    \centering
    \begin{subfigure}{0.42\linewidth}
        \centering
        \adjincludegraphics[width=\linewidth, trim={0 0 {0.6667\width} 0}, clip]{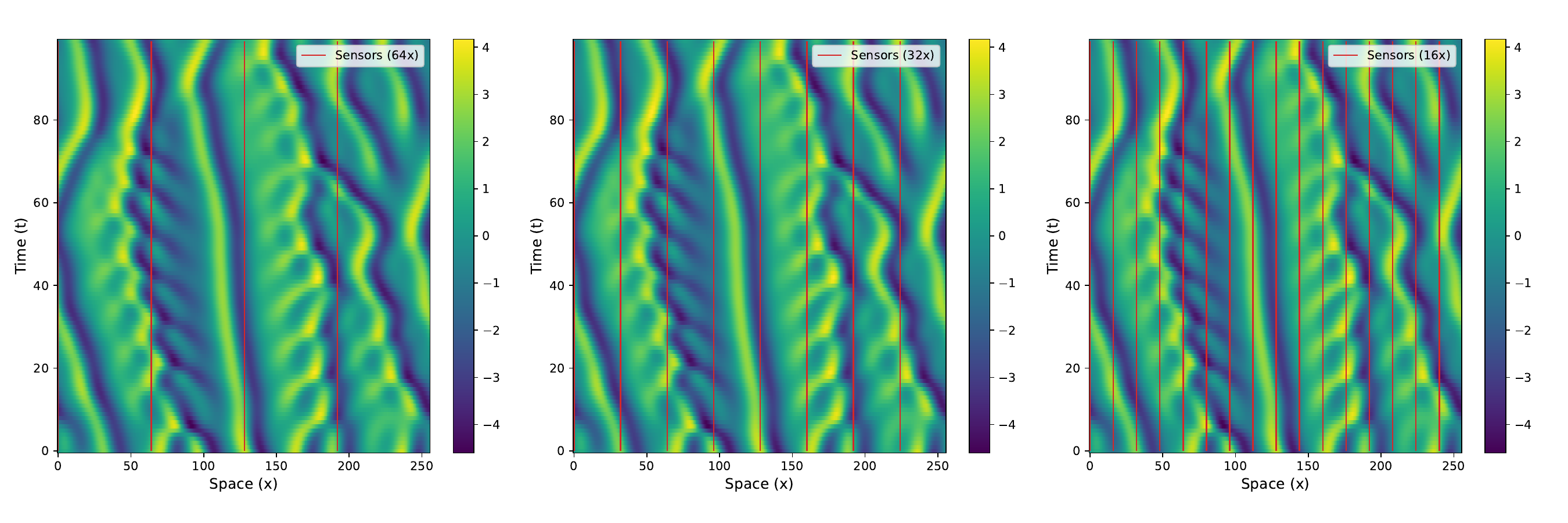}
        \caption{Entire trajectory.}
        \label{fig:sample_ks_traj}
    \end{subfigure}\hfill
    \begin{subfigure}{0.48\linewidth}
        \centering
        \adjincludegraphics[width=\linewidth, trim={0 0 {0.6667\width} 0}, clip]{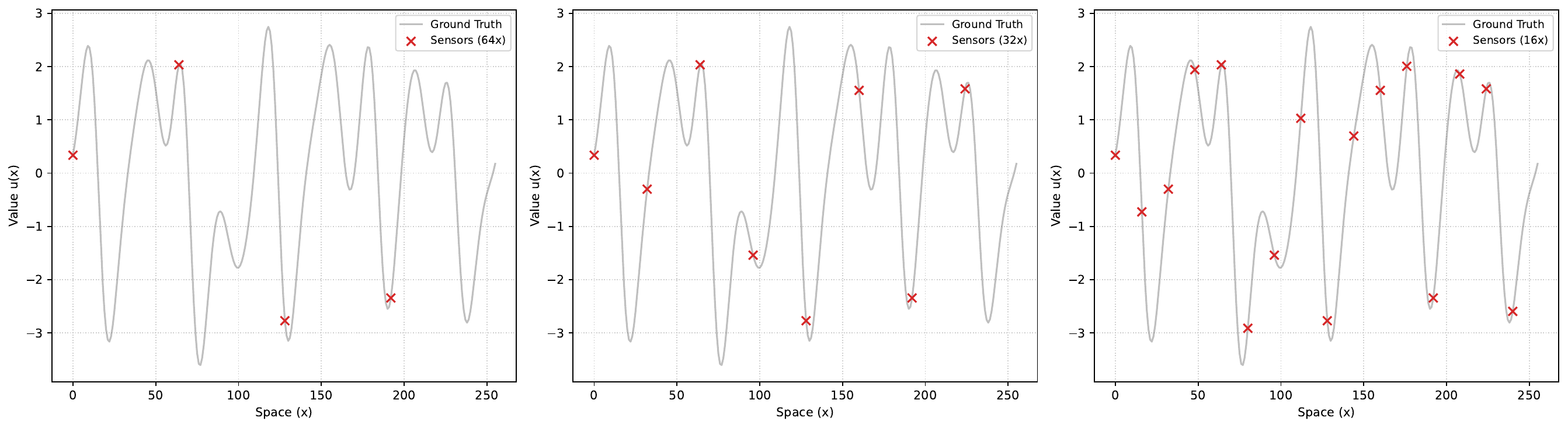}
        \caption{Single snapshot.}
        \label{fig:sample_ks_0_frame_t24}
    \end{subfigure}
    \caption{Illustration of the KS equation in the dataset. The four sensors are highlighted in red.}
\end{figure}
\clearpage

\subsection{Lorenz-63 System}
\label{app:lorenz}

The Lorenz-63 system is a classic three-dimensional chaotic ODE,
\begin{equation}
    \dot{X} = \sigma(Y-X), \qquad \dot{Y} = X(\rho-Z) - Y, \qquad \dot{Z} = XY - \beta Z,
\end{equation}
with the standard parameters $(\sigma, \rho, \beta) = (10, 28, 8/3)$, for which the system exhibits a butterfly-shaped strange attractor. We use it as a stress test for long-horizon recursive filtering: only the $Z$ coordinate is observed (with additive Gaussian noise, $\sigma_o=0.5$), while $X$ and $Y$ remain hidden. Because the Lorenz system is symmetric under the involution $(X,Y,Z)\mapsto(-X,-Y,Z)$ and the observation only sees $Z$, the filtering posterior is naturally bimodal in the $(X,Y)$-plane, providing a clean testbed for distributional accuracy.

\paragraph{Dataset.} Trajectories are integrated with an explicit fourth-order Runge-Kutta scheme at $\Delta t = 0.01$. To remove transients and ensure samples lie on the attractor, each trajectory is preceded by a $20$ time-unit burn-in starting from initial states drawn uniformly within $[-20,20]^3$. The dataset consists of $10^4$ training trajectories, each of length $T_{\text{train}}=100$, yielding $\approx\!1.2\times 10^5$ lobe-switching events in total, enough to densely cover the chaotic behavior of the attractor. Test trajectories are generated under the same protocol but are extended to $T_{\text{test}}\in\{1000, 4000\}$, i.e., $10\times$ and $40\times$ the training horizon, to evaluate generalization to far longer recursive inference.

\subsection{Tokamak Plasma Profile Estimation Benchmark}
\label{app:tokamak_dataset}

\paragraph{What a tokamak is, in one minute.}
A tokamak is a donut-shaped chamber that confines a hot ionized gas (a plasma) with strong magnetic fields, illustrated in the schematic of Figure~\ref{fig:tokamak}. To describe its state at any moment, plasma physicists collapse the 3D plasma volume to a 1D radial coordinate $\rho_N\in[0,1]$, where $\rho_N=0$ is the hot core at the centre of the donut's poloidal cross-section and $\rho_N=1$ is the cold plasma edge near the chamber wall. The full plasma state is then represented by three 1D profiles defined on this coordinate, the electron temperature $T_e(\rho_N)$, the electron density $n_e(\rho_N)$, and the magnetic flux $\psi(\rho_N)$. Together these three profiles summarise how energy and particles are distributed inside the plasma, and they evolve over time under a system of coupled, stiff, and nonlinear transport partial differential equations that play the role of the system dynamics.

\paragraph{Mapping the tokamak environment to a filtering problem.}
Cast in standard filtering terms, the tokamak environment fits cleanly into the partially observable Markov setup used throughout this paper. The state $s_t$ is the stack of three 1D profiles $(T_e, n_e, \psi)$ on $\rho_N$. The action $a_t$ encodes the operator's external knobs at time $t$, in particular the heating power and its spatial deposition through electron cyclotron heating. The transition $p(s_{t+1}\mid s_t, a_t)$ is given by the coupled transport solver described below. The observation $o_t$ comes from physical diagnostics mounted outside the chamber, which only report a sparse, often line-integrated, and nonlinear function of the state. The filter never sees the profiles themselves, only what the diagnostics report, and it must reconstruct a posterior over the full state $s_t$ from this stream of indirect measurements together with the running action sequence.

\paragraph{Why profile estimation matters in practice.}
Beyond serving as a clean filtering benchmark, accurate real-time estimation of $T_e$, $n_e$, and $\psi$ is a load-bearing capability for tokamak operation. Heating a plasma efficiently requires knowing where to deposit the heat, which in turn requires knowing how energy is currently distributed across the radial profile. Maintaining stable confinement requires the same information so that operators can react before instabilities have a chance to grow. This will become even more important for next-generation burning-plasma experiments such as ITER and follow-on commercial reactors, where the plasma must reach conditions in which fusion reactions release enough energy to keep the plasma hot on its own, and active control based on accurate profile estimates is needed to sustain these conditions over long pulses. A filter that reconstructs these profiles in real time from the available diagnostics is therefore not just a probabilistic-modelling exercise but also a practical building block for fusion control.

\paragraph{Why we additionally report the gradient error $\mathcal{E}_{\mathrm{grad}}$ on this benchmark.}
On the tokamak benchmark we report a fourth metric in addition to those defined in Appendix~\ref{app:metrics}, the channel-averaged gradient relative $L_2$ error
\begin{equation}
\mathcal{E}_{\mathrm{grad}} \;=\; \frac{1}{C}\sum_{c=1}^{C} \mathbb{E}_{t}\!\left[\frac{\lVert \partial_\rho \widehat{s}_t^{\,c} - \partial_\rho s_t^{*,c} \rVert_2}{\lVert \partial_\rho s_t^{*,c} \rVert_2}\right],
\end{equation}
where $\partial_\rho$ denotes the radial derivative of the channel-$c$ profile along the normalized coordinate $\rho_N\in[0,1]$, computed by finite differences on the discretization grid. The reason for this addition is that, in plasma operation, the radial \emph{gradients} of $T_e$ and $n_e$ are physically more important than their absolute values. Heat and particle fluxes are essentially driven by these gradients through transport coefficients, and most stability criteria used by tokamak physicists (such as edge-pedestal stability) are also gradient-driven. A filter whose mean prediction matches the truth in absolute value but smears the gradient out is therefore of limited practical value, since downstream analyses such as transport-coefficient estimation, turbulence diagnosis, and stability assessment all consume the local slope of the profile rather than its level.

\paragraph{Why this is a hard benchmark for AI filters.}
Three properties make this task a stress test for posterior estimators that have been trained mainly on PDE-like dynamics. First, the observation operator is highly nonlinear, because the POINT polarimetry channel reports a Faraday-rotation phase $\psi_F\!\propto\!\int n_e B_\parallel\,\mathrm d\ell$ where $B_\parallel$ is reconstructed from spatial derivatives of $\psi$, so the mapping is bilinear in the state pair $(n_e,\psi)$, and the ECE channel uses a signal-dependent noise model whose variance scales with the measured value. Second, the line-integrated diagnostics in the POINT system collapse the entire $n_e$ profile into a handful of scalar phases, which destroys most of the spatial information and forces the filter to lean heavily on its learned prior. Third, the system is actively controlled, so the filter must condition on the action sequence rather than treat the dynamics as autonomous, which differs from the unforced PDE benchmarks earlier in this paper.

The tokamak plasma profile estimation task targets a real-world filtering problem characterized by a highly nonlinear observation operator. The hidden state at each physical step $t$ consists of 1D electron-density $n_e(\rho_N)$, temperature $T_e(\rho_N)$, and flux $\psi(\rho_N)$ profiles defined on the normalized radial coordinate $\rho_N\!\in\![0,1]$ \citep{joung_gs-deepnet_2023}.
The evolution of these profiles is governed by coupled transport equations driven by external \textbf{actions} $a_t$ comprising the electron cyclotron power, deposition location, and width, which the filter must account for. 
Two types of diagnostics are included:
(1) electron cyclotron emission diagnostics with signal-dependent observation noise whose noise level is proportional to the signal, placed at 15 fixed locations on the radial coordinate $\rho_N$; and (2) the POINT system, which consists of 11 horizontal laser chords producing both an interferometry phase and a Faraday-rotation phase. It measures integrated line-of-sight quantities in the $(R,Z)$ plane, and is inherently nonlinear.

\textbf{Dataset Generation.} Ground-truth trajectories are produced by a 1D coupled transport solver (state scheme \texttt{coupled\_te\_psi}) anchored to the real EAST shot 152444. The flux geometry $\psi(R,Z)$ and the boundary $\rho_N(R,Z)$ map are taken directly from the EFIT g-file of that shot, while the base profiles for $T_e$, $n_e$ and $\psi$ are loaded from the same shot's reconstruction and used as the nominal initial condition. The electron temperature evolves through the Felici-Fable form-factor transport model under an implicit-Euler integrator, while $n_e$ is held at its randomized initial profile (electron density evolution is disabled to keep the focus on heat transport). Each trajectory spans $0.1$\,s of plasma evolution at $\Delta t=10^{-3}$\,s ($1$\,kHz, matching typical tokamak diagnostic rates), and snapshots are saved every $5$ solver steps to give $\approx 20$ frames per trajectory.

Inter-trajectory diversity is induced by two independent randomizations rather than a single global perturbation: (i) six initial-profile parameters (core gain, width, and edge gain for both $T_e$ and $n_e$) are drawn uniformly from explicit ranges around the shot-152444 base profiles, \textit{e.g.}\ $T_e$-core $\in[0.50,1.15]$, $T_e$-width $\in[0.32,0.80]$, $n_e$-core $\in[0.70,1.30]$; and (ii) the actuator trace is a piecewise-constant electron-cyclotron-heating (ECH) schedule with $K=4$ launchers, where for each segment the deposition radius $\rho_{\text{dep}}\!\in\![0.05,0.85]$, deposition width $w_{\text{dep}}\!\in\![0.12,0.28]$, and injected power $P_{\text{in}}\!\in\![0.2,0.9]$\,MW are resampled, segment durations are drawn uniformly from $[25,80]$\,ms, and each launcher independently has a $10\%$ probability of remaining off after activation.

\textbf{Observation Operator and Noise Model.} Two physical diagnostic systems are simulated, yielding three logical observation channels in total. (1)~An electron-cyclotron-emission (ECE) array reads $T_e$ at a set of fixed radial locations $\rho_N$ taken from \texttt{ECE\_Te\_shot152444.mat}, sampling the current $T_e$ profile at those $\rho_N$ values directly. (2)~A POINT-style polarimetry/interferometry system places $11$ horizontal sight-lines at fixed heights $z\!\in\!\{-0.425,\,-0.340,\,\ldots,\,+0.425\}$\,m and integrates $n_e$ along each sight-line (after a 512-point Riemann discretization in $R$) using a fixed $\rho_N(R,Z)$ map drawn from the EFIT g-file of shot 152444, which produces an interferometry phase $\phi$ that is linear in $n_e$ and a Faraday-rotation phase $\psi_F\!\propto\!\int n_e\,B_\parallel\,\mathrm d\ell$, where $B_\parallel$ is reconstructed from spatial derivatives of $\psi$. Because $B_\parallel$ depends on $\psi$, the Faraday channel is bilinear in the state pair $(n_e,\psi)$ and is the principal source of nonlinearity in the observation operator.

Diagnostic noise is added independently in physical units. ECE follows a signal-dependent Gaussian model $y=y_{\text{clean}}+\epsilon$, $\epsilon\!\sim\!\mathcal N(0,\sigma^2)$ with $\sigma=\sqrt{\sigma_{\text{abs}}^2+(\sigma_{\text{rel}}\,|y_{\text{clean}}|)^2}$, $\sigma_{\text{abs}}=20$\,eV and $\sigma_{\text{rel}}=0.03$. POINT phases use additive Gaussian noise of $\sigma_{\phi}=1.0^\circ$ for the interferometry channel (with an additional shared common-mode contribution of $0.35^\circ$ per chord-bundle), and $\sigma_{\psi_F}=0.1^\circ$ for the Faraday-rotation channel. Given the large scale disparity between physical quantities (e.g.\ $n_e\!\sim\!10^{20}$\,m$^{-3}$ vs.\ $T_e\!\sim\!10^{4}$\,eV), a per-channel, per-spatial-point mean/std normalization computed from the training set is applied prior to model ingestion.

\subsection{Evaluation Metrics}
\label{app:metrics}

A Bayesian filter targets the full posterior, so we evaluate along three complementary axes: \emph{mean accuracy}, \emph{spectral fidelity of individual samples}, and \emph{calibration of the posterior coverage}. Concretely, we report the channel-averaged relative $L_2$ error $\mathcal{E}_{L_2}$, the energy-spectrum error $\mathcal{E}_{\mathrm{spec}}$, and the per-pixel Miscalibration Area (MA).

\paragraph{Relative $L_2$ error ($\mathcal{E}_{L_2}$).} Let $s_t^{*,c}$ denote the ground-truth state of channel $c$ at time $t$, and $\widehat{s}_t^{\,c}$ the filtered mean estimate (sample average over the ensemble). We compute the relative $L_2$ error per channel and then average over channels:
\begin{equation}
\mathcal{E}_{L_2} \;=\; \frac{1}{C}\sum_{c=1}^{C} \mathbb{E}_{t}\!\left[\frac{\lVert \widehat{s}_t^{\,c} - s_t^{*,c} \rVert_2}{\lVert s_t^{*,c} \rVert_2}\right].
\end{equation}

\paragraph{Energy-spectrum error ($\mathcal{E}_{\mathrm{spec}}$).} A low mean error does not imply that individual samples respect the physical spectrum. We therefore evaluate spectral fidelity at the level of \emph{each generated sample}, not the ensemble mean. For sample $s_t^{(i)}$ we compute its energy spectrum $E^{(i)}(k)$ (1D Fourier spectrum for the 1D PDEs and the tokamak profiles, radially binned 2D Fourier spectrum for Navier-Stokes), take its logarithm, and evaluate the mean absolute error against the ground-truth log spectrum on a system-specific wavenumber band $[k_{\min},\,k_{\max}]$:
\begin{equation}
\mathcal{E}_{\mathrm{spec}} \;=\; \mathbb{E}_{t,\,i}\!\left[\frac{1}{k_{\max}-k_{\min}+1}\sum_{k=k_{\min}}^{k_{\max}} \big|\log E^{(i)}(k) - \log E^{*}(k)\big|\right].
\end{equation}
We fix $k_{\min}=1$ throughout and set $k_{\max}=200$ (KS), $500$ (Burgers'), $85$ (Navier--Stokes), and $200$ (tokamak), chosen to cover the resolved physical range of each system while excluding numerical-noise-dominated tails. Evaluating $\mathcal{E}_{\mathrm{spec}}$ per sample (rather than on the mean) is what distinguishes a probabilistic filter that produces physically realistic posterior draws from one whose mean is accurate but whose samples are over-smoothed or noise-like.

\paragraph{Per-pixel Miscalibration Area (MA).} To quantify whether \ourmethod's uncertainty aligns with the actual posterior spread, we adopt the miscalibration measure of~\citet{chung_uncertainty_2021}, evaluated independently at every spatial location. Given the observation history $h_t$, at each pixel $x$ we draw an ensemble of $100$ samples from \ourmethod and let $\widehat{Q}_\alpha(x;h_t)$ denote the empirical $\alpha$-quantile of the marginal predictive distribution at $x$ obtained from this ensemble. For a target coverage $q\in[0,1]$, the central predicted interval at $x$ is
\begin{equation}
Q_q(x;h_t) \;=\; \big(\widehat{Q}_{(1-q)/2}(x;h_t),\;\widehat{Q}_{(1+q)/2}(x;h_t)\big),
\end{equation}
and the miscalibration at coverage $q$ is the gap between the empirical coverage rate and the promised coverage:
\begin{equation}
\mathrm{MC}(q;x) \;=\; \big|\Pr_{s_t,\,h_t\sim P_{\mathrm{data}}}\!\big(s_t(x)\in Q_q(x;h_t)\big) \;-\; q\big|.
\end{equation}
Averaging over $q$ yields the per-pixel Miscalibration Area $\mathrm{MA}(x)=\int_0^1 \mathrm{MC}(q;x)\,\mathrm{d}q$, which we approximate by uniform discretization over $q$. We report the value averaged across all pixels and channels. Smaller MA indicates that the predicted intervals are simultaneously well-calibrated at every spatial location, not merely on average.

\section{Baseline Specifications}
\label{app:baselines}

In this appendix, we provide the technical specifications of our baselines.

We evaluate our method against a comprehensive set of baselines categorized into three distinct paradigms: classical data assimilation, deterministic supervised learning, and score-based generative modeling. Unless otherwise stated, neural baselines are tuned to have comparable parameter counts to our model for fair comparison.

\subsection{Classical Data Assimilation (DA) with Learned Surrogates}
To ensure a fair comparison with deep learning baselines, we strictly exclude analytical governing equations (e.g., Navier-Stokes) from the DA frameworks. Instead, we adopt a data-driven strategy where the system dynamics are learned from the training data. This setup evaluates the filters' ability to perform inference using imperfect, learned surrogates rather than ground-truth physics. We investigate two distinct data-driven DA approaches:

\paragraph{EKF.}
By default, we follow the standard Extended Kalman Filter formulation~\citep{kalman_new_1960}: the (non-linear) time evolution simulator computes Jacobian $\partial f_\theta/\partial \mathbf{x}$ via automatic differentiation since the corresponding solvers are rewritten in PyTorch with closed form implicit scheme (rather than multiple Picard or Newton iterations), and the full-state covariance is propagated through the linearized dynamics at each step. This vanilla implementation is used for the 1D Burgers' and KS systems, and the tokamak benchmark, where the ambient state dimension is moderate ($<1000$D).

For the 2D Navier-Stokes benchmark, however, propagating a full $\approx (128^2)^2$ covariance is computationally intractable, so we adopt a reduced-order variant denoted \textbf{POD-LKF}. We first apply Proper Orthogonal Decomposition (POD)~\citep{lumley1967structure,berkooz1993proper} to project the state $\mathbf{x} \in \mathbb{R}^{128^2}$ onto the top $r=30$ energetic modes. The temporal evolution of the latent coefficients $\mathbf{z} \in \mathbb{R}^{30}$ is then modelled as a global linear system $\mathbf{z}_{t+1} = \mathbf{A} \mathbf{z}_t$, with $\mathbf{A}$ identified by linear least-squares on training trajectories. Inference is carried out by a standard Linear Kalman Filter entirely within this latent space, which tracks the dominant flow features while avoiding any operation on the full ambient state~\citep{willcox2006unsteady}.

\paragraph{EnKF.}
By default, we use the standard stochastic Ensemble Kalman Filter~\citep{evensen2009sequential}: an ensemble of particles is propagated through the simulator to estimate the flow-dependent background covariance, sparse observations are assimilated with perturbed measurements. This configuration is used for the 1D Burgers' and KS systems and for the tokamak benchmark.

For the 2D Navier-Stokes benchmark, we instantiate the same EnKF inference loop with a resolution-invariant deep operator surrogate, and refer to this variant as \textbf{FNO-EnKF}. Specifically, we replace the learned surrogate with a Fourier Neural Operator (FNO)~\citep{li_fourier_2021} with $k=16$ modes and hidden width $d=64$, trained on the full physical state. To mitigate the stability issues common in long-term recursive prediction, the FNO is trained with an autoregressive pushforward strategy~\citep{brandstetter2022message}, unrolling the computation graph for $k=10$ steps to explicitly minimize error accumulation. The remaining EnKF hyperparameters ($N_e$, $\sigma_{obs}$, $\sigma_{proc}$) are kept identical to the default setting, so that the only change relative to the 1D case is the choice of surrogate.

\subsection{Supervised Learning Baselines}
We include supervised learning baselines to investigate the necessity of probabilistic modeling in current tasks. These deep learning methods learn a deterministic mapping $f_\theta$ from the observation and action sequences $\tau_{t}$ to the mean estimation of the posterior state $\hat{s}_t$. 
We introduce two Transformer-based \citep{vaswani_attention_2017} baselines that share the same temporal encoder architecture but differ in their decoding heads for state reconstruction.

We employ a standard Transformer encoder to process the sequence of historical observations $\tau_t=(o_{[1:t]},a_{[1:t-1]})$. The temporal dependencies are aggregated into the last token, which serves as a compact latent representation of the current system state. 
To enable parallel training, we input the entire sequence into the Transformer and compute loss with all tokens' outputs. $\hat{s}_t$ conditions only on past information $\tau_t$ is enforced by the causal attention mask in the Transformer model.

Next, we introduce the decoders utilized by the two baselines:

\paragraph{Transformer-CNN (Transformer).} In this baseline variant, we utilize a projection-based residual deconvolution decoder for spatial reconstruction. Specifically, the aggregated latent token is first linearly projected to the initial spatial dimension. Subsequently, the feature map is processed by $N$ layers of 1D transposed convolutions equipped with residual connections to generate the final state estimate.
We denote this baseline as ``Transformer'' in the text for brevity.

\paragraph{Transformer-FNO (FNO).} In this variant, we employ a Fourier Neural Operator (FNO) \citep{li_fourier_2021} as the decoding head. The latent token conditions the FNO, which performs global convolutions in the spectral domain to reconstruct the continuous function of the current state directly. We denote this baseline as ``FNO'' in the text for brevity.

\subsection{Diffusion Generative Model-Based Baselines}

\paragraph{Score-based diffusion assimilation (SDA).}
We adopt the trajectory inference framework proposed by \citet{rozet_score-based_2023}. SDA learns a joint score-based prior over state trajectories by approximating the global score with local conditional scores (Markov blankets). 
This allows the model to approximately capture temporal correlations efficiently.

We implement SDA framework using a continuous-time Variance Exploding (VE) SDE, discretized into 100 steps \citep{song_score-based_2021}. The model undergoes unconditional pretraining via denoising score matching, optimized using Adam with Exponential Moving Average (EMA) on model weights. During this pretraining stage, the model uses sequences consisting of $T=20$ snapshots as training units, where the score function $\nabla_{s_{[1:t]}} \log p_t(s_{[1:t]})$ is parameterized by a Video U-Net architecture to explicitly capture spatiotemporal correlations and learn the joint distribution. 

During inference, zero-shot reconstruction is performed using a Predictor-Corrector guidance sampler. Specifically, it utilizes a Langevin Corrector with a signal-to-noise ratio (SNR) of 0.128 and 1 corrector step per iteration. 
The pre-trained unconditional score is augmented with an analytically approximated likelihood gradient $\nabla_{s_{[1:t]}} \log p(h_t|s_{[1:t]})$ derived from the observation operator which, in Gaussian noise observation settings, simplifies to $\nabla_{s_{[1:t]}} \|h_t - \mathcal{H}(s_{[1:t]})\|_2^2$.

This method belongs to the trajectory-level posterior sampling category, modeling the conditional distribution $p(s_{[1:t]} | h_t)$ via iterative denoising, representing the state-of-the-art in generative Bayesian inference in physical systems.

\paragraph{Score-based generative model (S$^3$GM).} S$^3$GM \citep{li2024learning} also falls into the category of trajectory-level posterior sampling. It utilizes the same framework including training setup and backbone network as SDA, but balances the temporal consistency and computation effort by training multiple models with shorter time chunks and concatenating the partially denoised samples from adjacent time chunks during inference.

To achieve long-horizon sampling, the model generates 10 snapshots per call and constructs the full sequence by setting an overlap of $m=2$ snapshots. It concatenates partially denoised samples from adjacent time chunks to execute this process while ensuring sequence consistency.

\paragraph{Flow-based data assimilation (FlowDAS).}
FlowDAS \citep{chen_flowdas_2025} is a generative data assimilation framework based on stochastic interpolants. It learns a deterministic drift function conditioned on historical states. Consistent with other baselines, we instantiate the FlowDAS drift model using a U-Net architecture and the same sampling strategy as our flow-based BFF.

During the autoregressive inference phase, FlowDAS employs conditioning on samples generated in the last step to execute the dynamic propagation step and gradient guidance to conduct posterior sampling on sensor observations. To ensure a strictly fair comparison and avoid granting the baseline an unfair information advantage, we operate FlowDAS in the "observation-only" conditioning mode. Specifically, instead of using the unobserved ground-truth fields to initialize, the initial historical conditioning frames are reconstructed by sampling from the prior distribution, \textit{i.e.}, the training dataset rather than from the same test trajectory.

\section{Comprehensive Experimental Results}
\label{app:full_results}
\subsection{Computation Efficiency}
\label{app:efficiency}

\paragraph{The dominant cost is ODE integration, not the TTT update.}
We profile \ourmethod{}'s inference cost against the vanilla unconditional flow matching (FM) model (Figure~\ref{fig:inference_time_benchmark}). Across the commonly used range of $10$ to $200$ ODE steps, the TTT layer adds only marginal overhead on top of FM; per-step latency is essentially set by the number of ODE integration steps. The TTT weight update by itself is light: on the tokamak benchmark it accounts for only $5.11$\,ms per step (top row of Table~\ref{tab:ode_ablation}), while a $50$-step Euler ODE integration on top of it brings the total to $225.25$\,ms. Against external baselines (Table~\ref{tab:efficiency}), \ourmethod{} at $100$ ODE steps is already $\sim\!20\times$ faster than diffusion baselines (S$^3$GM, SDA), which must sample the full spatio-temporal trajectory at every physical step. Training is not as cheap as inference, since the inner-loop second-order derivative does add nontrivial cost, but the overhead remains moderate, with the additional training-step cost staying within $\sim\!30\%$ of vanilla flow matching in a large-batch regime (effective batch size $\!50$, Figure~\ref{fig:training_time_benchmark}).

\paragraph{Reducing ODE steps preserves point-estimation accuracy.}
Because ODE integration dominates inference cost, the natural lever for speedup is reducing the step count. We ablate this systematically on the tokamak benchmark (Table~\ref{tab:ode_ablation}): the relative-$L_2$ mean error $\mathcal{E}_{L_2}$ is essentially saturated, with even Euler at $5$ steps matching the $50$-step setting on $\mathcal{E}_{L_2}$ while dropping per-step latency from $225$\,ms to $19$\,ms. The same pattern holds on the 1D KS and 2D Navier--Stokes benchmarks: reducing the ODE steps from $100$ to $2$ leaves the point-estimation error nearly unchanged (and in some cases slightly \emph{improves} it, plausibly because fewer ODE steps act as an implicit regularizer that suppresses accumulated sampling noise), bringing the per-step cost to $\sim\!11$\,ms on KS and $\sim\!23$\,ms on Navier--Stokes. For real-time deployments such as tokamak plasma control, whose control loops typically run at $\sim\!10$\,Hz (i.e., $\sim\!100$\,ms per-step budget), this means \ourmethod{} comfortably fits the latency budget and can additionally be operated at very small step counts to leave headroom for downstream control computation without sacrificing point-estimation quality.

\paragraph{Calibration is more sensitive, but a higher-order solver recovers it.}
On the distributional axis the picture is different. On the tokamak benchmark, the Miscalibration Area (MA) degrades sharply when the Euler solver is run with too few steps ($0.0374\!\to\!0.6219$ as the step count drops from $50$ to $2$). This is a direct artifact of coarse flow integration: in flow matching, the learned velocity field $v_\theta(x_t,t)$ defines an ODE $\dot{x}_t=v_\theta(x_t,t)$ that transports a Gaussian prior to the target posterior along a continuous trajectory, and a few large Euler steps systematically under-shoot this trajectory wherever the flow is curved, so the discretized samples contract toward the posterior mean and underestimate its spread. Crucially, this is easy to fix: because the degradation lies in the numerical integrator rather than in \ourmethod{} itself, simply switching to a midpoint solver, which evaluates the flow at the segment midpoint and absorbs the leading curvature term, fully restores the multi-step MA at a fraction of the cost. At only $5$ midpoint steps we obtain MA$\,=\,0.0225$ at $36.92$\,ms/step, which is \emph{better} than $50$ Euler steps ($0.0374$ at $225.25$\,ms/step) on every metric.
We adopt the midpoint solver with $5$ ODE steps as the default \ourmethod{} configuration reported in the main text for the tokamak benchmark, since it delivers the lowest MA across all tested configurations while keeping per-step latency comparable to the small-step Euler runs.

\begin{table}[t]
\centering
\small
\caption{\textbf{Inference time and performance comparison.} All methods are benchmarked on a single NVIDIA H800 GPU. We report the update time per physical time step and the relative $L_2$ error ($\mathcal{E}_{L_2}\!\!\downarrow$) on the 1D KS equation (4 sensors) and 2D Navier-Stokes ($8\!\times\!8$ stationary sensors) benchmarks, both under noisy observation conditions ($\sigma=0.1$). The performance levels are maintained strictly consistent with the main results presented in Table~\ref{tab:all_datasets_results}.}
\label{tab:efficiency}
\setlength{\tabcolsep}{8pt} % 稍微放宽列间距，让精简后的表格更加舒展
\begin{tabular}{@{}lcccc@{}}
\toprule
\multirow{2}{*}{\textbf{Method}} & \multicolumn{2}{c}{\textbf{KS Equation}} & \multicolumn{2}{c}{\textbf{NS Equation}} \\
\cmidrule(lr){2-3} \cmidrule(lr){4-5}
 & $\mathcal{E}_{L_2}\!\!\downarrow$  & Time (s) & $\mathcal{E}_{L_2}\!\!\downarrow$ & Time (s) \\
\midrule
EKF                  & $1.4186$ & $0.015$ & $0.662$ & $0.016$ \\
FNO                      & $0.9873$ & $0.008$ & $0.309$ & $0.013$ \\
% Transformer              & $0.9621$ & $0.003$ & $0.309$ & $0.004$ \\
S$^3$GM (100 steps)      & $0.9828$ & $9.92$  & $1.135$ & $11.30$ \\
SDA (100 steps)          & $0.9622$ & $8.51$  & $0.195$ & $8.10$  \\
\midrule
\ourmethod (100 steps)   & $0.9061$ & $0.415$ & $0.165$ & $0.644$ \\
\ourmethod (2 steps)     & $0.8819 $   & $0.011$ & $0.162$ & $0.023$ \\
\bottomrule
\end{tabular}
\end{table}

\begin{figure}
    \centering
    \includegraphics[width=1\linewidth]{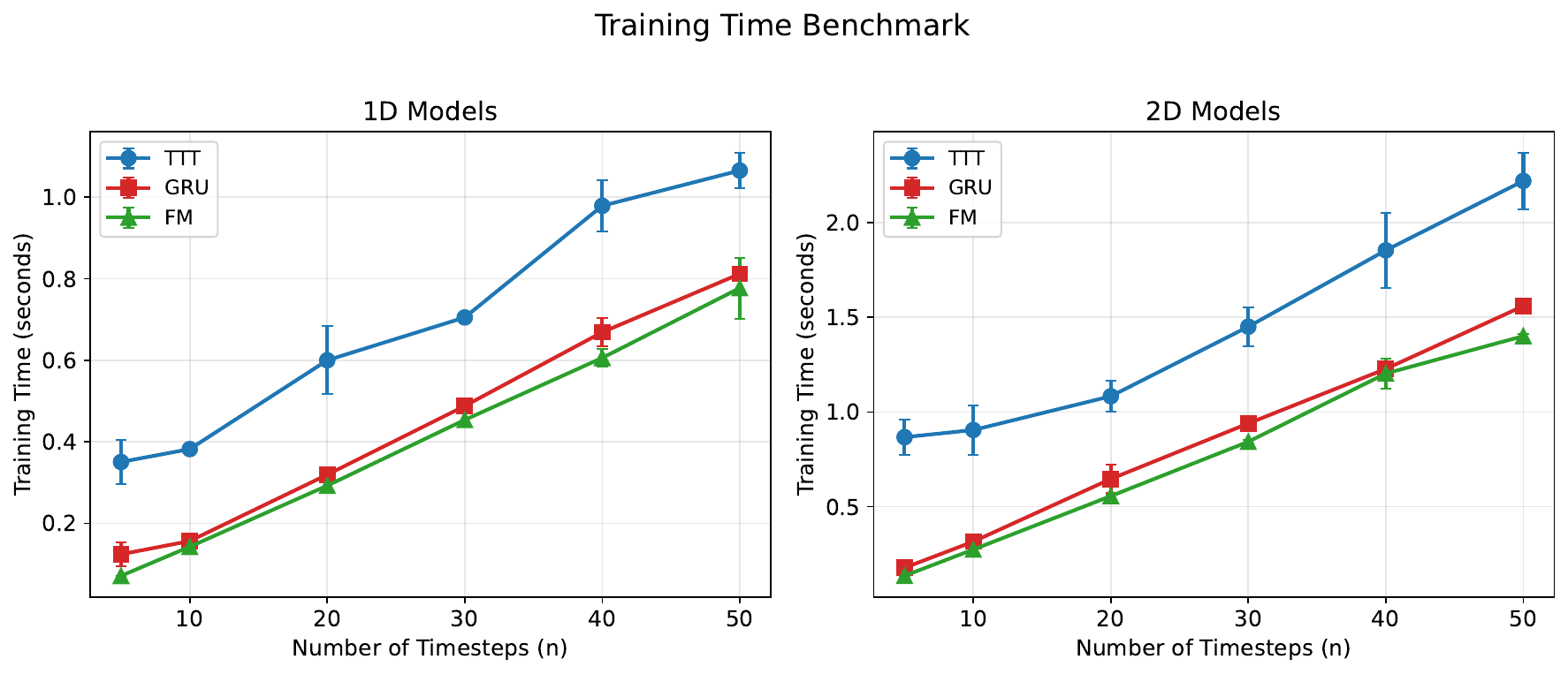}
    \caption{\textbf{Training wall-clock time comparison.} \ourmethod, the recurrent hidden state-based (GRU) variant, and the vanilla flow matching (FM) model across different numbers of unrolled timesteps.}
    \label{fig:training_time_benchmark}
\end{figure}

\begin{figure}
    \centering
    \includegraphics[width=1\linewidth]{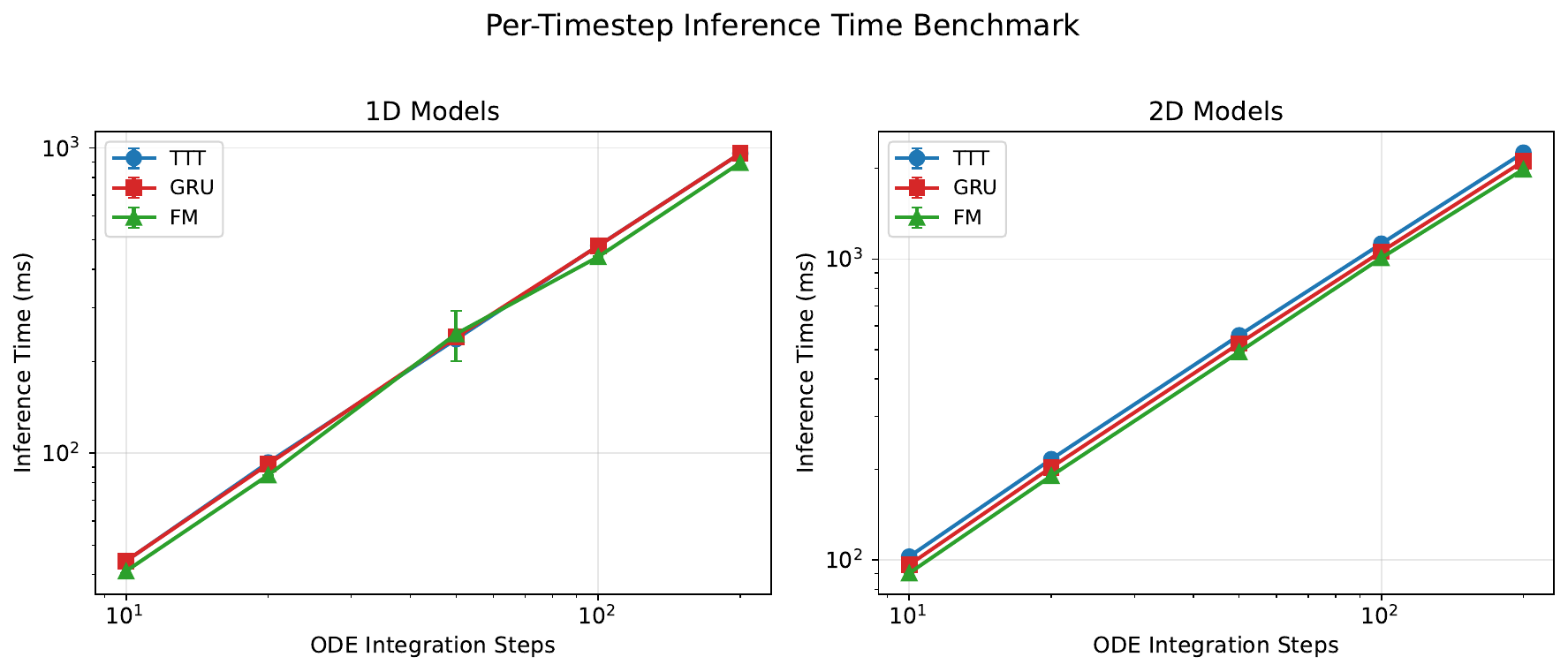}
    \caption{\textbf{Inference time comparison.} \ourmethod, the GRU variant, and the vanilla FM model across different ODE integration steps.}
    \label{fig:inference_time_benchmark}
\end{figure}

\begin{table}[tb]
    \centering
    \small
    \caption{\textbf{Ablation on ODE integration steps and solver choice on the tokamak benchmark.} The first row reports the cost of the TTT weight update alone (no ODE integration). Per-step inference latency is measured on a single NVIDIA H800 GPU. \textbf{Bold} denotes the best in each column.}
    \label{tab:ode_ablation}
    \label{app:ode_ablation}
    \begin{tabular}{@{}llccccc@{}}
    \toprule
    \textbf{Solver} & \textbf{ODE steps} & \textbf{ms/step $\downarrow$} & \textbf{$\mathcal{E}_{L_2}\downarrow$} & \textbf{$\mathcal{E}_{\mathrm{grad}}\downarrow$} & \textbf{$\mathcal{E}_{\mathrm{spec}}\downarrow$} & \textbf{MA $\downarrow$} \\
    \midrule
    \multicolumn{2}{@{}l}{TTT W update only (no ODE)} & 5.11 & --- & --- & --- & --- \\
    \midrule
    Euler & 50 & 225.25 & \textbf{0.0193} & \textbf{0.0471} & 0.0102 & 0.0374 \\
    Euler & 10 &  49.88 & \textbf{0.0193} & 0.0476 & 0.0108 & 0.1233 \\
    Euler &  5 &  18.66 & \textbf{0.0193} & 0.0490 & 0.0115 & 0.2529 \\
    Euler &  2 &   5.30 & 0.0197 & 0.0717 & 0.0150 & 0.6219 \\
    \midrule
    Midpoint & 5 &  36.92 & \textbf{0.0193} & 0.0486 & 0.0101 & \textbf{0.0225} \\
    Midpoint & 2 &   9.61 & 0.0198 & 0.0714 & \textbf{0.0095} & 0.1157 \\
    \bottomrule
    \end{tabular}
\end{table}

\subsection{Ablation Study: Weight-Space Belief Representation Is Better Than Latent Hidden States}
\label{app:ablation}
\begin{figure}
    \centering
    \includegraphics[width=0.5\linewidth]{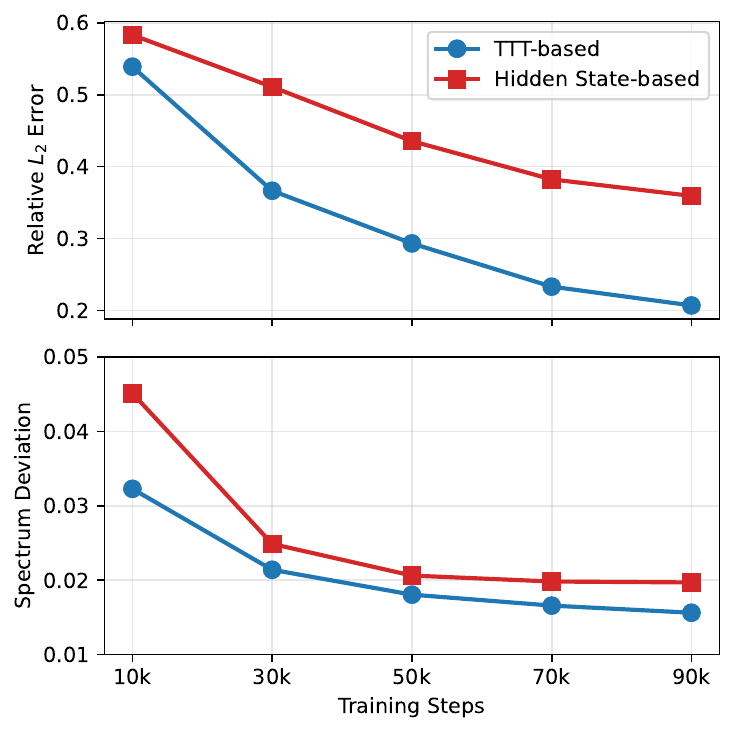}
    \caption{\textbf{Detailed training evolution of the ablation study on 2D Navier-Stokes.} The TTT mechanism significantly reduces tracking error and physical inconsistencies compared to standard recurrent architectures.}
    \label{fig:ns_2d_ablation}
\end{figure}
Vectors in latent space can also encode belief states without the need of a large number of particles.
We ablate the TTT block by replacing it with a flow matching model conditioned on a GRU hidden state (``hidden-state-based''), keeping the backbone, training data, and shared hyperparameters identical. As shown in Figure~\ref{fig:ns_2d_ablation}, the TTT block significantly improves performance on the complex 2D-flow posterior, exposing an inherent limit of recurrent hidden states: a fixed-dimensional latent cannot encode the full posterior that weight-space memory carries.

\subsection{Full Results on Lorenz-63 System}

As shown in Figure~\ref{fig:bimodal_visualizations}, \ourmethod's posterior remains symmetrically bimodal out to $T_{\text{test}}=1000$ ($10\times$ the training horizon, Figure~\ref{fig:lorenz_traj}), and the Wasserstein-2 distance to a $10^4$-particle bootstrap reference stays bounded out to $T=4000$ ($40\times$ training, Figure~\ref{fig:lorenz-w2}). 
The absence of error accumulation supports our design choice: % of modeling the Markov transition of the belief. 
As our TTT update is also Markovian, it is compatible with the sequential structure of Bayesian filtering, learning one-step transition correctly produces a stable filter that can operate recursively. 
Filtering anchors the posterior to fresh observations at every step, and the Markov TTT update re-aligns the weights by construction, so the learned operator $\mathcal{K}_\phi$ generalizes rather than overfits.

\paragraph{Results.} Figure~\ref{fig:lorenz-w2} reports the Wasserstein-2 distance between \ourmethod's posterior and a bootstrap particle filter with $10^4$ particles, averaged over $10$ test trajectories of length $T=4000$. The $W_2$ distance remains bounded over the entire $40\times$-training horizon, with no notable error accumulation, and the inset panels confirm that the posterior shape stays well-aligned with the particle-filter reference at selected times. Figure~\ref{fig:lorenz_traj} additionally visualizes the evolving posterior over a $T_{\text{test}}=1000$ trajectory: \ourmethod captures the two symmetric posterior modes simultaneously, with one mode consistently tracking the ground truth, and uncertainty correctly inflates near chaotic lobe-switching events, mirroring the system's local instability. Together, these findings show that the TTT update has internalized the recursive Bayesian operator rather than the training trajectory length: filtering anchors the posterior to fresh observations at every step, and the weight-level correction in \ourmethod re-aligns to each new observation by construction.

\begin{figure}[t]
    \centering
    \includegraphics[width=0.85\linewidth]{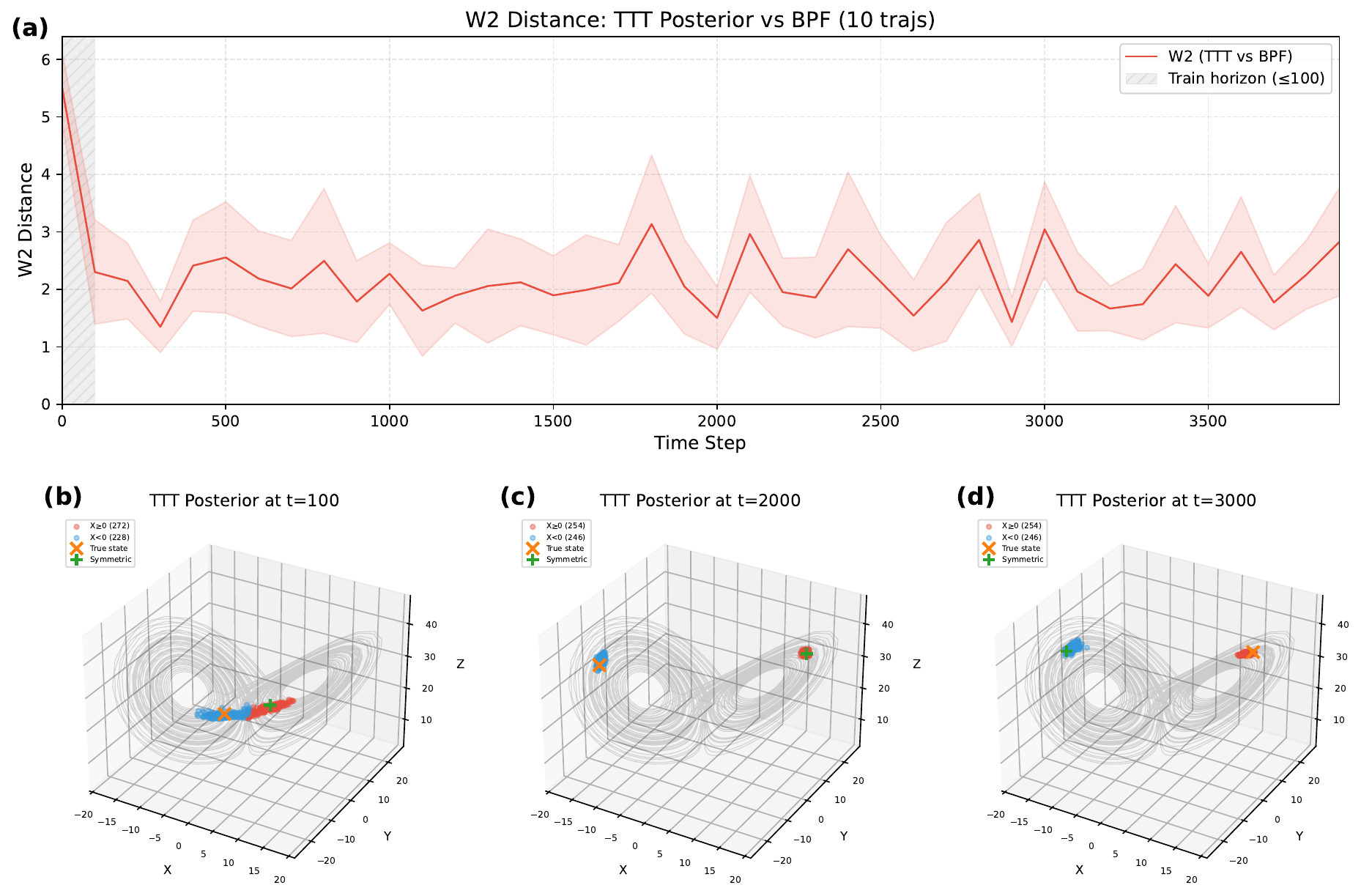}
    \caption{\textbf{Long-horizon stability on Lorenz-63.} We report the Wasserstein-2 distance between the ground-truth filtering posterior (estimated by a bootstrap particle filter with $10^4$ particles) and the posterior predicted by \ourmethod as a function of filtering time, extended to $T=4000$, that is, $40\times$ longer than training. Averaged over $10$ trajectories. The recursive update does not exhibit notable error accumulation. Inset panels at selected time steps further confirm that \ourmethod's posterior remains well-aligned with the ground truth.}
    \label{fig:lorenz-w2}
\end{figure}

\begin{figure}[!htbp]
    \centering
    \includegraphics[width=0.6\linewidth]{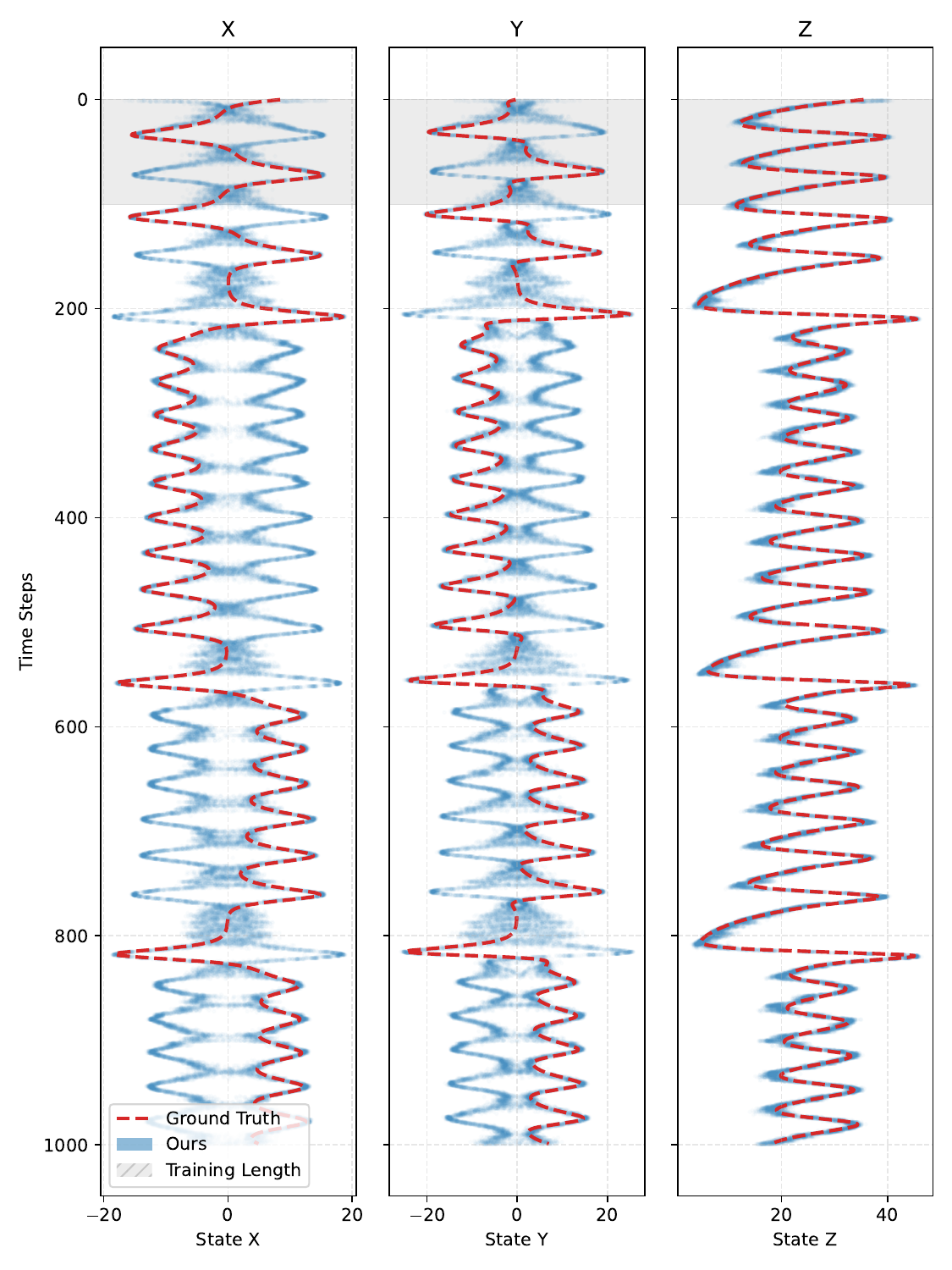}
    \caption{\textbf{Trajectory visualization on Lorenz-63.} $100$ samples per step (darker = higher density). \ourmethod captures the two symmetric modes, one consistently tracking ground truth, with uncertainty inflating near lobe switches.}
    \label{fig:lorenz_traj}
\end{figure}

\subsection{Trajectory Visualization on the 1D KS Benchmark}
\label{app:full_results_ks}

The relative $L_2$ error of \ourmethod on 1D KS is $\mathcal{E}_{L_2}\!\approx\!0.91$ (Table~\ref{tab:all_datasets_results}), which in absolute terms looks high and could be misread as the filter not having learned the dynamics. Figure~\ref{fig:combined_traj_ks} aims to resolve this concern by visualizing the recovered space-time fields against the ground truth. The \ourmethod posterior mean tracks the qualitative KS dynamics, including the locations and orientations of the dominant traveling structures and the spatial scale of the cellular pattern, throughout the rollout. The elevated rel-$L_2$ score is driven by a small number of localized regions where the predicted field is shifted in phase or amplitude relative to the true field, and these regions disproportionately inflate a normalized $L_2$ norm because KS exhibits sharp, large-amplitude features whose absolute error is amplified. Reading the metric alongside the visualization confirms that the rank ordering across methods is meaningful (\ourmethod is the best $\mathcal{E}_{L_2}$ on KS in Table~\ref{tab:all_datasets_results}), and that the absolute number reflects the difficulty of KS rather than a failure of the filter to capture the underlying flow.

\begin{figure}[tb]
    \centering
    \includegraphics[width=\linewidth]{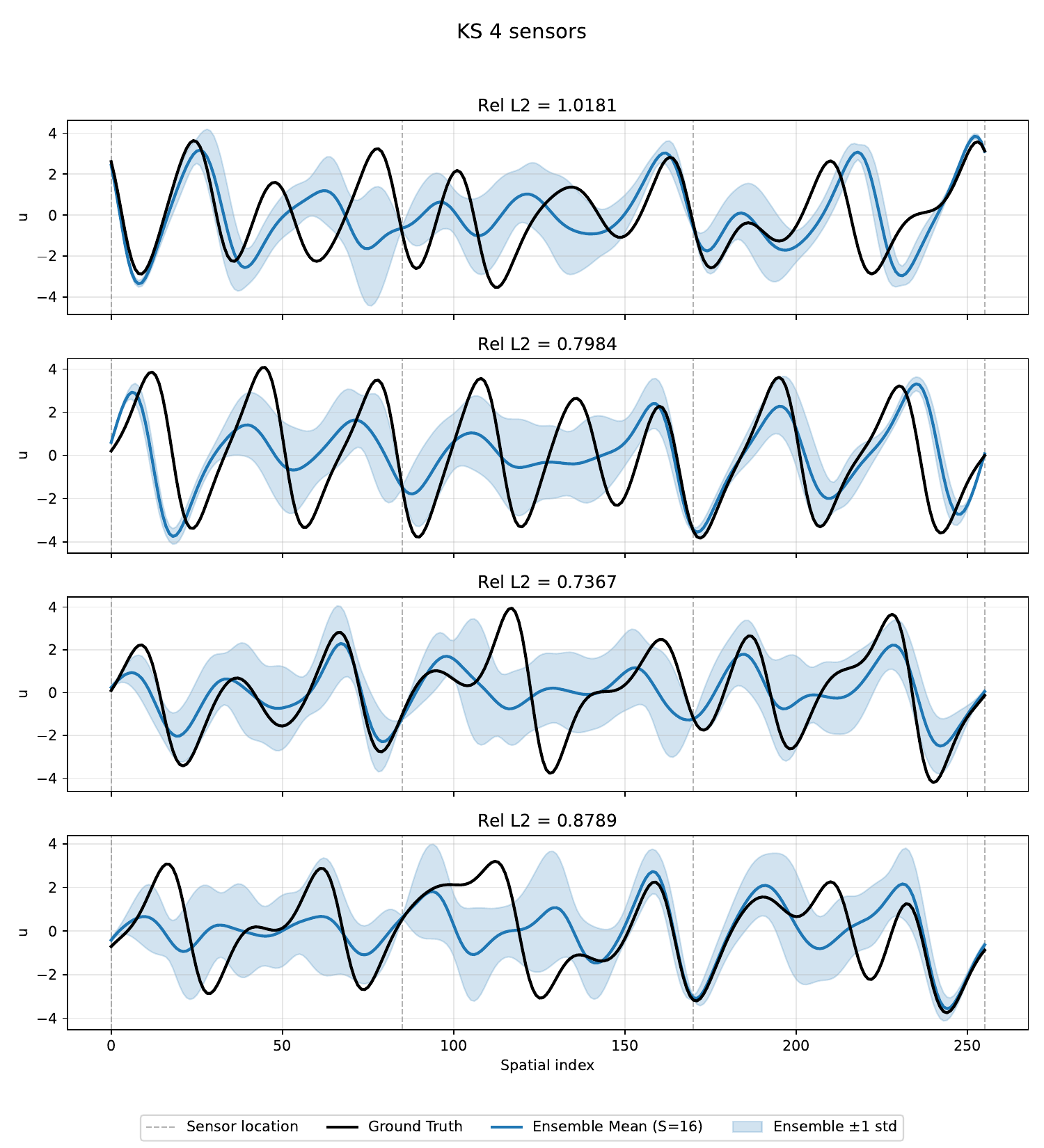}
    \caption{\textbf{Trajectory visualization on the 1D KS benchmark.} Space-time fields of the ground truth and the \ourmethod posterior mean over a representative test rollout. Despite the rel-$L_2$ score being $\sim\!0.9$ in absolute terms, \ourmethod recovers the qualitative dynamics of the KS equation, with the dominant traveling structures and cellular spatial pattern correctly placed. The remaining error is concentrated in where phase or amplitude deviate from the truth, and these regions dominate the normalized $L_2$ score because KS has large-amplitude sharp features.}
    \label{fig:combined_traj_ks}
\end{figure}

\subsection{Full Results on 2D NS System}
\label{app:full_reults_ns}
% Table~\ref{tab:merged_noisy_side_by_side} presents the results for the Burgers' and KS equation dataset, and 
Table~\ref{tab:all_datasets_results} includes the Navier-Stokes settings' results. All observations are corrupted with Gaussian noise ($\sigma=0.1$).

Beyond quantitative accuracy, we evaluate the physical consistency of the reconstructed fields. As illustrated in Figure~\ref{fig:vorticity_viz}, the 2D Navier-Stokes flow exhibits a complex turbulent pattern where wakes from multiple cylinders interact. While deterministic mean-prediction methods like FNO produce over-smoothed and non-physical blurring under uncertainty, \ourmethod successfully recovers physically consistent turbulent structures and sharp wake interactions.

Furthermore, we investigate the spatial awareness of the model under the challenging single moving sensor setting. As shown in Figure~\ref{fig:scan-prediction}, \ourmethod accurately reconstructs the implicit cylinder position and the surrounding flow field purely from sparse moving observations, whereas the Transformer baseline exhibits significant spatial deviation and blurring. To understand the dynamics of this active exploration, we track the binomial entropy and the ensemble standard deviation of the predicted probability maps for the cylinders and their associated wake regions (Figure~\ref{fig:entropy_std_drop}). As the exploration progresses and more informative measurements are assimilated into the belief state, both metrics exhibit a significant decline. This effectively demonstrates that \ourmethod actively reduces spatial ambiguity and precisely localizes the implicit obstacles over time.

Beyond structural fidelity, we further validate the probabilistic calibration and temporal stability of \ourmethod. As visualized in Figure~\ref{fig:vorticity_uncertainty_map}, we compute the vorticity for each generated sample and extract the standard deviation to form a spatial uncertainty map. This predicted uncertainty tightly correlates with the actual spatial error field, confirming that \ourmethod stably captures the underlying posterior distribution and accurately assigns higher uncertainty to regions with complex, unpredictable wake interactions.

Furthermore, we analyze the temporal dynamics of information assimilation in Figure~\ref{fig:ns_ttt_error_vs_time}. During the initial phase of the filtering process, the state estimation error drops rapidly as historical observations are sequentially accumulated. After approximately 10 time steps, the error stabilizes and maintains a steady plateau. This temporal dependency aligns perfectly with the Markovian nature of the filtering problem: while recent observations effectively reduce the uncertainty of the belief state, highly distant historical information provides diminishing returns, demonstrating that \ourmethod avoids error accumulation and achieves long-term stability

\subsection{Non-Gaussian Posterior on the Tokamak Benchmark}
\label{app:tokamak_posterior}

To complement the multimodal visualizations in the main text, we examine the shape of the posterior on the tokamak plasma profile estimation task, where the observation operator is highly nonlinear (Appendix~\ref{app:tokamak_dataset}). Figure~\ref{fig:tokamak_pair_corner_Te} shows a pair-plot of the electron-temperature residual $T_e - \widehat{T}_e$ at four selected radial coordinates $\hat\rho\in\{0.10,\,0.40,\,0.70,\,0.90\}$, aggregated over $N=20{,}000$ samples drawn from \ourmethod across trajectories, prediction horizons, and time steps.

\begin{figure}[tb]
    \centering
    \includegraphics[width=0.85\linewidth]{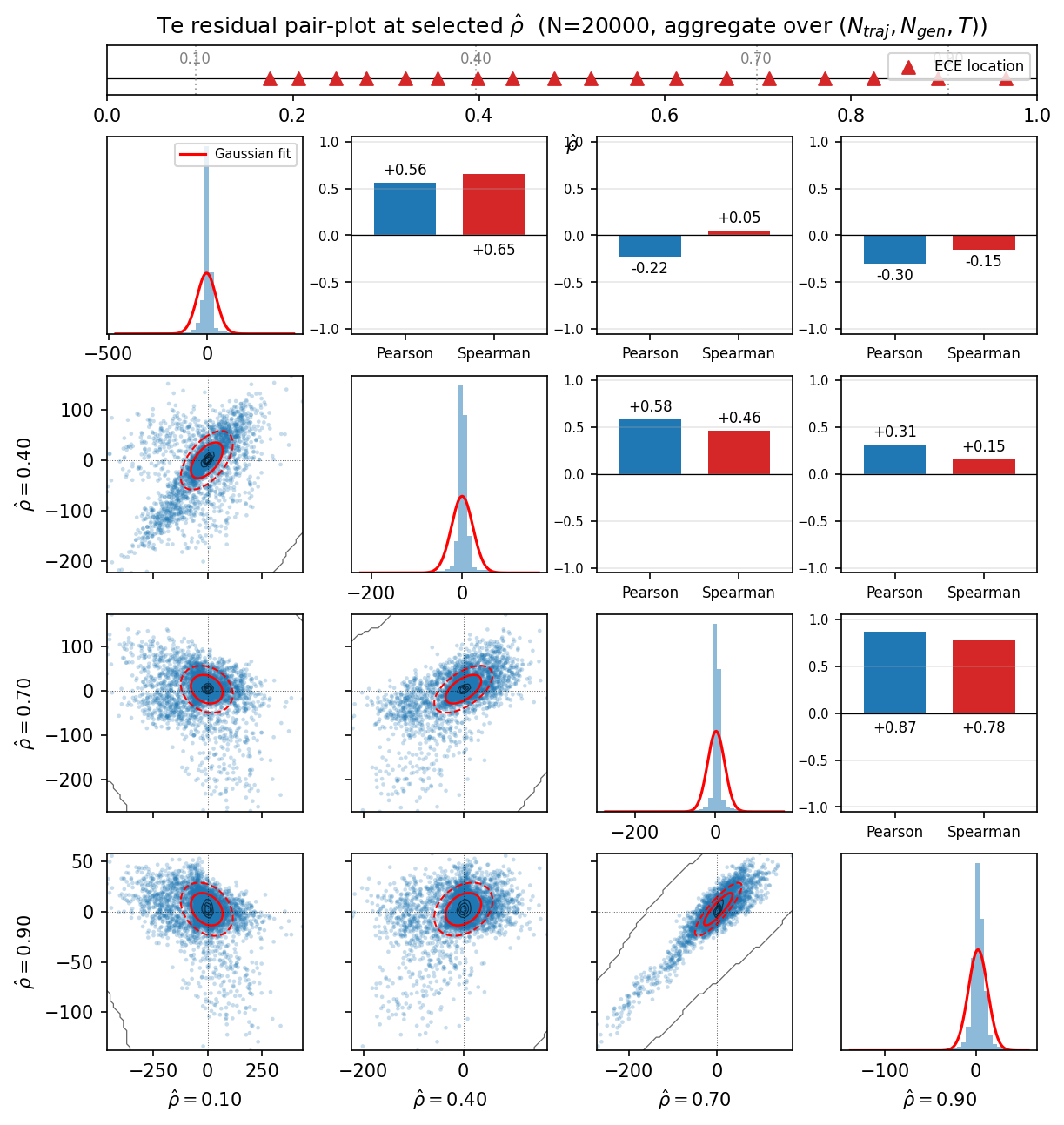}
    \caption{\textbf{Non-Gaussian posterior structure on the tokamak benchmark.} Pair-plot of the electron-temperature residual $T_e-\widehat{T}_e$ at four normalized radial coordinates $\hat\rho\in\{0.10,\,0.40,\,0.70,\,0.90\}$, aggregated over $N=20{,}000$ samples (across trajectories, prediction horizons, and time steps). The top strip marks ECE diagnostic locations along $\hat\rho\in[0,1]$ (red triangles). \textbf{Diagonal:} marginal residual distributions with a Gaussian fit (red); the marginals exhibit sharp peaks, heavier tails, and visible skew, deviating clearly from the best-fit Gaussian. \textbf{Lower triangle:} joint scatter plots between residual pairs, revealing nontrivial coupling structure: near independence at the core ($\hat\rho\!=\!0.10$ vs.\ $\hat\rho\!=\!0.40$) but strong, nearly collinear dependence near the edge ($\hat\rho\!=\!0.70$ vs.\ $\hat\rho\!=\!0.90$). \textbf{Upper triangle:} Pearson and Spearman correlation coefficients for each pair, summarizing the strength of these dependencies. Together, the marginals and joints confirm that \ourmethod represents a high-dimensional, non-Gaussian belief that a single Gaussian approximation cannot capture.}
    \label{fig:tokamak_pair_corner_Te}
\end{figure}

\begin{figure}[tb] 
    \centering
    \includegraphics[width=\linewidth]{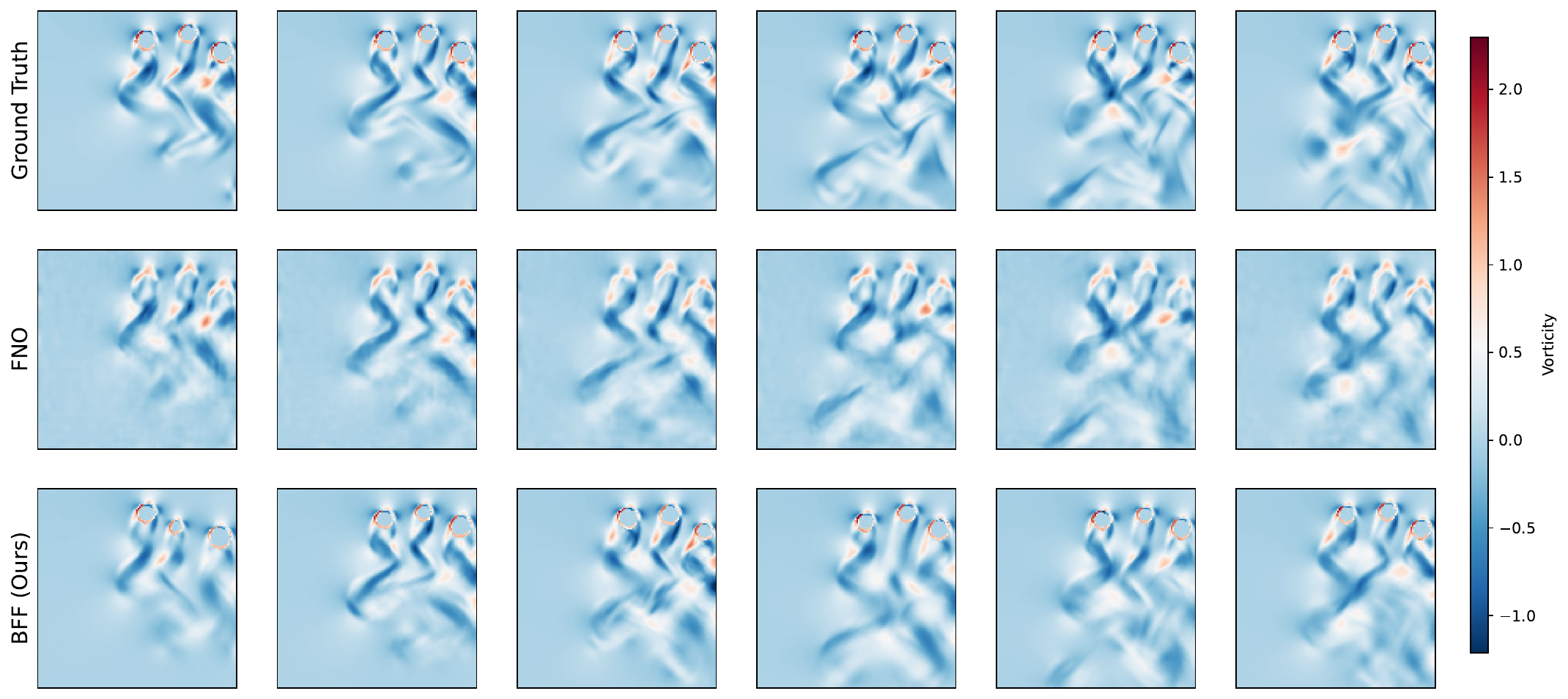}
    \caption{Qualitative comparison of 2D Navier-Stokes vorticity fields between the ground truth, FNO, and \ourmethod.}
    \label{fig:vorticity_viz}
\end{figure}

\begin{figure}[htbp]
  \centering
  \includegraphics[width=\linewidth]{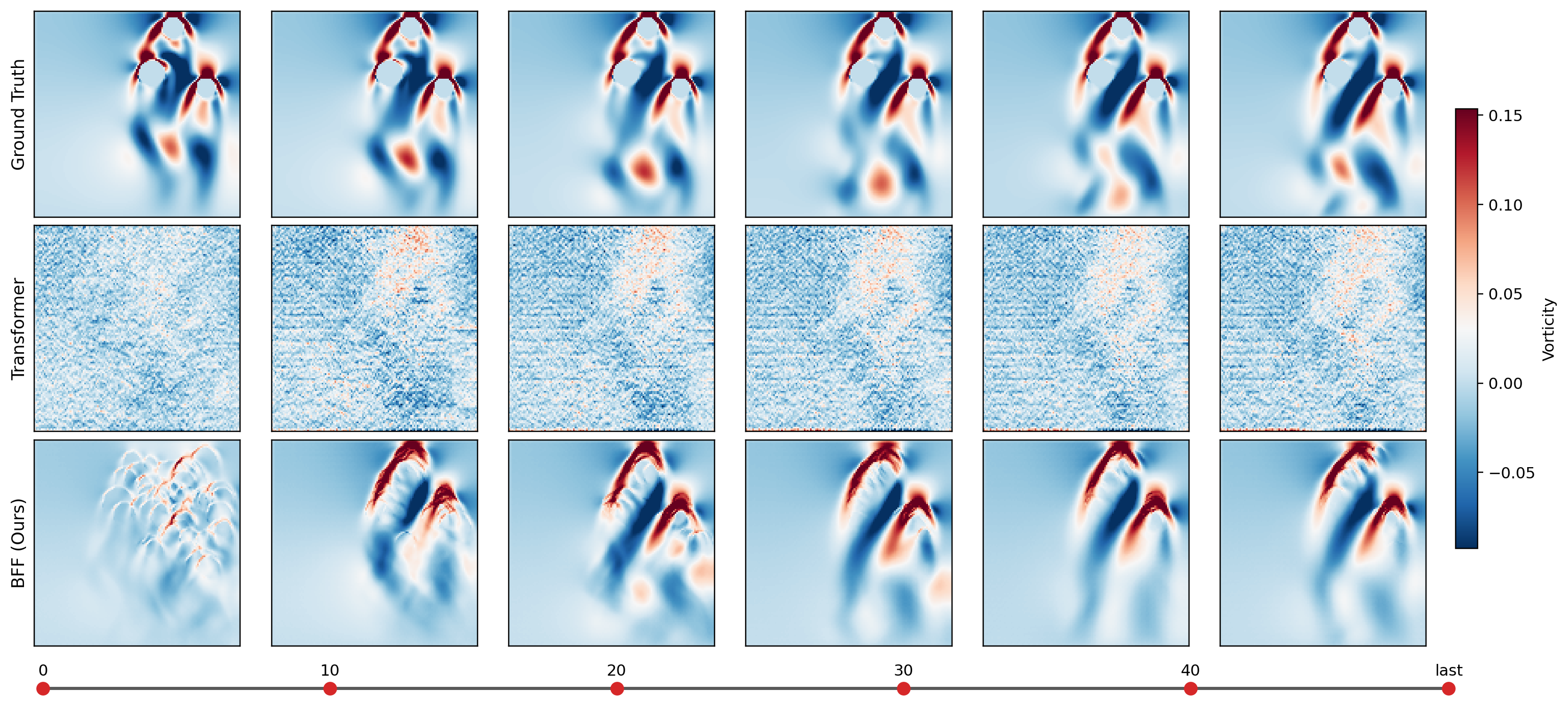}
  \caption{Qualitative comparison of single-sample predictions at the final frame under the moving sensor setting. Final $L_2$ error: \ourmethod ($0.330$) vs. Transformer ($0.769$).}
  \label{fig:scan-prediction}
\end{figure}

\begin{figure}[htbp]
  \centering
  \includegraphics[width=\linewidth]{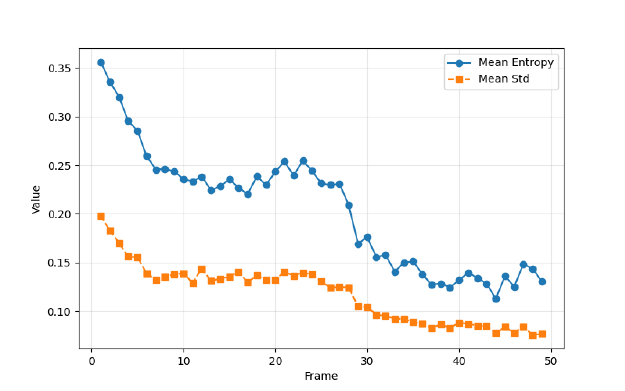}
  \caption{Evolution of spatial uncertainty under the single moving sensor setting. The uncertainty is quantified by tracking the binomial entropy and the ensemble standard deviation of the predicted probability maps for the cylinders and their associated wake regions.}
  \label{fig:entropy_std_drop}
\end{figure}

\begin{figure}[tb]
    \centering
    \includegraphics[width=\linewidth]{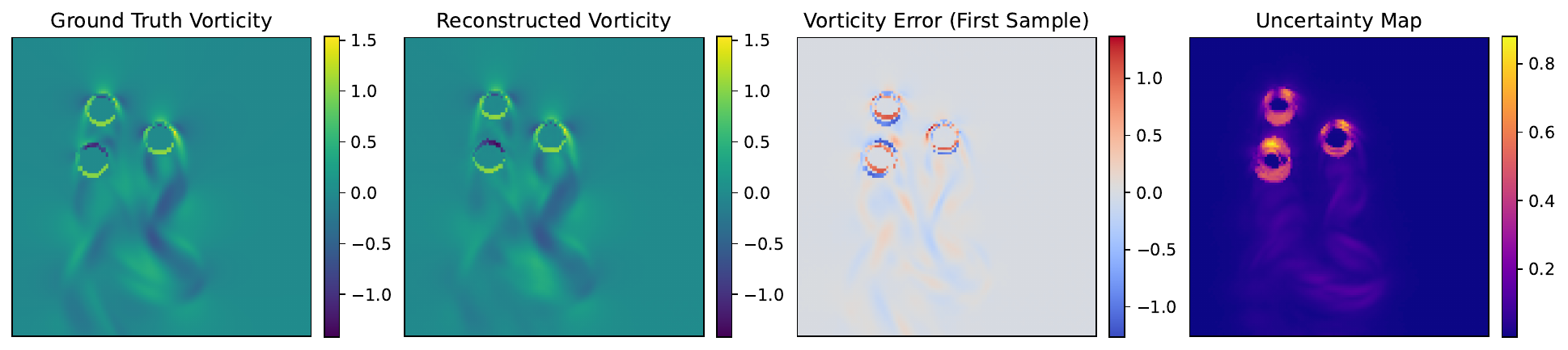}
    \caption{Qualitative comparison between the spatial error field and the predicted uncertainty map (standard deviation of sample vorticities) for the 2D Navier-Stokes benchmark.}
    \label{fig:vorticity_uncertainty_map}
\end{figure}

\begin{figure}[tb]
    \centering
    \includegraphics[width=0.5\linewidth]{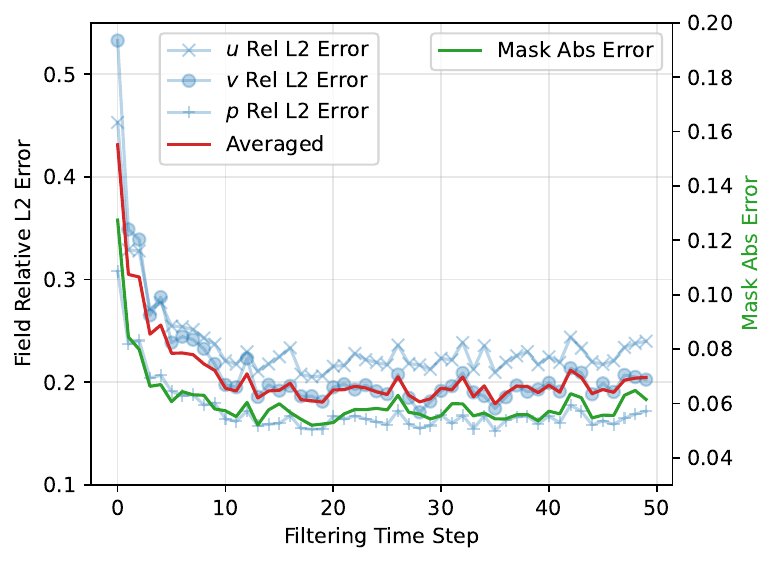}
    \caption{Temporal evolution of the state estimation error over filtering time steps for the 2D Navier-Stokes system.}
    \label{fig:ns_ttt_error_vs_time}
\end{figure}

%%%%%%%%%%%%%%%%%%%%%%%%%%%%%%%%%%%%%%%%%%%%%%%%%%%%%%%%%%%%

\newpage

\clearpage
\section*{NeurIPS Paper Checklist}

\begin{enumerate}

\item {\bf Claims}
    \item[] Question: Do the main claims made in the abstract and introduction accurately reflect the paper's contributions and scope?
    \item[] Answer: \answerYes{}
    \item[] Justification: The abstract and Section~\ref{sec:introduction} state three concrete contributions---weights-as-belief-state perspective, BFF with theoretical guarantees, and empirical validation on 1D/2D PDE and Tokamak benchmarks---each of which is supported by the corresponding methodological (Sec.~\ref{sec:ttt}--\ref{sec:training}) and experimental (Sec.~\ref{sec:experiments}) sections.

\item {\bf Limitations}
    \item[] Question: Does the paper discuss the limitations of the work performed by the authors?
    \item[] Answer: \answerYes{}
    \item[] Justification: A dedicated ``Limitations'' paragraph in the Conclusion enumerates five concrete limitations (finite-capacity gap of the linear $f_\phi,g_\phi$, the deterministic-transition assumption in the proof, 3D scalability, cross-regime/Reynolds-number generalization, and second-order training cost), with further discussion in Appendix~\ref{app:related_work}.

\item {\bf Theory assumptions and proofs}
    \item[] Question: For each theoretical result, does the paper provide the full set of assumptions and a complete (and correct) proof?
    \item[] Answer: \answerYes{}
    \item[] Justification: All theoretical results--- Proposition~\ref{prop:sec23-emp-filt}, Proposition~\ref{prop:sec23-latent-amortization}, and Theorem~\ref{theorem:ttt_outer_loss}---are stated with explicit assumptions (differentiability, Lipschitz regularity, deterministic transition kernel) and proved in Appendices~\ref{app:two-channel} and~\ref{app:loss_proof}.

    \item {\bf Experimental result reproducibility}
    \item[] Question: Does the paper fully disclose all the information needed to reproduce the main experimental results of the paper to the extent that it affects the main claims and/or conclusions of the paper (regardless of whether the code and data are provided or not)?
    \item[] Answer: \answerYes{}
    \item[] Justification: Section~\ref{sec:results} and Appendices~\ref{app:dataset}--\ref{app:baselines} fully specify the BFF backbone (DiT), the TTT block design, the surrogate-loss heads, the training/inference procedures (Algorithms~\ref{alg:ttt_train}--\ref{alg:ttt_test}), the data-generation protocols, sensor placements, noise levels, baseline implementations, and ODE step counts; an anonymized code repository is also provided.

\item {\bf Open access to data and code}
    \item[] Question: Does the paper provide open access to the data and code, with sufficient instructions to faithfully reproduce the main experimental results, as described in supplemental material?
    \item[] Answer: \answerYes{}
    \item[] Justification: An anonymized code repository (\url{https://anonymous.4open.science/r/BeliefFlow-81D4}) is released with training/inference scripts, data-generation code for all PDE benchmarks, configuration files, and reproduction instructions; the PDE datasets are simulator-generated and fully reproducible from the documented solver settings in Appendix~\ref{app:dataset}.

\item {\bf Experimental setting/details}
    \item[] Question: Does the paper specify all the training and test details (e.g., data splits, hyperparameters, how they were chosen, type of optimizer) necessary to understand the results?
    \item[] Answer: \answerYes{}
    \item[] Justification: Sensor configurations, noise levels, train/val/test splits, optimizer/learning-rate/batch-size choices, ODE-sampling steps, and per-experiment hyperparameters are reported in Appendix~\ref{app:implementation} and Appendices~\ref{app:dataset}--\ref{app:baselines}, with the remaining details available in the released code.

\item {\bf Experiment statistical significance}
    \item[] Question: Does the paper report error bars suitably and correctly defined or other appropriate information about the statistical significance of the experiments?
    \item[] Answer: \answerYes{}
    \item[] Justification: While empirical error bars across random seeds are omitted due to the prohibitively high computational cost of the coupled physics solver and other Bayesian filters such as S$^3$GM and SDA, we explicitly report "other appropriate information" regarding statistical significance. Specifically, the inherent Uncertainty Quantification (UQ) of our Bayesian filtering framework provides the posterior variance of the state estimates, which serves as the rigorous and physically meaningful measure of uncertainty and statistical confidence in our SciML setup.

\item {\bf Experiments compute resources}
    \item[] Question: For each experiment, does the paper provide sufficient information on the computer resources (type of compute workers, memory, time of execution) needed to reproduce the experiments?
    \item[] Answer: \answerYes{}
    \item[] Justification: All wall-clock measurements in Table~\ref{tab:efficiency} and Figures~\ref{fig:training_time_benchmark}--\ref{fig:inference_time_benchmark} are obtained on a single NVIDIA H800 GPU; per-step training and inference costs (and the $\sim\!30\%$ overhead from second-order TTT gradients) are reported in Appendix~\ref{app:full_results}.
    
\item {\bf Code of ethics}
    \item[] Question: Does the research conducted in the paper conform, in every respect, with the NeurIPS Code of Ethics \url{https://neurips.cc/public/EthicsGuidelines}?
    \item[] Answer: \answerYes{}
    \item[] Justification: The work uses only simulator-generated PDE data and a Tokamak transport-equilibrium benchmark with no human subjects, no scraped or sensitive data, and no foreseeable dual-use risk; the research conforms to the NeurIPS Code of Ethics.

\item {\bf Broader impacts}
    \item[] Question: Does the paper discuss both potential positive societal impacts and negative societal impacts of the work performed?
    \item[] Answer: \answerNA{}
    \item[] Justification: The contribution targets state estimation in physical systems (fluid dynamics, plasma) and is several steps removed from any direct societal-impact deployment; we identify no concrete path from this method to negative social impact.
    
\item {\bf Safeguards}
    \item[] Question: Does the paper describe safeguards that have been put in place for responsible release of data or models that have a high risk for misuse (e.g., pre-trained language models, image generators, or scraped datasets)?
    \item[] Answer: \answerNA{}
    \item[] Justification: We do not release pretrained generative models for natural images, language, or other content with misuse potential; the released checkpoints model PDE belief states on scientific filtering benchmarks and pose no high-risk dual-use concern.

\item {\bf Licenses for existing assets}
    \item[] Question: Are the creators or original owners of assets (e.g., code, data, models), used in the paper, properly credited and are the license and terms of use explicitly mentioned and properly respected?
    \item[] Answer: \answerYes{}
    \item[] Justification: All third-party assets we build on---DiT~\citep{peebles_scalable_2023}, FNO~\citep{li_fourier_2021}, SDA~\citep{rozet_score-based_2023}, S$^3$GM~\citep{li2024learning}, the LilyPad solver~\citep{weymouth2015lilypadrealtimeinteractive}, and classical Kalman/EnKF references---are properly cited at the points of use; the PDE datasets are generated by our own solvers under the parameters specified in Appendix~\ref{app:dataset}.

\item {\bf New assets}
    \item[] Question: Are new assets introduced in the paper well documented and is the documentation provided alongside the assets?
    \item[] Answer: \answerYes{}
    \item[] Justification: The released anonymized code repository ships with a README, training/inference instructions, configuration files for every benchmark, dataset-generation scripts, and a license file documenting allowed use.

\item {\bf Crowdsourcing and research with human subjects}
    \item[] Question: For crowdsourcing experiments and research with human subjects, does the paper include the full text of instructions given to participants and screenshots, if applicable, as well as details about compensation (if any)?
    \item[] Answer: \answerNA{}
    \item[] Justification: The paper does not involve crowdsourcing or research with human subjects; all data come from numerical PDE solvers.

\item {\bf Institutional review board (IRB) approvals or equivalent for research with human subjects}
    \item[] Question: Does the paper describe potential risks incurred by study participants, whether such risks were disclosed to the subjects, and whether Institutional Review Board (IRB) approvals (or an equivalent approval/review based on the requirements of your country or institution) were obtained?
    \item[] Answer: \answerNA{}
    \item[] Justification: No human-subject research is conducted, so IRB approval is not applicable.

\item {\bf Declaration of LLM usage}
    \item[] Question: Does the paper describe the usage of LLMs if it is an important, original, or non-standard component of the core methods in this research? Note that if the LLM is used only for writing, editing, or formatting purposes and does \emph{not} impact the core methodology, scientific rigor, or originality of the research, declaration is not required.
    %this research?
    \item[] Answer: \answerNA{}
    \item[] Justification: LLMs were not used as a component of the core methodology; any LLM use was limited to writing/editing/formatting and therefore does not require declaration under the NeurIPS LLM policy.

\end{enumerate}

\end{document}